\pdfoutput=1
\documentclass[twoside]{article}
\usepackage{xr}
\usepackage[round]{natbib}

\usepackage[accepted]{aistats2025}
\usepackage{makecell}
\usepackage[utf8]{inputenc} %
\usepackage[T1]{fontenc}    %
\usepackage{hyperref}       %
\usepackage{url}            %
\usepackage{booktabs}       %
\usepackage{amsfonts}       %
\usepackage{nicefrac}       %
\usepackage{etoc}
\usepackage{microtype}  %
\usepackage{minitoc}
\usepackage{xcolor}         %
\usepackage{amsmath}
\usepackage{float}
\usepackage{amssymb}
\usepackage{amsfonts}
\usepackage{amsthm}
\usepackage{mathrsfs}
\usepackage{csquotes}
\usepackage{algorithm}
\usepackage{algpseudocode}
\usepackage{graphicx}
\usepackage{subfigure}
\usepackage{subcaption}
\usepackage{multirow}
\usepackage{etoolbox}
\apptocmd{\thebibliography}{\raggedright}{}{}
\newtheorem{theorem}{Theorem}

\newtheorem{lemma}{Lemma}
\newtheorem{definition}{Definition}
\newtheorem{proposition}{Proposition}
\newtheorem{assumption}{Assumption}

\newtheorem{remark}{Remark}

\begin{document}

\twocolumn[

\aistatstitle{ScoreFusion: Fusing Score-based Generative Models
via Kullback–Leibler Barycenters}

\aistatsauthor{ Hao Liu$^*$ \And Junze Tony Ye$^*$ \And  Jose Blanchet \And Nian Si}

\aistatsaddress{ Stanford University \And  Stanford University \And Stanford University \And HKUST } ]

\begin{abstract}

We introduce ScoreFusion, a theoretically grounded method for fusing multiple pre-trained diffusion models that are assumed to generate from auxiliary populations. ScoreFusion is particularly useful for enhancing the generative modeling of a target population with limited observed data. Our starting point considers the family of KL barycenters of the auxiliary populations, which is proven to be an optimal parametric class in the KL sense, but difficult to learn. Nevertheless, by recasting the learning problem as score matching in denoising diffusion, we obtain a tractable way of computing the optimal KL barycenter weights. We prove a dimension-free sample complexity bound in total variation distance, provided that the auxiliary models are well-fitted for their own task and the auxiliary tasks combined capture the target well. The sample efficiency of ScoreFusion is demonstrated by learning handwritten digits. We also provide a simple adaptation of a Stable Diffusion denoising pipeline that enables sampling from the KL barycenter of two auxiliary checkpoints; on a portrait generation task, our method produces faces that enhance population heterogeneity relative to the auxiliary distributions.

\end{abstract}

\section{INTRODUCTION}
Our goal in this paper is to propose and analyze a general method (which we call ScoreFusion) for fusing multiple pre-trained diffusion models that are assumed to simulate auxiliary populations. 

There are multiple needs that motivate this goal. For example, it is well known that diffusion models rely on large datasets (often involving high-dimensional features) and there simply may not be enough data to train a diffusion model from a target population \citep{large1, large2, large4, large5}. Another motivation is that we may wish (at inference time) to sample from a region that has a low probability in the target population, but such a region may be targeted with the help of auxiliary models. As we will demonstrate, ScoreFusion addresses precisely these types of needs.

ScoreFusion starts from the idea that if the auxiliary populations are well chosen, then the target population could be well represented by some KL-weighted-barycenter of the auxiliary populations for a suitable choice of weights, which has an analytical closed form expression \citep{modelFusion, KL}. The ScoreFusion method then tries to find the distribution that optimizes the fit to the target population within this parametric family, based on a limited number of samples. From a statistical standpoint, ScoreFusion reduces the problem of fitting a non-parametric distribution (a task that is extremely challenging to do with a limited sample size) to that of fitting a parametric family (a much more manageable task with a moderate sample size). Moreover, the parametric family is not arbitrary, but derived from a key optimality principle, namely, the KL-barycenter criterion.

Another common barycenter criterion that we could have used is the Wasserstein barycenter \citep{cuturi2014fast, peyre2016gromov, solomon2015convolutional, claici2018stochastic, janati2020debiased}. However, computing Wasserstein barycenters is generally challenging \citep{peyre2019computational, benamou2015iterative, staib2017parallel, genevay2016stochastic}. This is why we utilize the Kullback–Leibler (KL) barycenter. 

Next, we proceed to optimize the weights of the KL-barycenter parametric family to minimize the empirical score-matching loss computed from a limited collection of target samples. Although this formulation is elegant---essentially reducing to a maximum likelihood estimation problem that is convex (as a function of the parametric weights, see Equation (\ref{convex}))---applying gradient descent poses challenges, particularly in high-dimensional settings, due to the need for complex numerical integration.

Fortunately, diffusion generative modeling is particularly useful in our setting, and the way we use them underscores a novel way to apply diffusion modeling. The success of diffusion models \citep{scoreSDE,ddpm, 2015paper} has sparked a well-developed infrastructure to train those models, with many pre-trained or fine-tuned checkpoints available on websites like Hugging Face and Civitai. We take advantage of these resources and show that if the auxiliary target populations have been fitted using them, then our learning task is greatly simplified. To be more precise, using auxiliares trained with diffusion and solving Problem (\ref{newtrainloss}) is a simpler way of learning the KL-barycenter parametric family. 

Our method can be viewed as an example of ensemble learning \citep{ensemble, merging_survey}. We combine the idea of ensemble learning and score-based generative model. Moreover, our approach is optimal in the sense of KL divergence, and our weights can be easily learned on limited data. It is related to empirical methods like checkpoint merging for diffusion models \citep{AUTOMATIC1111, merge_loras}, which are not based on an optimality principle as we do (Equation (\ref{newtrainloss})) and lack theoretical basis.

A Taylor expansion argument is presented in \citet{biggs2024}, which tries to connect interpolation in the parameter space (checkpoint merging) of the denoiser U-Net to that in its function space (ScoreFusion sampling). The catch is, the Taylor expansion argument is valid only if the parameters of the merging neural networks differ by small amounts. So our paper provides theoretical basis to checkpoint merging in the context of small parameter perturbations. But we also show in experiments that the KL-barycenter approach is different and can produce more heterogeneity in populations that have relatively low representation.

The sampling aspect of ScoreFusion is also related to concurrent methods of synthesizing semantic concepts to create novel images. For example, \citet{richardson2024conceptlab}, \citet{li2024tp2o}, and \citet{xiong2024novel} proposed different ways of combining concepts in text-to-image generation. The difference between their approaches and ours is the specific component of the generative pipeline where the interpolation happens: our KL barycenter controls interpolation in the probability space of observable outcomes (e.g. the pixel space, in the case of generative vision), whereas theirs combine concepts in the latent embedding spaces.

Our main contributions are summarized as: 
\begin{itemize}
    \item We demonstrate that KL barycenter fusion of auxiliary models can be efficiently implemented when the auxiliary models are trained by score matching. In this case, the optimal score is linear in the auxiliary scores.

    \item We provide generalization bounds which split the error into four components. First, the error between the optimal KL barycenter and the target at time zero (whose direct implementation is difficult due to numerical integration). The second term corresponds to the sample complexity $O(n^{-1/4})$ and the third term is the approximation error incurred from converting the training into a regression problem. The fourth component reflects the quality of auxiliary score estimations.

    \item We demonstrate empirically that ScoreFusion succesfully addresses the needs that motivate our goal, as mentioned earlier. We demonstrate ScoreFusion's sample efficiency on MNIST digits measured in both image fidelity and class proportions. Further, we show ScoreFusion's ability to sample from low probability regions on the task of generating professional portraits.  %
\end{itemize}

The rest of the paper is organized as follows. Section \ref{sec:setup} reviews the background of KL barycenter and diffusion models. Section \ref{Sec:fusion} details our proposed fusion methods. Section \ref{sec:theory} provides convergence results for our methods.  Section \ref{section:numerics} presents experiment results and comparisons with baseline methods. Finally, Section \ref{sec:conclusion} concludes the paper with future directions.  All proofs, code, and additional experimental details are relegated to the Supplementary Materials.

\section{PRELIMINARIES AND SETUP}
\label{sec:setup}
\subsection{Notations}
The following notation will be used. Given two functions $f,g: D \to \mathbb{R}$, we say $ f \lesssim g$ if there exists a constant $C > 0$ such that for all $x \in D$, $f(x) \leq Cg(x)$. When $x \to a$, where $a \in [-\infty, \infty]$, we say $f(x) = \mathcal{O}(g(x))$ if there exists a constant $C > 0$ such that for all $x$ close enough to $a$, $|f(x)| \leq Mg(x)$. In asymptotic cases, we use $\mathcal{O}$ and $\lesssim$ interchangeably. $f \sim g$ if and only if $f \lesssim g$ and $g \lesssim f$. $C([0, T ] : \mathbb{R}^d)$ is the space of all continuous functions on $\mathbb{R}^d$ equipped with the uniform topology. 
In this paper, we consider a Polish spaces $S$, which could be $\mathbb{R}^d$ or $C([0, T ] : \mathbb{R}^d)$.
For a Polish space $S$ equipped with Borel $\sigma$-algebra $\mathcal{B}(S)$, we denote $\mathcal{P}(S)$ as the space of probability measures on $S$ equipped with the topology of weak convergence. 
In a normed vector space $\left(X, \left\lVert .\right\rVert\right)$, $\left\lVert .\right\rVert$ denotes the corresponding norm. $\left\lVert .\right\rVert_p$ denotes the standard $L^p$ norm. Given a matrix $\boldsymbol{A}$, we use $\boldsymbol{A}^T$ to denote its transpose. We denote $\boldsymbol{\lambda} = \left(\lambda_1, \ldots, \lambda_k\right)^T \in [0,1]^k$. We use $\Delta_k$ to present the $k$-dimensional probability simplex, i.e., $\Delta_k=\{\boldsymbol{\lambda}\in[0,1]^k:\sum_{i=1}^k\lambda_i=1\}$.

\subsection{Barycenter Problems and Kullback–Leibler Divergence}\label{sub:KL-barycenter}
 Given a set of probability measures $P_1, \ldots, P_k \in \mathcal{P}(S)$ on a Polish space $S$ and a measure of dissimilarity (e.g. a metric or a divergence) between two elements in $\mathcal{P}(S)$, $D$,  we define the barycenter problem with respect to $D$ and weight $\boldsymbol{\lambda}$ as the optimization problem  
$\min_{\mu \in \mathcal{P}(S)} \sum_{i=1}^k \lambda_iD\left(\mu, P_i\right)        \quad \text{s.t. } \boldsymbol{\lambda} \in \Delta_k, $  where $P_1, \ldots, P_k$  are called the reference measures. With a fixed choice of weight and reference measures, the solution of the barycenter problem is denoted as $\mu_{\boldsymbol{\lambda}}$.

Recall the definition of Kullback–Leibler (KL) divergence: suppose $P,Q \in \mathcal{P}(S)$, then $ D_{\text{KL}}\left(P \parallel Q\right) =  \int \log \left(\frac{dP}{dQ} \right) \, dP$ if $P \ll Q$ and $ D_{\text{KL}}\left(P \parallel Q\right) = \infty$ otherwise;
where $\frac{dP}{dQ}$ is the Radon-Nikodym derivative of $P$ with respect to $Q$. In particular, if $S = \mathbb{R}^d$, $P$ and $Q$ are absolutely continuous random vectors (with respect to Lebesgue measure) in $\mathbb{R}^d$ with densities $p$ and $q$ respectively, then 
    $D_{\text{KL}}\left(P \parallel Q\right) = \int p(x)\log\left( 
\frac{p(x)}{q(x)} \right)dx. $
If $D$ is the KL divergence, we recover the KL barycenter problem \citep{modelFusion}. In fact, for any Polish space $S$, the KL barycenter problem is strictly convex hence has at most one solution. 

\subsection{Background on Diffusion Models}\label{DDPM}
Our score fusion method depends  the generative diffusion model driven by stochastic differential equations (SDEs) developed in \citet{scoreSDE, ddpm, 2015paper}. In this section, we review the  background of generative diffusion model.
\subsubsection{Forward Process}
We begin with the unsupervised learning setup. Given an unlabeled dataset i.i.d. from a distribution $p_0$, the forward diffusion process is defined as the differential form
\begin{equation}\label{eq:forwardSDE}
    dX(t) = f(t, X(t))dt + g(t)dW(t), X(0) \sim p_0, 
\end{equation}
where $f: \mathbb{R}_{+} \times \mathbb{R}^d \to \mathbb{R}^d$ is a vector-valued function, $g:  \mathbb{R}_{+} \to \mathbb{R}_{+}$ is a scalar function, and $W(t)$ denotes a standard $d$-dimensional Brownian motion. From now on, we assume the existence and denote by $p_t(x)$ the marginal density function of $X(t)$, and let $p_{t|s}\left(X(t)|X(s)\right)$ be the transition kernel from $X(s)$ to $X(t)$, for $0 \leq s \leq t \leq T < \infty$, where $T$ is the terminal time for the forward process (time horizon). If $f(t,x) = -ax$ and $g(t) = \sigma$ with $a > 0$ and $\sigma > 0$, then Equation (\ref{eq:forwardSDE}) becomes a linear SDE with Gaussian transition kernels
\begin{equation}\label{eq: forwardOU}
    dX(t) = -a X(t)dt + \sigma dW(t), X(0) \sim p_0,
\end{equation}
which is an Ornstein-Ulenback (OU) process. If $T$ is large enough, then $p_T$ is close to $\pi \sim \mathcal{N}\left(\textbf{0}, \frac{\sigma^2}{2a}\textbf{I}\right)$, a Gaussian distribution with mean $\textbf{0}$ (vector) and covariance matrix $\frac{\sigma^2}{2a}\textbf{I}$. The forward process can be viewed as the following dynamic: given the data distribution, we gradually add noise to it such that it becomes a known distribution in the long run. 

\subsubsection{Backward Process (Denoising)}
If we reverse a diffusion process in time, then under some mild conditions (see, for example, \citet{reverse2, reverse1}) which are satisfied for all processes under consideration in this work, we still get a diffusion process. To be more precise, we want to have a process $\tilde{X}$ such that for $t \in [0,T]$, $\tilde{X}(t) = X(T-t)$. From the Fokker–Planck equation and the log trick \citep{Anderson1}, the corresponding reverse process for Process (\ref{eq:forwardSDE}) is 
\begin{align}
    &d\tilde{X}(t) =  g^2(T-t) \nabla \log p_{T-t}\left(\tilde{X}(t)\right)dt \notag\\
    &-f(T-t, \tilde{X}(t))dt + g(T-t)dW(t), \tilde{X}(0) \sim p_T, \label{eq:genearl1}
\end{align}

where $\nabla$ represents taking derivative with respect to the space variable $x$. We call the term $\nabla \log p_t(x)$ as the (Stein) score function.
If the forward process is an OU process, then the reverse process is $\tilde{X}(0) \sim p_T$ and
\begin{equation}\label{eq:OU1}
    d\tilde{X}(t) = \left( a\tilde{X}(t) + \sigma^2 \nabla \log p_{T-t}\left(\tilde{X}(t)\right)\right)dt + \sigma dW(t).
\end{equation}
If  the backward SDE can be simulated (which is typically done via Euler–Maruyama method, see details in Supplementary Material Section \ref{Euler}),  we can generate samples from the distribution $p_0$. 
We can view simulating the backward SDE as the denoising step from pure noise to the groundtruth distribution. 

\subsubsection{Score Estimation}
The only remaining task is score estimation for $\nabla \log p_t(x)$. There are many ways to achieve this, and some of them are equivalent up to constants that is independent of the training parameters. In this paper, we choose the time-dependent score matching loss used in \citet{MLE}:
\begin{align}\label{trainloss}
    &\mathcal{L}\left(\theta; \gamma\right) := \mathbb{E}_{t \sim \mathcal{U}[0,T]}\left[\gamma(t) \mathbb{E}_{X(t) \sim p_t} \left[ \right.\right. \notag\\
    &\left.\left. \left\lVert s_{t, \theta}\left(X(t)\right) - \nabla \log p_t(X(t)) \right\rVert_2^2 \right] \right],
\end{align}
where $\gamma: [0,T] \to \mathbb{R}_+$ is a weighting function,  and $s_{t, \theta}: \mathbb{R}^d \to \mathbb{R}^d$ is a score estimator  $s_{t,\theta}$, usually chosen as a neural network.  Then score estimation is done by the empirical loss using SGD \citep{generalizations}.

There are many performance measures. Suppose $D(.,.)$ is a measure of dissimilarity in $\mathcal{P}(S)$, then we say $D(\mu, \hat{\mu})$ is a generalization error with respect to $D$, where $\mu$ is the target distribution and $\hat{\mu}$ is the distribution of the generated samples.

Recently, several analysis about the generative properties of diffusion models has been done; however, even in the case of compactly supported target distributions and sufficient smoothness regularity, the basic diffusion model encounters the curse of dimensionality. Therefore, a large amount of target data is needed to generate high quality samples. For a detailed discussion, see Supplementary Material Section \ref{randomfeature}. 

\section{KL BARYCENTER AND FUSION METHODS}
\label{Sec:fusion}
In Section \ref{Sec:KL_sol}, we propose and analytically solve two types of KL barycenter problems. These solutions will lead to the development of our fusion methods, which is detailed in Section \ref{Sec2}.
\subsection{KL Barycenter Problems}\label{Sec:KL_sol}
\textbf{Distribution-level KL barycenter.} We first consider a simple case, where the KL barycenter problem is solved on an Euclidean space, i.e.,  $S = \mathbb{R}^d$.  The barycenter obtained in this scenario is referred to as a distribution-level KL barycenter. Theorem \ref{sol1}  gives the analytical solution.

\begin{theorem}\label{sol1}
    Suppose $\{\mu_1, \ldots, \mu_k \} \subset \mathcal{P}(\mathbb{R}^d)$ and for each $i = 1, \ldots, k,$ $\mu_i$ is absolutely continuous  with respect to the Lebesgue measure, with densities $p_1, \ldots, p_k$ respectively. Then, the distribution-level KL barycenter $\mu_{\boldsymbol{\lambda}}$  is unique with density $p_{\boldsymbol{\lambda}}(x) = \frac{\prod _{i=1}^k p_i(x)^{\lambda_i}}{\int_{\mathbb{R}^d} \prod _{i=1}^k p_i(x)^{\lambda_i} dx}.$
\end{theorem}
\textbf{Process-level KL barycenter.} 
Our second barycenter problem is performed when the state space is the continuous-function space, i.e., $S = C([0, T ] : \mathbb{R}^d)$.
This context yields a process-level KL barycenter. When the underlying measures are represented by SDEs, we offer a closed-form solution for the process-level KL barycenter in Theorem \ref{sol2}.

\begin{theorem}\label{sol2}
    Suppose for each $i = 1, 2, \ldots, k$,  the $i$-th SDE has the form with $X_i(0) \sim \mu_i$,
    \begin{equation*}
       dX_i(t) = a_i\left(t, X(t)\right)dt + \sigma(t) dW_i(t) 
    \end{equation*}
    and has a unique strong solution. The law of solution of each SDE is denoted as $P_i \in  \mathcal{P}(C([0, T ] : \mathbb{R}^d))$. We further assume, for each $i = 1, 2, \ldots, k$, $\mu_i$ has an absolutely continuous density with respect to the Lebesgue measure and $a_i$ uniformly bounded, then process-level KL barycenter can be represented as the SDE with $X(0) \sim \mu$ and
    \begin{equation*}
       dX(t) =  a\left(t, X(t)\right)dt + \sigma(t) dW(t),  
    \end{equation*}
    where $a(t,x) =  \sum_{i=1}^{k}\lambda_i a_i(t,x)$, $\mu$ is the distribution-level KL barycenter of  reference measures $\mu_1, \ldots, \mu_k$, and $W$ is a standard Brownian motion.
\end{theorem}

In this paper, fusing  $k$ distributions is viewed as computing a KL barycenter with optimized weights. Given $k$ well-trained reference generative models,  our fusing method optimizes the weights $\lambda_1, \ldots, \lambda_k$ to approximate  a target distribution.

\subsection{Fusion Methods}\label{Sec2}
Recall that in our task, we are given $k$ datasets with abundant samples, and our goal is to generate samples for a target dataset (with limited available data). Therefore, in this section, we denote the target measure as $\nu$ and we assume that we are given $k$ reference diffusion generative models and  they are able to generate samples from $k$ different reference measures $\mu_1,\ldots,\mu_k$, respectively. Specifically,  each reference measure corresponds to an auxiliary backward diffusion process with 
$\tilde{X}_i(0) \sim p^i_T$ and  \begin{equation}\label{auxbackward}
 	d\tilde{X}_i(t) = \left( a\tilde{X}_i(t) + \sigma^2s^i_{T-t,\theta^*}\left(\tilde{X}_i(t)\right)\right)dt + \sigma dW_i(t), 
 \end{equation}
 where $s^i_{T-t,\theta^*}$ is a well-trained score function for the the $i$-th reference measure.
 we introduce two fusion algorithms and related generalization error results.

In practice, the discretized version of the SDE (\ref{auxbackward}) is used. Specifically, we employ a small time-discretization step $h$ and a total of  $N$ time steps (hence $T = Nh$). Since $p_T^i$ is close to the Gaussian distribution $\pi$, the SDE  (\ref{auxbackward})  is approximated by  $\hat{X}(0) \sim \pi$ and for all $t \in [lh, (l+1)h]$,
\begin{equation}\label{auxbackward:SDE}
	d\hat{X}_i(t) = \left( a\hat{X}_i(t) + \sigma^2  s^i_{T-lh,\theta^*}\left(\hat{X}_i(lh)\right) \right)dt + \sigma dW(t).
\end{equation}

As shown in \citet{generalizations} Lemma 1, when the target distribution is compactly supported, then with high probability, the trajectory is also compactly supported, thus the score term is uniformly bounded. Thus, given a weight $\boldsymbol{\lambda}$,  Theorem \ref{sol2} implies that the corresponding process-level KL barycenter follows the SDE: for all $t \in [lh, (l+1)h], \hat{Y}(0)\sim \pi$,
\begin{align}
    &d\hat{Y}(t) = \sigma^2 \sum_{i=1}^k  \lambda_i s^i_{T-lh,\theta^*}\left(\hat{Y}(lh)\right) dt \notag\\
    & + a\hat{Y}(t)dt +  \sigma dW(t). \label{pracbarycenter}
\end{align}
We denote the distribution of the terminal variable $\hat{Y}(T)$ as $\hat{p}_{\boldsymbol{\lambda}}$, which will later serve as the distribution of generated sample.

The key component in our diffusion method is to find an optimal $\boldsymbol{\lambda}^*$ such that the $\hat{p}_{\boldsymbol{\lambda}^*}$ is as close as the target measure $\nu$ as possible. To achieve this goal, we propose two fusion methods that relies on two different optimization problems.

The first method directly optimizes on the probability measure defined on the Euclidean space, which is based on Theorem \ref{sol1}. We notice that a similar idea of fusing component distributions via KL barycenter compared with vanilla fusion has been proposed in \citet{modelFusion}, which uses averaging KL divergence as a metric to recover the mean-field approximation of posterior distribution of the fused global model. 

Namely, we consider the following convex problem 
\begin{align}\label{convex}
    &\min_{\boldsymbol{\lambda} \in \Delta_k}D_{\text{KL}}(\nu \parallel \mu_{\boldsymbol{\lambda}})\\&= \min_{\boldsymbol{\lambda} \in \Delta_k}\left\{  \mathbb{E}_{\nu}\left[\log q(X) - \sum_{i=1}^{k}\lambda_i  \log p_i(X)\right] \right. \notag\\
    &\left. + \log\left(  \int \prod _{i=1}^k p_i(y)^{\lambda_i}dy \right)\right\},\notag
\end{align}

where  $p_1, \ldots, p_k$ denote the densities of the reference measures and $q(x)$ denote the density of target distribution $\nu$.  We  refer to this fusion method as \textit{vanilla fusion}. Suppose we have an accurate estimation of the densities $p_i$s. We then use Frank-Wolfe method to solve Problem (\ref{convex}) and get an optimal $\boldsymbol{\lambda}^*$. In the Frank-Wolfe method, the gradient term can be approximated by sample mean estimators from target data $\nu$ (See Remark \ref{gradient} in Supplementary Material Section \ref{lemmaSec}). To generate samples, we plug in the $\boldsymbol{\lambda}^*$ to (\ref{pracbarycenter}) and simulate the SDE. 
However, the diffusion generative models usually cannot directly estimate the densities   $p_1, \ldots, p_k$. Therefore, for  complicated high-dimension distributions, it is usually hard to directly apply vanilla fusion. Therefore,  we propose a practical alternative, process-level method called \textbf{ScoreFusion}. The numerical results in Section \ref{section:numerics} were generated by employing Algorithm \ref{Algo:2}.

In our second method, we first build a forward process starting from the target dataset, according to (\ref{eq: forwardOU}). We denote this forward process as $\tilde{X}^\nu(t)$ and the corresponding density as $p_t^\nu(x)$. Then, we modify the loss function (\ref{trainloss}) as a linear regression problem 
\begin{align}
&\tilde{\mathcal{L}}\left(\boldsymbol{\lambda};\theta^*, \gamma\right)= \mathbb{E}_{t \sim \mathcal{U}[0,\tilde{T}]}\left[\gamma(t)\left(\mathbb{E}_{\tilde{X}^\nu(t) \sim p^\nu_t}\left[ \right.\right.\right. \notag\\
&\left.\left.\left.\left\lVert \sum_{i=1}^{k} \left( \lambda_i s^i_{t, \theta^*}\left(\tilde{X}^\nu(t)\right)\right) - \nabla \log p^\nu_t(\tilde{X}^\nu(t))  \right\rVert_2^2 \right]\right) \right], \label{newtrainloss}
\end{align}
where we choose  
$\tilde{T} \ll T$. The intuition behind the choice of $\tilde{T}$ is that we want to learn an optimal $\boldsymbol{\lambda}^*$ such that $p_{\boldsymbol{\lambda}^*}$ is close to the target $\nu$. Therefore, when $\tilde{T} \ll T$ (the forward process has not inject much noise), the $\hat{\boldsymbol{\lambda}}$ obtained from the training is affected less by the noise. Theoretically, choosing $\tilde{T} = 0$ is optimal, but this is hard to implement. Algorithm \ref{Algo:2} with $\tilde{T} = 0$ can be viewed as a variant of vanilla fusion since the learning is only performed on the distribution level ($p_0$), and extremely small $\tilde{T}$ causes numerical instability in practice, which makes sense given the numerical integration and density estimations needed in the vanilla fusion.  
The optimization problem associated with our second method is 
$
	\min_{\boldsymbol{\lambda} \in \Delta_k} \tilde{\mathcal{L}}\left(\boldsymbol{\lambda};\theta^*, \gamma\right).
	$ 
The details are in Algorithm \ref{Algo:2}. 
\begin{algorithm}\caption{ScoreFusion}\label{Algo:2}
\begin{algorithmic}[1]
    \State \textbf{Input:} Calibration data $\mathcal{D}$, pre-trained score functions $s^1_{t,\theta^*},\ldots,s^k_{t,\theta^*}$. Hyperparameter $\tilde{T}$.
    \State \textbf{Output:} Samples from a distribution $\hat{\nu}_P$.
    \State \textbf{I. Calibration Phase}
    \State Randomly initialize non-negative $\lambda_1,\ldots,\lambda_k$ s.t. $\sum \lambda_i = 1$.
    \Repeat
        \State Run forward process $\tilde{X}^\nu(t)$ using a mini-batch from $\mathcal{D}$.
        \State Evaluate the loss function (\ref{newtrainloss}) and back-propagate onto $\lambda_1,\ldots,\lambda_k$. 
        \State \Comment{$\lambda_i$'s are softmaxed to enforce the probability simplex constraint}
    \Until{converged. Save the optimal $\boldsymbol{\lambda}^*=\{\lambda^*_1,\lambda^*_2,\ldots,\lambda_k^*\}$.}
    \State \textbf{II. Sampling Phase}
    \State $s_{t, \boldsymbol{\lambda}^*} (\hat{Y}(t)) := \sum_{i=1}^k \lambda^*_i s_{t, \theta^*}^i( \hat{Y}(t))$. 
    \State Simulate the backward SDE (\ref{pracbarycenter}) with $s_{t, \boldsymbol{\lambda}^*} (\cdot)$ starting from a Gaussian prior and generate samples. 
\end{algorithmic}
\end{algorithm}

\section{CONVERGENCE RESULTS}\label{Sec4}
\label{sec:theory}
This section details the convergence results for our proposed fusion
methods. We focus on sample complexities, quantified by the necessary number
of samples in the target dataset, in terms of total variation distance. We
show that the sample complexities of our methods are dimension-free, given
that the auxiliary processes are accurately fitted to their reference
distributions and together offer adequate information for the target
distribution. 
To begin with, we assume all distributions are compactly supported.

\begin{assumption}
\label{compact} The target and reference distributions are all compactly
supported in $\mathbb{K}\subset \mathbb{R}^{d}$ with absolutely continuous
densities. We assume that their second moments are bounded by $M\in
(0,\infty )$. 
\end{assumption}

\begin{proposition}
\label{lemma2} Under Assumption \ref{compact}, Problem (\ref{convex}) is
convex in $\boldsymbol{\lambda}$.
\end{proposition}
Proposition \ref{lemma2} implies that Problem (\ref{convex}) is easy to solve given that the reference densities can be estimated.  
We further require Assumption \ref{Lipschitzanderror} below, which guarantees that each auxilary process is accurately trained in the sense that the score function at each time step is well-fitted.
\begin{assumption}
	\label{Lipschitzanderror} For each $1,2, \ldots, k$ and for all $t \in [0,T]$%
	, $\nabla \log p^i_t$ is $L$-Lipschitz with $L\geq1$, the step size $h = T/N$ satisfies $h
	\lesssim 1/L$, and the inverse of $\nabla \log p^i_t$ is also Lipschitz; for each $1, \ldots, k$ and $l = 0,
	1, \ldots, N$, $\mathbb{E}_{p^i_{lh}}\left[ \left\| s^i_{lh, \theta^*} -
	\nabla \log p^i_{lh}\right\|_2^2\right] \leq \epsilon_{\text{score}}^2$
	with small $\epsilon_{\text{score}}$.
\end{assumption}
Assumption \ref{Lipschitzanderror} is widely used in the diffuion model literature (see, for example, \citet{discretizationerror}).

To proceed, we denote $\boldsymbol{\lambda }^{\ast }$ and $\boldsymbol{\Lambda }%
^{\ast }$ to be the solutions of Problems \ref{convex} and \ref%
{newtrainloss}, respectively. The corresponding barycenters are denoted as $\mu_{\boldsymbol{\lambda }^{\ast }}$ and  $\mu_{\boldsymbol{\Lambda }^{\ast }}$.
Assumption \ref{optimalaux} below states that the theoretical optimal barycenters are close to the target measure, which ensures all reference distributions together are able to provide sufficient information for the target distribution.
\begin{assumption}
\label{optimalaux} $D_{\text{KL}}\left( \nu \parallel \mu_{\boldsymbol{\lambda}%
^*}\right) \leq \epsilon_0^2$ and $D_{\text{KL}}\left( \nu \parallel \mu_{%
\boldsymbol{\Lambda}^*}\right) \leq \epsilon_1^2$, with small $\epsilon_0$
and $\epsilon_1$.
\end{assumption}

Based on Assumptions \ref{compact}, \ref{Lipschitzanderror} and \ref{optimalaux}, we provide convergence results for the vanilla fusion and ScoreFusion (Algorithm \ref{Algo:2}) in Theorems \ref{err1} and \ref{err2}, respectively.

\begin{theorem}
\label{err1} Suppose that   Assumptions \ref{compact}%
, \ref{Lipschitzanderror}, and \ref{optimalaux}  are satisfied. We further assume for each fixed $\boldsymbol{\lambda} \in \Delta_k$, $\text{TV}\left(\mu_{%
	\boldsymbol{\lambda}}, \hat{\mu}_{\boldsymbol{\lambda}}\right) \leq \epsilon_2$, where $\hat{\mu}_{\boldsymbol{\lambda}}$ is the barycenter of the output distributions of 
 $k$ auxiliary processes.
 Then, for $\delta > 0$ and $\delta \ll 1$, the output distribution of the vanilla fusion method, $\hat{\nu}_D$, we have with probability at least $1-\delta$, 
 \begin{align*}
  \text{TV}\left(\nu, \hat{\nu}_D \right)
&\lesssim \underbrace{\epsilon_0}_{\text{quality of combined auxiliaries}}\\  &+\underbrace{\epsilon_2}_{\text{auxiliary density estimation}}\\
& + \underbrace{{\mathcal{O}\left(\left(\log\left(\frac{1}{\delta}\right)\right)^{1/4}n^{-1/4}%
\right)}}_{\text{mean estimation error}} +\text{ SE},
 \end{align*}
 where SE is the error of auxiliary score estimation, defined as
\begin{align*}
    &SE := \left[\exp(-T)\max_{i = 1,2,\ldots, k}%
\sqrt{D_{\text{KL}}\left(p^i_T \parallel \pi\right)} \right. \notag\\   &\left.+ \sigma\sqrt{kT}%
\left(\epsilon_{\text{score}} + L\sqrt{dh} + Lh\sqrt{M}\right)\right].
\end{align*}
\end{theorem}

\begin{theorem}
\label{err2} Suppose that Assumptions \ref{compact}%
, \ref{Lipschitzanderror}, and \ref{optimalaux}  are satisfied. 
Then, for $\delta > 0$ and $\delta \ll 1$, for the output distribution of Algorithm \ref{Algo:2}, $\hat{\nu}_P$, with probability at least $1-\delta$, 
\begin{align*}
  \text{TV}\left(\nu, \hat{\nu}_P \right)
&\lesssim \underbrace{(\sigma+1)\epsilon_1}_{\text{quality of combined auxiliaries}}\\  &+\underbrace{\sigma\sqrt{k}\mathcal{%
O}\left(\tilde{T}^{1/4}\right)}_{\text{approximation of time 0}}\\
& + \underbrace{{\mathcal{O}\left(\sigma\left(\log\left(\frac{1}{\delta}\right)\right)^{1/4}n^{-1/4}%
\right)}}_{\text{sampling error}} +\text{ SE}.
 \end{align*}
\end{theorem}

Theorems \ref{err1} and \ref{err2} demonstrate dimension-free sample complexities given that auxiliaries are well approximated and  auxiliaries all combined capture the features of target well. More specifically, each bound in  Theorems \ref{err1} and \ref{err2} has 4 terms, which represents different sources of error.

The quality of combined auxiliaries is the essential assumption in both Theorems \ref{err1} and \ref{err2}. The sampling error in Theorem \ref{err2} reflects the fact that with the help of diffusion models, the optimization in fact becomes linear in terms of scores, making the problem easier and escape the curse of dimensionality. The the approximation to time $t=0$ term replaces the vanilla fusion with a small controllable noise but makes the implementation much easier. It worth noticing that there is a tradeoff between choosing $\tilde{T}$: the smaller $\tilde{T}$, the more accurate the optimal weights are, but the more probably that the algorithm will encounter numerical instability. Finally, the score estimation term of the auxiliaries can be small with a careful choice of discretization time steps and accurate auxiliary score approximation (see Remark  \ref{remark3} in Supplementary Material Section \ref{E1}).

\section{EXPERIMENTS}
\label{section:numerics}
In this section, we use images to illustrate some of the key features highlighted in our contributions (third bullet point) in the Introduction. Additional conceptual experiments are given in the Supplementary Material. We release our code at \url{https://github.com/junzeye/ScoreFusion}.

\subsection{Calibrating MNIST Digits Distribution}\label{sub:MNIST}
EMNIST \citep{emnist} is an enriched version of the well-known MNIST dataset in 1x28x28 grayscale. We selected five non-overlapping subsets ($D_i$, $i=1,\ldots,5$), composed entirely of the digits \texttt{7} and \texttt{9} but with varying mixtures: $(10\%, 90\%)$, $(30\%, 70\%)$, $(70\%, 30\%)$, $(90\%, 10\%)$, and $(60\%, 40\%)$, respectively. Four auxiliary score networks are trained from scratch on $D_1,\ldots,D_4$ with enough data to ensure adequate validation loss convergence. $D_5$ is reserved as the target dataset for calibrating (finetuning) both ScoreFusion and the baseline models. In our experiment, we vary the quantity of class-balanced samples from $D_5$ made available for the model to learn.

Our first baseline method trains a score network from scratch using only limited target data, illustrating the difficulty of density estimation in the low-data regime. Our second baseline directly fine-tunes (with early-stopping) one of the pre-trained auxiliary models on the target data by unfreezing its checkpoint; the chosen auxiliary model was pre-trained on $D_3$, whose digits mix (70-30) is the closest to $D_5$ in a Wasserstein sense. Lastly, the KL barycenter weights of ScoreFusion are obtained by optimizing exclusively a linear projection layer on top of the four auxiliary score networks.

\begin{table}[htbp]
\centering
\caption{Mean NLL (bits/dim) under different sample sizes. A smaller value is better.}
\label{nll}

\resizebox{3.2in}{!}{%
\begin{tabular}{@{}ccccc@{}}
\Xhline{1.2pt} 
Sample size & $2^6$ & $2^8$ & $2^{10}$ & $2^{12}$ \\ 
\midrule
B1 & $7.186 \pm 0.019$ & $6.235 \pm 0.016$ & $5.725 \pm 0.024$ & $4.979 \pm 0.028$ \\
B2 & $4.779 \pm 0.042$ & $4.786 \pm 0.034$ & $4.769 \pm 0.032$ & $4.763 \pm 0.034$ \\
Frozen B2 &  \multicolumn{4}{c}{$4.768 \pm 0.024$} \\
Ours & $\mathbf{4.733 \pm 0.029}$ & $\mathbf{4.733 \pm 0.018}$ & $\mathbf{4.718 \pm 0.022}$ & $\mathbf{4.715 \pm 0.021}$ \\
\Xhline{1.2pt} 
\end{tabular}%
}
\end{table}

We evaluate both fidelity and diversity of samples generated by each method at inference time. Fidelity is measured by the negative log likelihood (NLL) of a holdout test dataset under the trained model \citep{ddpm}. To quantify digits distribution, we use a high-accuracy digits classifier \citep{spinalnet} to classify $1024$ samples generated by each method.

\begin{table*}[h]
  \centering
  
  \caption{Class proportions of $1024$ handwritten digits generated by each differently trained model, estimated by a high-accuracy MNIST classifier. ``B1'' and ``B2'' refer to Baselines 1 and 2 respectively. ``Others'' are samples assigned to a class that is neither \texttt{7} nor \texttt{9}. $2^6, \cdots, 2^{12}$ refer to number of training data from the target distribution.}
  
  \resizebox{0.8\textwidth}{!}{%
    \begin{tabular}{cc ccc ccc ccc ccc} 
    \Xhline{1.2pt}  
    \multirow{2}{*}{Digit} & \multirow{2}{*}{Target} & \multicolumn{3}{c}{$2^{6}$} & \multicolumn{3}{c}{$2^8$} & \multicolumn{3}{c}{$2^{10}$} & \multicolumn{3}{c}{$2^{12}$} \\
    & & B1 & B2 & Ours & B1 & B2 & Ours & B1 & B2 & Ours & B1 & B2 & Ours \\
    \hline
    \texttt{7}     & 60\% & 47.9\% & 72.4\% & \textbf{55.6\%} & 66.8\% & 65.5\% & \textbf{57.5\%} & 65.5\% & 65.1\% & \textbf{56.6\%} & 66.7\% & 65.5\% & \textbf{59.8\%} \\
    \texttt{9}     & 40\% & 10.3\% & 23.2\% & \textbf{39.4\%} & 23.8\% & 29.9\% & \textbf{38.0\%} & 26.7\% & 30.6\% & \textbf{39.8\%} & 27.9\% & 30.4\% & \textbf{36.7\%} \\
    Others & 0     & 41.8\% & 4.4\% & \textbf{5.0\%} & 9.4\% & 4.6\% & \textbf{4.5\%} & 7.8\% & 4.3\% & \textbf{3.6\%} & 5.4\% & 4.1\% & \textbf{3.5\%} \\
   \Xhline{1.2pt} 
    \end{tabular}%
  }
  \label{tab:digits-count}
\end{table*}

Tables \ref{nll} and \ref{tab:digits-count} show summary statistics of our evaluations. We relegate training details and sample images to the Appendix. Across all four tested sample sizes, ScoreFusion achieves a lower NLL than the other two baselines. Moreover, despite the alignment of digits proportions not being hard-coded in the score matching loss minimization, Table \ref{tab:digits-count} shows that with as few as $64$ samples, ScoreFusion already learns a generative model whose digits proportions closely align with the ground-truth 60-40 split. At the same time, Baseline 2 is slow to calibrate its digits proportions even as we increase the fine-tuning data size.

\subsection{Sampling a New Facial Distribution}\label{sub:SDXL}

SDXL 1.0 is the newest model in the Stable Diffusion family \citep{sdxl1_release}, capable of generating realistic images in 1024x1024 resolution. We downloaded two fine-tuned SDXL checkpoints to use them as our auxiliary models \citep{af_lora, wm_lora}. These two models were each finetuned by their creators to generate human portraits who look like White males and East Asian females, respectively, in Figure \ref{fig:biased}. For consistency, we use the same text prompt in all images generations: \textit{``a photo of a mathematics scientist, looking at the camera, ultra quality, sharp focus''}. Importantly, this prompt only specifies a person's profession, leaving their other traits unspecified. 

\begin{figure}[h]
  \centering
  \includegraphics[width=0.85\columnwidth]{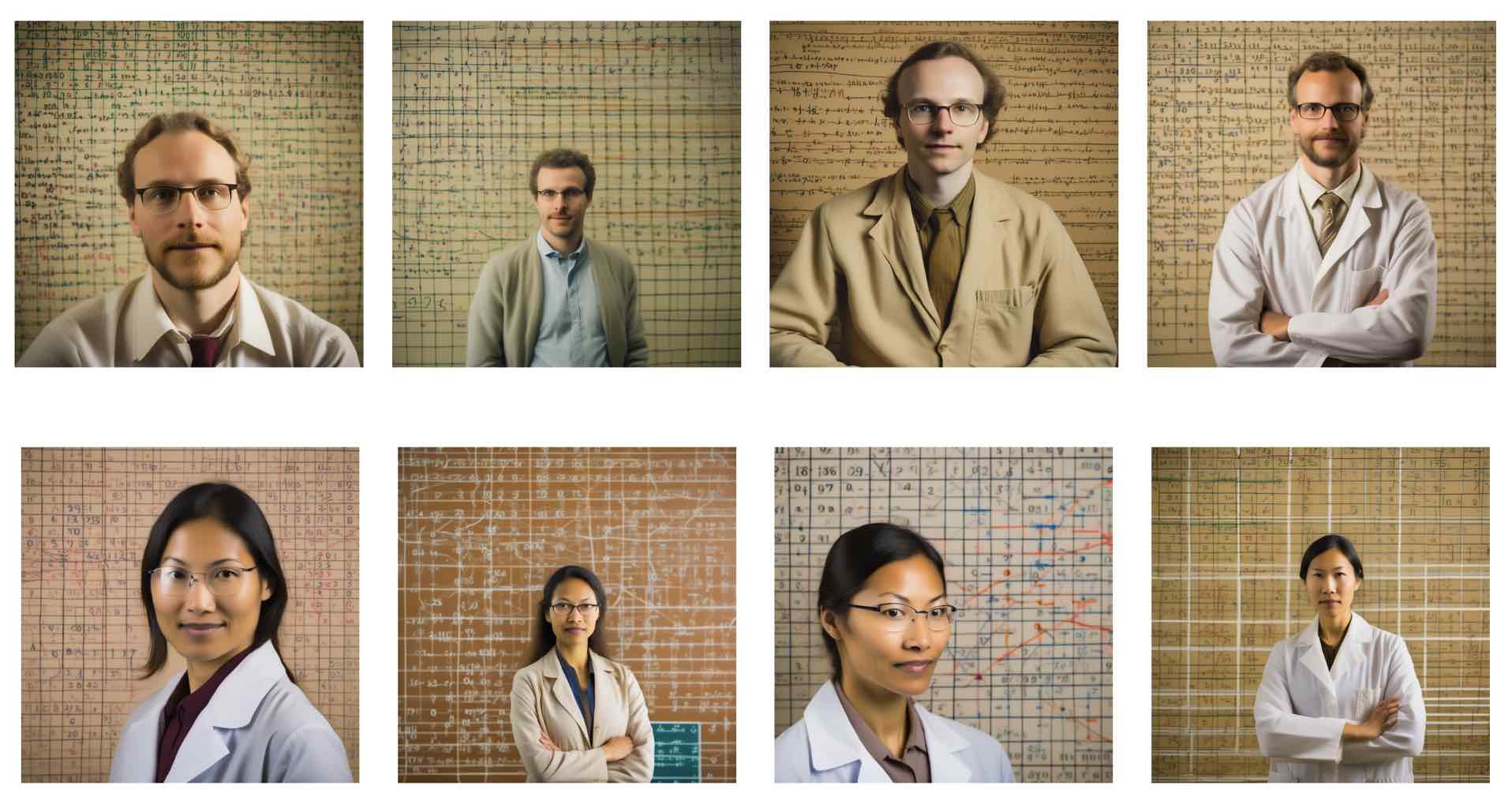}
  \caption{\textbf{Top}: Generations by the 1st auxiliary model alone, resembling the White male phenotype. \textbf{Bottom}: Generations by the 2nd auxiliary model, resembling the Asian female phenotype. These are model generations \textbf{without} using the KL barycenter sampler.}
  \label{fig:biased}
\end{figure}

\begin{figure*}[h]
    \centering
    \begin{subfigure}
        \centering     
        \includegraphics[width=0.99\columnwidth]{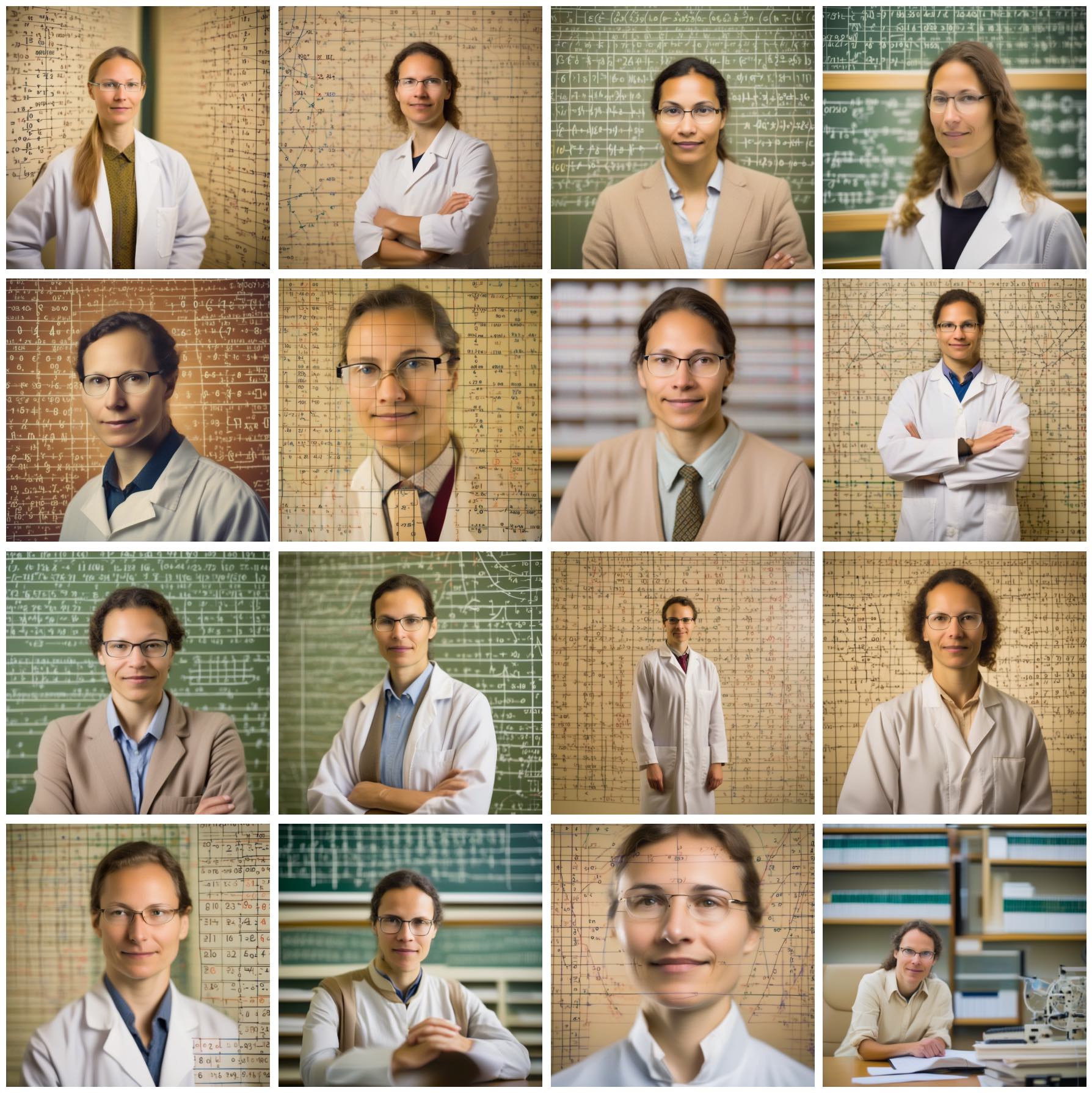}
    \end{subfigure}
    \hfill
    \begin{subfigure}
        \centering
        \includegraphics[width=0.99\columnwidth]{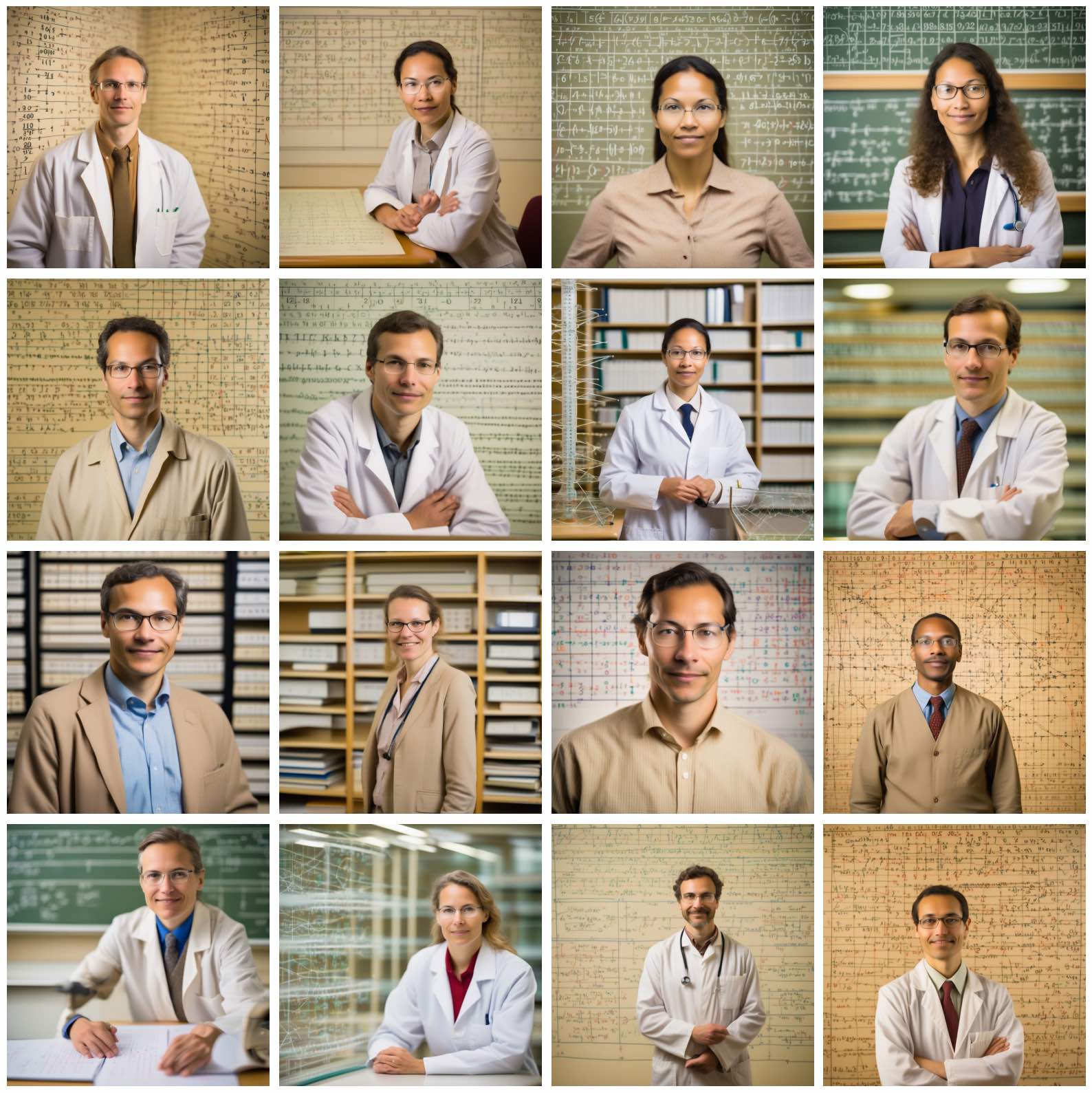}
    \end{subfigure}
    
    \caption{\textbf{Left:} i.i.d. samples from KL barycenter. \textbf{Right:} i.i.d. samples from checkpoint merging. \newline Interpolation weights are $\boldsymbol{\lambda} = (0.5, 0.5)$ for both. The same text prompt as Figure \ref{fig:biased} was used: \textit{``a photo of a mathematics scientist, looking at the camera, ultra quality, sharp focus''}. Both approaches enhance ethnic diversity relative to the monolithic representations in Figure \ref{fig:biased}, but the KL barycenter approach also produces samples that embody a more ambiguous and rarer representation of gender and ethnicity, mitigating stereotypes.}
    \label{fig:compare-4x4}
\end{figure*}

Since the MNIST experiment already tests the sample efficiency of training with ScoreFusion, we focus on investigating its inference-time behavior in the SDXL experiment. Specifically, we hardcode the barycenter weights $\boldsymbol{\lambda} \in \Delta^1$, use the same prompt as before, and sample images from the KL barycenter distribution of the two auxiliary models. As a baseline, we also sampled images from a model created from linearly merging the checkpoints of the two auxiliary models, as is the empirical practice in \cite{AUTOMATIC1111, biggs2024}.

The side-by-side in Figure \ref{fig:compare-4x4} suggests a caveat of \citet{biggs2024}'s Taylor expansion argument to approximate interpolation in the parameter space of the score network by the function space. Samples on the left are visibly different from those on the right, indicating that the error term introduced by Taylor approximation amplifies non-trivially after denoising diffusion. The observation lends further evidence that our approach fuses a distribution meaningfully distinct from that of weight-averaging two auxiliary models.

We also conduct an ablation study on the samples' sensitivity to $\boldsymbol{\lambda}$. Figure \ref{fig:compare-interpolation} qualitatively compares how KL barycenter and checkpoint merging respectively responds to variations in $\boldsymbol{\lambda}$. The KL barycenter shows a smooth spectrum of gender expressions and physical appearances that move beyond stereotypical portrayals, embodying qualities that may be read as gender-neutral. This distinction highlights our approach's ability to encourage the sampling from the \textit{tails} of the auxiliary distributions, which one could not efficiently obtain by resampling from the auxiliary models.

\begin{figure*}[htbp]
  \centering
  \includegraphics[width=0.95\textwidth]{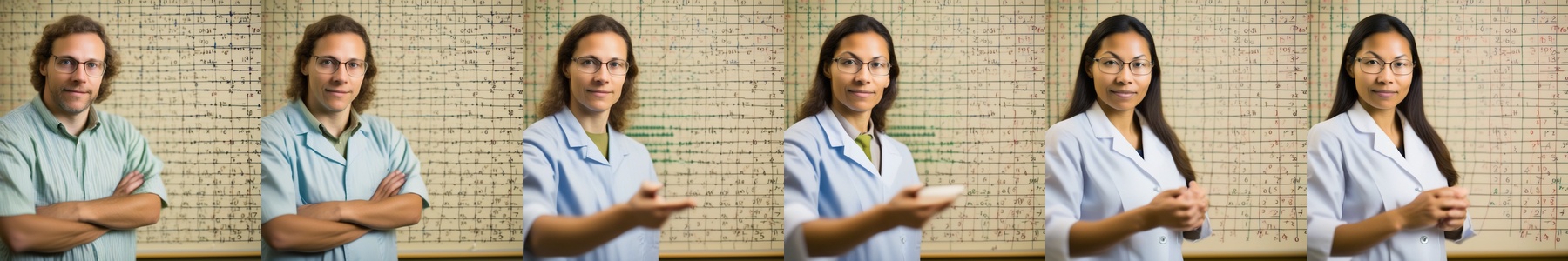}
  \includegraphics[width=0.95\textwidth]{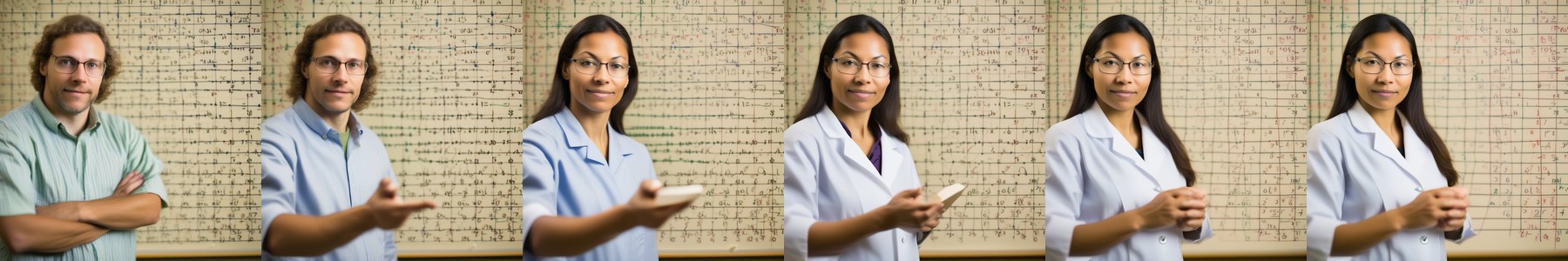}

  \caption{
    \textbf{Top row}: KL barycenter. 
    \textbf{Bottom row}: Checkpoint merging.\newline The same Gaussian noise was used to seed all twelve images, the only difference being the interpolation approach (top vs bottom) and interpolation weight; from left to right, $\lambda_2 \in \{0, 0.2, 0.4, 0.6, 0.8, 1.0\}$ and $\lambda_1 = 1 - \lambda_2$. $\lambda_2=0$ and $\lambda_2=1$ each reduce to an original auxiliary (biased) SDXL model. Observe that the bottom row samples show an abrupt identity shift between $\lambda_2 = 0.2$ and $0.4$, whereas the top row shows a smoother transition from one demographic visual concept to another.
    }
  \label{fig:compare-interpolation}
\end{figure*}

\begin{figure}[h]
    \centering
    \includegraphics[width=0.66\linewidth]{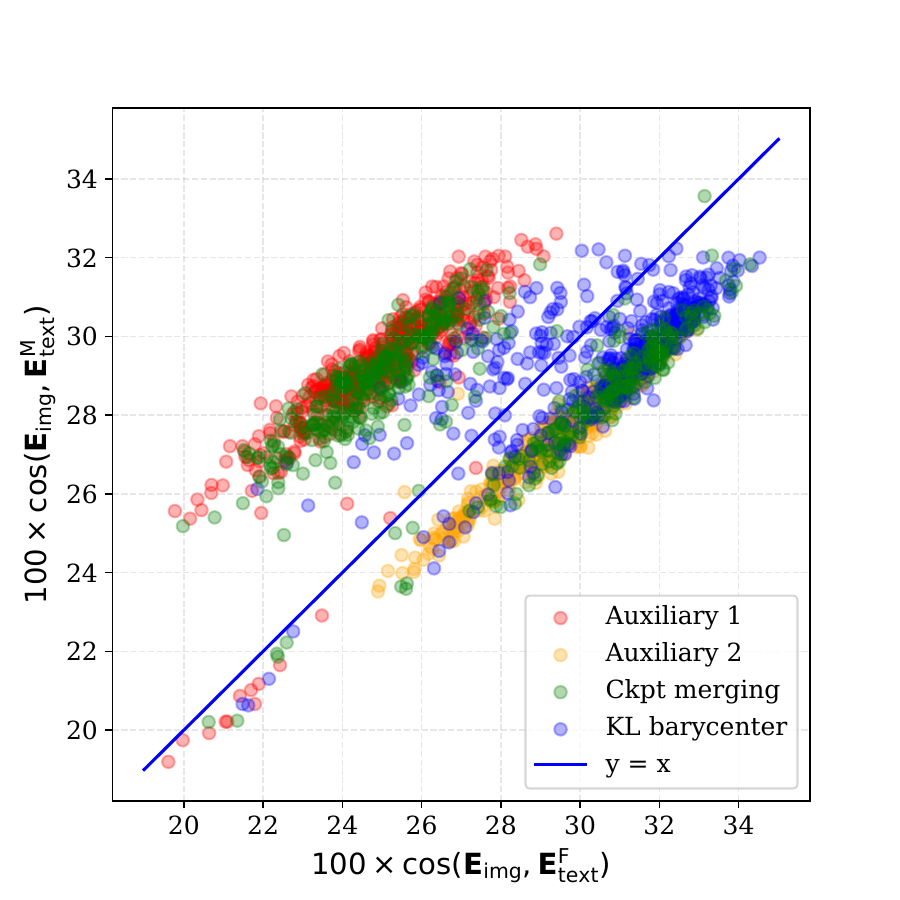}
    \caption{Empirical distribution of each model, projected onto a 2D gender semantic space. $x$ and $y$ coordinates are their CLIP scores. $\mathbf{E}_{img}$ are CLIP embeddings of each sample. $\mathbf{E}_{text}^{F}$ and $\mathbf{E}_{text}^M$ are embeddings of \textit{``a photo of a female scientist''} and \textit{``a photo of a male scientist''}. Gender neutrality and diversity can be interpreted as the middle region between the two auxiliary models' unimodal clusters; KL barycenter samples from this unexplored region, whereas checkpoint merging induces a bimodal mixtures distribution.}
    \label{fig:gender-scatter}
\end{figure}

Besides visual evaluations, we quantify gender and ethnic representations by using the CLIP encoder as a human surrogate \citep{clip}, running it on $512$ image samples generated by each method. Figure \ref{fig:gender-scatter} shows a joint 2D scatter plot of the CLIP distances of samples of different generators to two opposite gender concepts, confirming the qualitative comment for Figure \ref{fig:compare-4x4}. Figure \ref{fig:main-kdes} shows a KDE-smoothed CLIP scatter plot for the ablation study, complementing Figure \ref{fig:compare-interpolation}.
Implementation details, more image samples and CLIP plots for semantic probing are given in Appendix \ref{face}.

\begin{figure}[h]
    \centering
    \begin{subfigure}
        \centering
        \includegraphics[width=0.48\columnwidth]{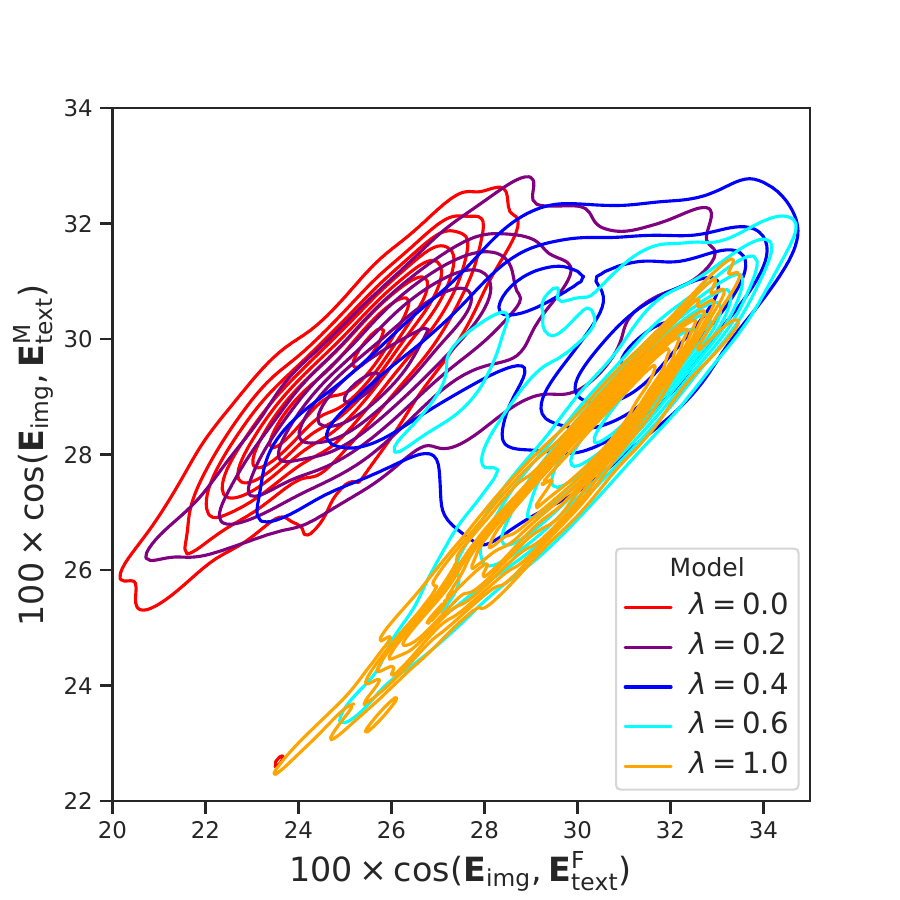}
    \end{subfigure}
    \hfill
    \begin{subfigure}
        \centering     
        \includegraphics[width=0.48\columnwidth]{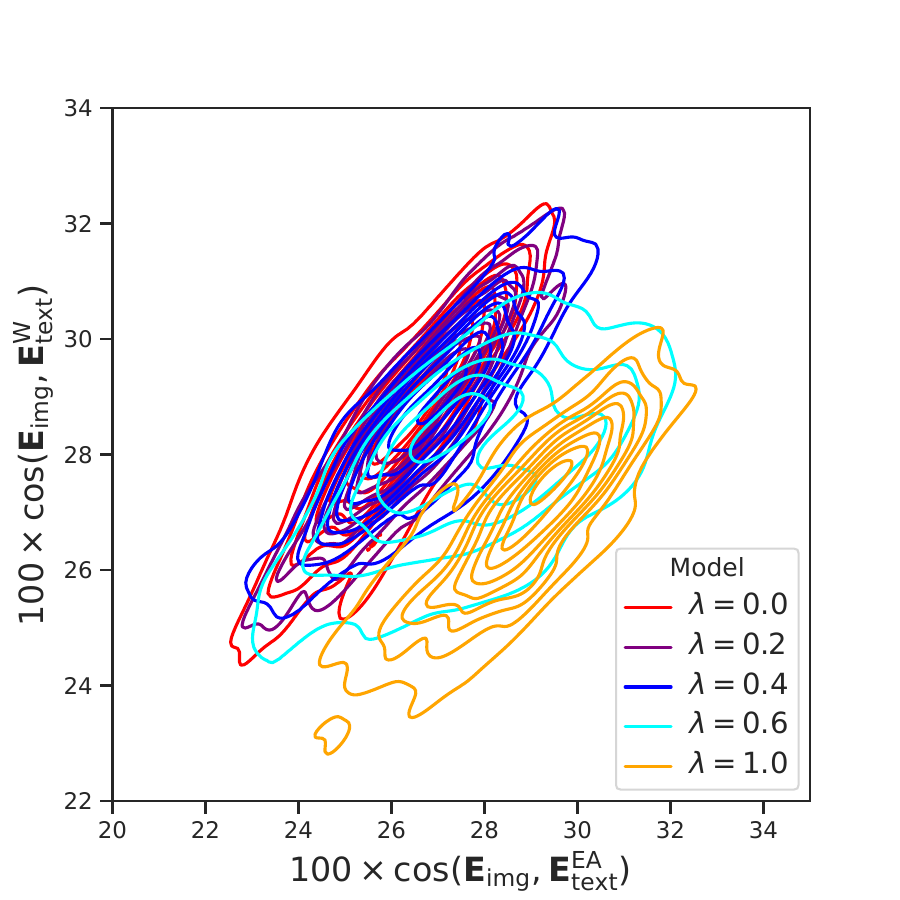}
    \end{subfigure}
    \caption{KDE contours of the KL barycenter distribution under various $\lambda_2$ (denoted as $\lambda$ in legends) values, estimated using a bandwidth of $0.8$. \textbf{Left}: $\mathbf{E}_{text}^{F}, \mathbf{E}_{text}^M$ are text embeddings of \textit{``a photo of a female scientist''} and \textit{``a photo of a male scientist''}.  \textbf{Right}: $\mathbf{E}_{text}^{EA}, \mathbf{E}_{text}^{W}$ are text embeddings of \textit{``a photo of an East Asian scientist''} and \textit{``a photo of a White scientist''}.}
    \label{fig:main-kdes}
\end{figure}

\section{CONCLUSION}
\label{sec:conclusion}
In this paper, we propose an ensemble method based on KL barycenter that can be easily implemented if the auxiliary score estimations are obtained from diffusion. Our method not only simplifies a parametric training in the low data regime, but also contributes a mathematically grounded algorithm for visual generative models. We provide a theoretical analysis of the sample complexity, showing that it is dimension-free given accurate auxiliary score estimation and closeness between optimal KL barycenter and the target distribution. The numerical experiments further demonstrate that our method performs well in the low data regime and show the difference between our method and checkpoint merging. This method can also extend to other gradient flow variants, which we leave for future work.
\section*{Acknowledgements}
The material in this paper is partly supported by the Air Force Office of Scientific Research under award number FA9550-20-1-0397 and ONR N000142412655. Support from NSF 2229012, 2312204, 2403007 is also gratefully acknowledged. J.Y. would like to thank Michael Y. Li, Yuhui Zhang, and Zeyu Wang for helpful feedback on the experiments. This work was partly conducted using the Stanford Yen Cluster, provided by the Data, Analytics, and Research Computing (DARC) group at the Stanford Graduate School of Business. We are also grateful
to the anonymous reviewers for suggesting numerous
improvements.

\bibliographystyle{plainnat}
\bibliography{ref}

\section*{Checklist}
\begin{enumerate}
 \item For all models and algorithms presented, check if you include:
 \begin{enumerate}
   \item A clear description of the mathematical setting, assumptions, algorithm, and/or model. [Yes]
   \item An analysis of the properties and complexity (time, space, sample size) of any algorithm. [Yes]
   \item (Optional) Anonymized source code, with specification of all dependencies, including external libraries. [Yes]
 \end{enumerate}

 \item For any theoretical claim, check if you include:
 \begin{enumerate}
   \item Statements of the full set of assumptions of all theoretical results. [Yes]
   \item Complete proofs of all theoretical results. [Yes]
   \item Clear explanations of any assumptions. [Yes]     
 \end{enumerate}

 \item For all figures and tables that present empirical results, check if you include:
 \begin{enumerate}
   \item The code, data, and instructions needed to reproduce the main experimental results (either in the supplemental material or as a URL). [Yes]
   \item All the training details (e.g., data splits, hyperparameters, how they were chosen). [Yes]
         \item A clear definition of the specific measure or statistics and error bars (e.g., with respect to the random seed after running experiments multiple times). [Yes]
         \item A description of the computing infrastructure used. (e.g., type of GPUs, internal cluster, or cloud provider). [Yes]
 \end{enumerate}

 \item If you are using existing assets (e.g., code, data, models) or curating/releasing new assets, check if you include:
 \begin{enumerate}
   \item Citations of the creator If your work uses existing assets. [Yes]
   \item The license information of the assets, if applicable. [Yes]
   \item New assets either in the supplemental material or as a URL, if applicable. [Not Applicable]
   \item Information about consent from data providers/curators. [Not Applicable]
   \item Discussion of sensible content if applicable, e.g., personally identifiable information or offensive content. [Not Applicable]
 \end{enumerate}

 \item If you used crowdsourcing or conducted research with human subjects, check if you include:
 \begin{enumerate}
   \item The full text of instructions given to participants and screenshots. [Not Applicable]
   \item Descriptions of potential participant risks, with links to Institutional Review Board (IRB) approvals if applicable. [Not Applicable]
   \item The estimated hourly wage paid to participants and the total amount spent on participant compensation. [Not Applicable]
 \end{enumerate}

 \end{enumerate}

\appendix

\onecolumn
\aistatstitle{Supplementary Materials}

Section \ref{More} gives additional information about the basic diffusion models, including the time reversal formulas, time discretization steps, and the current generalization error results. \ref{B} and \ref{C} provide missing proofs in Section \ref{Sec:KL_sol} and \ref{Sec4} of the main text, respectively. \ref{finance} discusses a related work in quantitative finance. \ref{sec:clarify} clarifies some experiment notations. \ref{sec:mnist-appendix} provides implementation details and additional results of the MNIST experiment. \ref{face} provides implementation details of the Stable Diffusion experiment, including the CLIP distance computation (\ref{subsub:CLIP}), the ablation study setup (\ref{sub:ablation}), more image samples (\ref{SDXL:more-samples}), and a heuristic mathematical explanation of the observation (\ref{sec:theory-inspired}). Lastly, \ref{sub:1D} provides results from a stylized Gaussian mixtures experiment. 

\section{MORE ABOUT BASIC DIFFUSION MODELS}\label{More}
\subsection{About the Time Reversal Formula}
Note that Equations (\ref{eq:genearl1}) and (\ref{eq:OU1}) are still represented as a \enquote{forward} processes. If we replace $W(t)$ by $\tilde{W}(t)$, where $\tilde{W}(t)$ is a standard $d$-dimensional Brownian motion which flows backward from time $T$ to 0, then Equation (\ref{eq:genearl1}) becomes 
\begin{equation*}
    d\hat{X}(t) = \left(f(T-t, \hat{X}(t)) - g^2(T-t) \nabla \log p_{T-t}\left(\hat{X}(t)\right)\right)dt + g(T-t)d\tilde{W}(t), \hat{X}(T) \sim p_T,
\end{equation*}
which is the reverse SDE presented in \citet{scoreSDE}.
Hence for the forward OU process, the reverse process has another representation by
\begin{equation}\label{continuous_rep}
    d\hat{X}(t) = \left( -a\hat{X}(t) - \sigma^2 \nabla \log p_{T-t}\left(\hat{X}(t)\right)\right)dt + \sigma d\tilde{W}(t), \hat{X}(T) \sim p_T.
\end{equation}

\subsection{Discretization and Backward Sampling}\label{Euler}
In this section, we follow the scheme in \citet{discretizationerror}.

Given $n$ samples $X_0^{(1)}, \ldots, X_0^{(n)}$ from $p_0$ (data distribution), we train a neural network with the loss function (\ref{trainloss}). Let $h > 0$ be the step size of the time discretization, and there are $N$ steps, hence $T = Nh$. We assume that for each time $l = 0, 1, \ldots, N$, the score estimation $s_{lh, \theta^*}$ of $\nabla \log p_t$ is obtained. In order to simulate the reverse SDE (\ref{eq:genearl1}), we first replace the score function $\nabla \log p_{T-t}$ with the estimate $s_{T-t, \theta^*}$. Next, for each $t \in [lh, (l+1)h]$, the value of this coefficient in the SDE at time $lh$, which yields the new time-discretized SDE with each $t \in [lh, (l+1)h]$,
\begin{equation}\label{timedis}
    d\hat{X}(t) = \left(-f(T-t, \hat{X}(t)) + g^2(T-t) s_{T-t, \theta^*}\left(\hat{X}_{kh}\right)\right)dt + g(T-t)dW(t)
\end{equation} and $\hat{X}(0) \sim \Pi$, where $\Pi$ is the (theoretical) stationary distribution of the forward process (\ref{eq:forwardSDE}).

There are several details in this implementation. In practice, when we use OU process as the forward, then Equation (\ref{timedis}) becomes 
    \begin{equation*}
        d\hat{X}(t) = \left( a\hat{X}(t) + \sigma^2 s_{T-t, \theta^*}\left(\hat{X}_{kh}\right)\right)dt + \sigma dW(t), t \in [lh, (l+1)h],
    \end{equation*}with $\Pi = \pi$, which is a linear SDE. In particular, $X_{(l+1)h}$ conditioned on $X_{lh}$ is Gaussian, so the sampling is easier.

In theory, we should use $\Pi \sim p_T$, which we have no access to. The above implementation takes advantage of $p_T \approx \Pi$ as $T$ is large enough. This introduces a small initialization error. 

\subsection{About the Generalization Error of Basic Diffusion Model}\label{randomfeature}
In \citet{generalizations}, a random feature model is considered as the score estimator. The basic intuition is that the generalization error with respect to the KL divergence, $D_{\text{KL}}\left(\mu\parallel \hat{\mu}\right)$ is decomposed into three terms: the training error, approximation error of underlying random feature model, and the convergence error of stationary measures. Among these three, the third one is ignorable since the fast rate of convergence of an OU process (or, from log Sobolev inequality for Gaussian random variables in \citet{logSobolev}). The first one is also small since random feature model in this setting is essentially linear regression with least squares.

Moreover, as stated in \citet{Universal}, random feature model can approximate Lipschitz functions with compact supports. However, the approximation error can be large and cause curse of dimensionality if we choose $m \sim n$. To illustrate this, we make a more general statement including smoothness considerations.

To be more precise, we introduce the following setting. We use the basic diffusion model with a forward OU process. The score function $s_{t, \theta}(x)$ is parameterized by the random feature model with $m$ random features:
\begin{equation*}
    s_{t, \theta}(x) = \frac{1}{m}A\sigma\left(Wx + Ue(t)\right) = \frac{1}{m}\sum_{j=1}^m a_j\sigma\left(w_j^Tx + u_j^Te(t)\right),
\end{equation*}
where $\sigma$ is the ReLU activation function, $A = (a_1, \ldots, a_m) \in \mathbb{R}^{d \times m}$ is the trainable parameters, $W = (w_1, \ldots, w_m)^T \in \mathbb{R}^{m \times d}$, $U = (u_1, \ldots, u_m)^T \in \mathbb{R}^{m \times d_e}$ are initially sampled from some pre-chosen distributions (related to random features) and remain frozen during the training, and $e: \mathbb{R}_+ \to \mathbb{R}^{d_e}$ is the time embedding function. The precise description is given below.

Assume that $a_j, w_j,$ and $u_j$ are drawn i.i.d. from a distribution $\rho$, then as $m \to \infty$, from strong law of large numbers, with probability 1, 
\begin{equation}\label{temp}
    s_{t, \theta}(x) \to \bar{s}_{t, \bar{\theta}}(x) = \mathbb{E}_{(w,u) \sim \rho_0}\left[  a(w,u)\sigma\left( 
w^Tx + u^Te(t) \right) \right],
\end{equation}
where $a(w,u) = \frac{1}{\rho_0(w,u)}\int a \rho(a,w,u)da$ and $\rho_0(w,u) = \int \rho(a,w,u)da$. From the positive homogeneity of ReLU function, we may assume $\left\lVert u\right\rVert + \left\lVert w\right\rVert \leq 1$. The optimal solution is denoted by $\bar{\theta^*}$ when replacing $s_{t, \theta}(x)$ in loss objective with $\bar{s}_{t, \bar{\theta}}(x)$.

Define a kernel $K_{\rho_0}(x,y) = \mathbb{E}_{(w,u) \sim \rho_0} \left[ \sigma\left( 
w^Tx + u^Te(t) \right)\sigma\left( 
w^Ty + u^Te(t) \right)  \right]$ and denote the induced reproducing kernel Hilbert space (RKHS) as $\mathcal{H}_{K_{\rho_0}}$; if there is no misunderstanding, we denote $\mathcal{H}:= \mathcal{H}_{K_{\rho_0}}$. It follows that $\bar{s}_{t, \bar{\theta}} \in \mathcal{H}$ if and only if $\left\lVert \bar{s}_{t, \bar{\theta}} \right\rVert_{\mathcal{H}} = \mathbb{E}_{(w,u) \sim \rho_0}  \left[ \left\lVert a(w,u)\right\rVert_2^2  \right] < \infty$. 

In \citet{Universal}, a notion of approximation quality called minimum width of the neural network is defined to measure the minimum number of random features needed to guarantee an accurate enough approximation with high probability. The exact definition is given below.
\begin{definition}
Given $\epsilon , \delta > 0$ and a function $f: \mathbb{R}^d \to \mathbb{R}$ with bounded norm $\left\lVert f\right\rVert_{\alpha} < \infty$, where $\alpha$ is the measure in $\mathbb{R}^d$ associated with the corresponding function space. We also denote $g^{(i)}(x) = \sigma\left( 
w^Tx + u^Te(t)\right) $. The minimum width $m_{f, \epsilon, \delta, \alpha, \rho_0}$ is defined to be the smallest $r \in \mathbb{Z}^+$ such that with probability at least $1-\delta$ over $g^{(1)}, \ldots, g^{(r)}$, 
\begin{equation*}
   \inf_{g \in \text{span}\left(  g^{(1)}, \ldots, g^{(r)} \right)}\left\lVert f - g\right\rVert_{\alpha} < \epsilon. 
\end{equation*}
\end{definition}

Moreover, for $s \geq 0$, $p \in [1,\infty]$, and $U \subset \mathbb{R}^d$ be an open and bounded set, $W^{s,p}(U)$ is the Sobolev space with order $s,p$ consists of all locally integrable function $f$ such that for each multiindex $\alpha$ with $|\alpha| \leq s$, weak derivative of $f$ exists and has finite $L^p$ norm (see \citet{Evans}). If $p = 2$, we denote $W^{s,2}(U) = H^s(U)$ to reflect the fact that it is a Hilbert space now. Finally, recall that the space of all Lipshitz functions on $U$ is the same as $W^{1, \infty}(U)$.

With these settings and definitions, we can state and prove the following generalization error for the basic diffusion model using random feature model.
\begin{theorem}
    Suppose that the target distribution $\mu$ is continuously differentiable and has a compact support, we choose an appropriate random feature $\rho_0$, and there exists a RKHS $\mathcal{H}$ such that $\bar{s}_{0, \bar{\theta^*}} \in \mathcal{H}$. Assume that the initial loss, trainable parameters, the embedding function $e(t)$ and the weighting function $\gamma(t)$ are all bounded. We further suppose that for all $t \in [0,T]$, the score function $\nabla \log p_t \in H^{s}(K) \cap W^{1, \infty}(K)$ and there exists $\gamma > 0$ such that $\left\lVert \nabla \log p_t\right\rVert_{H^s(K)} \leq \gamma$, where $K \subset \mathbb{R}^d$ is compact. Then for fixed $0 < \epsilon, \delta \ll 1$, with probability at least $1 - \delta$, we have 
    \begin{align*}
        D_{\text{KL}}\left( \mu || \hat{\mu} \right) &\lesssim \left( \frac{\tau^4}{m^3n} + \frac{\tau^2}{mn} + \frac{\tau^3}{m^2} + \frac{1}{\tau} + \frac{1}{m}  \right)\\ &+ \min \left( \left( 
\frac{s}{\log m} \right)^{s/2}, \left(  \frac{d\left(m^{1/d}-2\right)}{s\gamma^{2/s}}  \right)^{-s/2}\right) + D_{\text{KL}}\left( p_T || \pi \right), 
    \end{align*}
where $\tau$ is the training time (steps) in the gradient flow dynamics (see \citet{generalizations}), $m$ is the number of random features, $n$ is the sample size of the target distribution, $\pi$ is the stationary Gaussian distribution, $p_T$ is the distribution of the forward OU process at time $T$, $\mu$ is the target distribution, and $\hat{\mu}$ is the distribution of the generated samples.
\end{theorem}
\begin{proof}
The proof follows exactly the same as in the proof of Theorem 1 in \citet{generalizations}. The only extra work is to compute the universal approximation error of the random feature model for Sobolev functions on a compact domain. From compacted supported assumption (Lemma 1 in \citet{generalizations}), the forward process defines a random path $\left(X(t),t\right)_{t \in [0,T]}$ contained in a compact rectangular domain in $\mathbb{R}^{d+1}$. 

Theorem 35 in \citet{Universal} states the existence of a random feature $\rho_0$ such that for any $f \in H^s(K)$ with $\left\lVert f\right\rVert_{H^s(K)} \leq \gamma$, $m_{f, \epsilon, \delta, \alpha, \rho_0} \lesssim \frac{s^2 \gamma^{2+4/s}d^2}{\epsilon^{2+4/s}}\log\left(\frac{1}{\delta}\right)\exp\left( 
 \min \left( d\log\left(\frac{\gamma^2}{\epsilon^2d} + 2\right), \frac{\gamma^2}{\epsilon^2} \log\left( \frac{d\epsilon^2}{\gamma^2} + 2\right) \right)  \right)$, which implies the approximation error term.
\end{proof}

\begin{remark}The random feature model has two difficulties in implementation.

If $m$, $T$, and $\tau$ are large enough, then the generalization error is small regardless to the sample size $n$. However, the choice of random feature $\rho_0$ is hard in practice, especially in neither \citet{Universal} nor \citet{generalizations} the method to choose $\rho_0$ is specified. Therefore, the assumption that $\rho_0$ is appropriately chosen is very strong.
        
Even if $\rho_0$ is appropriately chosen, if we let $m \sim n$ and try to find an optimal early stopping time as in \citet{generalizations}, the term $\min \left( \left( 
\frac{s}{\log n} \right)^{s/2}, \left(  \frac{d\left(n^{1/d}-2\right)}{s\gamma^{2/s}}  \right)^{-s/2}\right)$ still dominates and shows the curse of dimensionality.
\end{remark}

\section{PROOF OF RESULTS IN SECTION \ref{Sec:KL_sol}}\label{B}
Before the proofs, we note the strict convexity of the KL barycenter problems via a simple lemma.
\begin{lemma}\label{unique}
    For any Polish space $S$, the KL barycenter problem $ \min_{\mu \in \mathcal{P}(S)} \sum_{i=1}^k \lambda_iD_{\text{KL}}\left(\mu \parallel P_i\right) \text{s.t.} \sum_{i=1}^{k}\lambda_i = 1$ is strictly convex.
\end{lemma}
\begin{proof}
 Let $t \in (0,1)$ and $\mu_1, \mu_2 \in S$ such that $\mu_1 \ll P_i$ and $\mu_2 \ll P_i$, for each $i = 1,2,\ldots, k$, then
 \begin{align*}
      \sum_{i=1}^k \lambda_i D_{\text{KL}}\left(  t\mu_1 + (1-t)\mu_2 \parallel P_i \right) &< \sum_{i=1}^k \lambda_i \left[ t D_{\text{KL}}\left(\mu_1 \parallel P_i\right) + (1-t)D_{\text{KL}}\left(\mu_2 \parallel P_i\right)\right]\\
      &= t\sum_{i=1}^k \lambda_i D_{\text{KL}}\left(\mu_1 \parallel P_i\right) + (1-t)\sum_{i=1}^k \lambda_iD_{\text{KL}}\left(\mu_2 \parallel P_i\right),
 \end{align*}where the inequality follows from the strictly convexity of KL divergence in terms of $\mu$ with fixed $P_i$. Therefore, the KL barycenter problem is strictly convex.
\end{proof}

\subsection{Proof of Theorem \ref{sol1}}\label{B1}
\begin{proof}
It suffices to consider a probability measure $\mu \in \mathcal{P}(\mathbb{R}^d)$ with absolutely continuous density $q(x)$ (otherwise the KL divergence is $\infty$) and show the existence. If there is no confusion, we use the density and measure interchangeably. We denote $\mathcal{P}_{\text{ac}}(\mathbb{R}^d)$ as the space of all absolutely continuous distributions and define a functional $F: \mathcal{P}_{\text{ac}}(\mathbb{R}^d) \to \mathbb{R}$ that for $x \in \mathbb{R}^d,$
\begin{equation*}
    F(q,x) = \sum_{i=1}^k \lambda_i q(x)\log\left( \frac{q(x)}{p_i(x)} \right).
\end{equation*}
Therefore, the barycenter problem becomes
\begin{equation*}
    \min_{\mu \in \mathcal{P}_{\text{ac}}(\mathbb{R}^d)} 
\int_{x \in \mathbb{R}^d} F(q,x)dx      \quad \text{s.t.} \sum_{i=1}^{k}\lambda_i = 1 \text{ and } \int_{x \in \mathbb{R}^d} q(x)dx = 1,
\end{equation*}which is a variational problem with a subsidiary condition (\citep{CalVariation}). 
Therefore, from calculus of variations, a necessary condition for $q$ to be an extremal of the variational problem is for some constant $m$
\begin{equation*}
    \frac{\partial}{\partial q}F(q) + m = 0.
\end{equation*}
Hence, the optimal solution is 
\begin{equation*}
    q^*(x) = \frac{\prod _{i=1}^k p_i(x)^{\lambda_i}}{\int \prod _{i=1}^k p_i(x)^{\lambda_i} dx}. 
\end{equation*}
\end{proof}

\subsection{Proof of Theorem \ref{sol2}}\label{B2}
Before the proof of Theorem \ref{sol2}, we review a consequence of Girsanov's theorem (Theorem 8 in \citet{discretizationerror}). We will use a similar technique as in \citet{discretizationerror}) to prove Theorem \ref{sol2}.

\begin{theorem}\label{Girsanov}
Suppose $Q \in \mathcal{P}(C([0, T ] : \mathbb{R}^d))$. For $t \in [0,T]$, let $\mathcal{L}(t) = \int_0^t b(s)dB(s)$ and the stochastic exponential $\mathcal{E}\left(\mathcal{L}\right)(t) = \exp\left( \int_0^t b(s)dB(s) -  \frac{1}{2}\int_0^t \left \Vert b(s)\right \Vert_2^2ds\right) $, where $B$ is a $Q$-Brownian motion. Assume $\mathbb{E}_{Q}\left[\int_0^T \left \Vert b(s)\right \Vert_2^2 ds\right] < \infty$. Then $\mathcal{L}$ is a square integrable $Q$-martingale. Moreover, if $\mathbb{E}_{Q}\left[ \mathcal{E}\left(\mathcal{L}\right)(T) \right] = 1,$ then $\mathcal{E}\left(\mathcal{L}\right)$ is a true $Q$-martingale and the process $B(t) - \int_0^t b(s)ds$ is a $P$-Brownian motion, where $P$ is a probabilty measure such that $P = \mathcal{E}\left(\mathcal{L}\right)(T) Q$.
\end{theorem}

In most applications of Girsanov's theorem, we need to check a sufficient condition to hold, known as Novikov's condition. In the context of Theorem \ref{Girsanov}, Novikov's condition is
\begin{equation}\label{checkGirsanov}
    \mathbb{E}_{Q}\left[\exp\left( \frac{1}{2}\int_0^T\left \Vert b(s)\right \Vert_2^2 ds \right)\right] < \infty.
\end{equation}

Now we begin the proof of Theorem \ref{sol2}.
\begin{proof}
From Lemma \ref{unique}, it suffices to show the existence. Let $\alpha \in \mathcal{P}(C([0,T]:\mathbb{R}^d))$ with initial distribution $\alpha_0$. We denote $\alpha(0)$ as the initial distribution of the process whose law is measure $\alpha$ as notation. From the chain rule of KL divergence, we have 
\begin{align*}
 \sum_{i=1}^k \lambda_iD_{\text{KL}}\left(\alpha \parallel P_i\right) &= \sum_{i=1}^k \lambda_i D_\text{KL}\left(\alpha_0 \parallel \mu_i\right)\\ &+ \mathbb{E}_{z \sim \alpha_0}\left[ \sum_{i=1}^k \lambda_i D_\text{KL} \left(        \alpha\left(.|\alpha(0) = z\right) \parallel  P_i\left(.|P_i(0) = z\right)    \right)  \right],
\end{align*}
where the first term solves the KL barycenter problem with respect to the initial distributions, and the second term solves the KL barycenter problem with all reference processes have the same initial distribution. 
Therefore, to finish the proof, we can assume for each $i = 1, \ldots, k$, $\mu_i \sim \mu$, the same initial distribution.

Since we are finding the minimizer of the weight sum of KL divergences, it is sufficient to assume that $\alpha$ is the law of a diffusion process which is a strong solution of an SDE with the same diffusion (volatility) coefficient as all reference processes:
\begin{equation*}
    dX(t) =  a\left(t, X(t)\right)dt + \sigma(t) dB(t), X(0) \sim \mu,
\end{equation*} where $B$ is a standard Brownian motion, and otherwise the KL divergence would be $\infty$. For now, we assume that $a(t,x)$ is uniformly bounded.

When applying Girsanov's theorem, it is more convenient to view different path measures on $\mathcal{P}(C([0,T]:\mathbb{R}^d)$ as the different laws of the same single stochastic process. For notational convenience, we denote the single process as $\{Z(t)\}_{t \in [0,T]}$.

For each $i = 1, \ldots, k$, we can apply the Girsanov's theorem to $Q = \alpha$ and 
\begin{equation*}
    b(t) = \frac{1}{\sigma(t)}\left(a_i(t, Z(t)) - a(t, Z(t))\right)
\end{equation*} in the setting of Theorem \ref{Girsanov}. Therefore, under the measure $P=\mathcal{E}\left(\mathcal{L}\right)(T) \alpha$, there exists a Brownian motion $\{\beta(t)\}_{t \in [0,T]}$ such that \begin{equation*}
    dB(t) = \frac{1}{\sigma(t)}\left(a_i(t, Z(t)) - a(t, Z(t))\right)dt + d\beta(t).
\end{equation*}
Since under the measure $\alpha$, with probability 1, 
\begin{equation*}
    dZ(t) =  a\left(t, Z(t)\right)dt + \sigma(t) dB(t), Z(0) \sim \mu,
\end{equation*}
then this also holds $P$-almost surely, which implies that $P$-almost surely, $ Z(0) \sim \mu$, and 
\begin{align*}
    dZ(t) &= a\left(t, Z(t)\right)dt  + \sigma(t)\left(\frac{1}{\sigma(t)}\left[  a_i(t, Z(t)) - a(t, Z(t))\right]dt + d\beta(t)\right)\\
    &=  a_i\left(t, Z(t)\right) dt + \sigma(t) d\beta(t).
\end{align*} In other words, $P \sim P_i$ in law.

Therefore, \begin{align*}
    D_{\text{KL}}\left(\alpha \parallel P_i\right) &= \mathbb{E}_{\alpha}\left[ \log\left( \frac{d\alpha}{dP_i} \right) \right]\\
    &= \mathbb{E}_{\alpha}\left[ \log\left(\frac{1}{\mathcal{E}\left( \mathcal{L} \right)(T)}\right) \right]\\
    &= \frac{1}{2} \mathbb{E}_{\alpha}\left[ \int_0^T \frac{1}{\sigma(t)^2} \left \Vert a_i(t, Z(t)) - a(t, Z(t)\right \Vert_2^2 dt \right] \\ &+ \mathbb{E}_{\alpha}\left[\int_0^T \frac{1}{\sigma(t)^2}\left(  a(t, Z(t)) - a_i(t, Z(t)\right) dt  \right]\\
    &= \frac{1}{2} \mathbb{E}_{\alpha}\left[ \int_0^T \frac{1}{\sigma(t)^2} \left \Vert a_i(t, Z(t)) - a(t, Z(t)\right \Vert_2^2 dt \right]
\end{align*}since Ito integral with regular integrand is a true martingale. 

Therefore, the objective function of process level KL barycenter problem becomes
\begin{equation*}
 \frac{1}{2} \sum _{i=1}^k \lambda_i \mathbb{E}_{\alpha}\left[ \int_0^T \frac{1}{\sigma(t)^2}\left \Vert a(t, Z(t)) - a_i(t, Z(t)\right \Vert_2^2 dt \right],    
\end{equation*}given we assume that all of reference laws have the same initial distribution.
Therefore, as a functional optimization problem, the minimizer $a^*(t,x) =  \sum_{i=1}^{k}\lambda_i a_i(t,x)$, which is indeed uniformly bounded and optimal, thus this finishes the proof.
\end{proof}

\section{PROOF OF RESULTS IN SECTION \ref{Sec4}}\label{C}
\subsection{Preliminaries and Basic Tools}
\subsubsection{Preliminaries}
We include this subsection to present basic definitions and notations used in our proofs.

\begin{definition}
$S$ is a Polish space equipped with Borel $\sigma$-algebra $\mathcal{B}(S)$, $\{P_n\}_{n \in \mathbb{N}} \subset \mathcal{P}(S)$ is a set of probability measures, we say $P_n$ converges to $P \in \mathcal{P}(S)$ weakly if and only if for each bounded and continuous function $f: S \to \mathbb{R}$, as $n \to \infty$,
\begin{equation*}
  \int_S f(x)dP_n(x) \to \int_S f(x)dP(x).  
\end{equation*}   
\end{definition}

\begin{definition}
Given two measurable spaces $\left(X, \mathcal{F}\right)$ and $\left(Y, \mathcal{G}\right)$, $f: X \to Y$ is a measurable function, and $\left(X, \mathcal{F}, \mu\right)$ is a (positive) measure space. The pushforward of $\mu$ is defined to be a measure $f_{\#}\mu$ such that for any $B \in \mathcal{G}$,
\begin{equation*}
    f_{\#}\mu(B) = \mu\left(f^{-1}(B)\right).
\end{equation*}
\end{definition}

\begin{definition}
A differentiable function $F: \mathbb{R}^d \to \mathbb{R}$ is called $L$-smooth if for any $x,y \in \mathbb{R}^d$, 
\begin{equation*}
    \lvert F(x) - F(y) - F'(y)(x-y) \rvert \ \leq \frac{L}{2}\left\Vert y-x\right\Vert_2^2.
\end{equation*}
\end{definition}

\begin{definition}
A stochastic process $\{X_t\}_{t \in [0,T]}$ is called a local martingale if there exists a sequence of nondecreasing stopping times $\{T_n\}_{n \in \mathbb{N}}$ such that $T_n \to T$ and $\{X_{t \wedge T_n}\}_{t \in [0,T]}$ is a true martingale.    
\end{definition}

Next we define some notations and stochastic processes that will be used in the following proofs. 

Recall the process (\ref{auxbackward}) is a backward SDE with score terms replaced by the estimations. We say for each $i = 1,2, \ldots, k$, process $\bar{X}_i$ is the theoretical backward process with exact score terms: 
\begin{equation}\label{eq: auxtheobackward}
    d\bar{X}_i(t) = \left( a\bar{X}_i(t) + \sigma^2\nabla \log p^i_{T-t}\left(\bar{X}_i(t)\right)\right)dt + \sigma dW_i(t), \bar{X}_i(0) \sim p^i_T.
\end{equation}
The corresponding forward process is denoted as $X_i$:
\begin{equation}\label{auxtheoforward}
    dX_i(t) = -a X_i(t)dt + \sigma dW(t), X_i(0) \sim p_i \sim \mu_i.
\end{equation}
We denote the marginal density of $X_i(t)$ as $p^i_t$; when $t = 0$, we use the notation $p_i \sim \mu_i$.
Process (\ref{pracbarycenter}) is a time-discretized SDE to be implemented in practice. It can be viewed as an approximation of the theoretical barycenter process (denoted as $\tilde{Y}$) of the backward SDEs of the form (\ref{eq: auxtheobackward}):
\begin{equation}\label{theobarycenterprocess}
d\tilde{Y}(t) = \left( a\tilde{Y}(t) + \sigma^2  \sum_{i=1}^k \lambda_i \nabla \log p^i_{T-t}\left(\tilde{Y}(t)\right)\right)dt + \sigma dW(t), \tilde{Y}(0) \sim \gamma^d_T,
\end{equation}where $\gamma^d_T$ is the distribution level KL barycenter at time $T$ with respect to the reference measures $\{p^1_T, \ldots, p^k_T\}$. When $T$ is large, $\gamma^d_T$ is approximated by $\pi$ in Equation (\ref{pracbarycenter}).
In theory, there is corresponding forward process with respect to process (\ref{theobarycenterprocess}): 
\begin{equation}\label{theoforwardbarycenter}
    dY(t) = -a X(t)dt + \sigma dW(t), Y(0) \sim \tilde{Y}(T).
\end{equation}

For a fixed $\boldsymbol{\lambda}$, we denote $p_{\boldsymbol{\lambda},t}$ as the marginal distribution of process (\ref{theoforwardbarycenter}) at time $t$; when $t=0$, we ignore the time subscript. 

\subsubsection{Basic Algorithms}
In this section, we recall the Frank-Wolfe method \citep{Frank-Wolfe}, which is used to solve an optimization problem with $L$-smooth convex function $f: \mathcal{X} \to \mathbb{R}$ on a compact domain $\mathcal{X}$: 
\begin{equation}\label{example}
    \min_{x \in \mathcal{X}}f(x)
\end{equation} 
\begin{algorithm}
\caption{(vanilla) Frank-Wolfe with function-agnostic step size rule \citep{Frank-Wolfe}}
\label{alg:fw}
\begin{algorithmic}[1]
\State \textbf{Input:} Start atom $x_0 \in \mathcal{X}$, objective function $f$, smoothness $L$
\State \textbf{Output:} Iterates $x_1, \ldots, x_{\tau} \in \mathcal{X}$
\For{$\tau = 0 \text{ to } \ldots$}
    \State $v_{\tau} \gets \arg\min_{v \in \mathcal{X}} \langle \nabla f(x_{\tau}), v \rangle$
    \State $\gamma_{\tau} \gets \begin{cases} 
1 & \text{if }  \tau = 1 \\
    \frac{2}{\tau + 3} & \text{if } \tau > 1 
\end{cases}$
    \State $x_{\tau+1} \gets x_{\tau} + \gamma_{\tau} (v_{\tau} - x_{\tau})$
\EndFor
\end{algorithmic}
\end{algorithm}
To measure the error of the algorithm, we define for each $\tau \geq 1$, the primary gap is
\begin{equation*}
    h_{\tau} = h(x_{\tau}) = f(x_{\tau}) - f(x^*),
\end{equation*}where $x^*$ is the minimizer of problem (\ref{example}).

\subsubsection{Basic Lemmas}\label{lemmaSec}
In this subsection, we first list some basic lemmas (Lemma \ref{10} to \ref{13}) that serve as essential tools in our proofs. All proofs can be found in \citep{discretizationerror}.

\begin{lemma}\label{10}
Suppose that Assumption \ref{compact} and \ref{Lipschitzanderror} hold. For each $i = 1,2,\ldots, k$, let $Z_i(t)$ denote the forward auxiliary process (\ref{auxtheoforward}), then for all $t \geq 0$,
\begin{equation*}
    \mathbb{E}\left[ \left \Vert Z_i(t)\right\Vert_2^2 \right] \leq d \vee M \text{ and }\mathbb{E}\left[ \left \Vert \nabla \log p^i_t\left(Z_i(t)\right)\right\Vert_2^2 \right] \leq Ld.
\end{equation*}
\end{lemma}

\begin{lemma}\label{11}
Suppose that Assumption \ref{compact} holds. For each $i = 1,2,\ldots, k$, let $Z_i(t)$ denote the forward auxiliary process (\ref{auxtheoforward}). For $0 \leq s < t$, let $\delta = t-s$. If $\delta \leq 1$, then
\begin{equation*}
    \mathbb{E}\left[ \left \Vert Z_i(t)- Z_i(s)\right\Vert_2^2 \right] \lesssim \delta^2M + \delta d.
\end{equation*}
\end{lemma}
 
\begin{lemma}\label{12}
Consider a sequence of functions $f_n: [0,T] \to \mathbb{R}^d$ and a function $f: [0,T] \to \mathbb{R}^d$ such that there exists a nondecreasing sequence $\{T_n\}_{n \in \mathbb{N}} \subset [0,T]$ such that $T_n \to T$ as $n \to \infty$ and for each $t \leq T_n$, $f_n(t) = f(t)$, then for each $\epsilon > 0$, $f_n \to f$ uniformly over $[0, T-\epsilon]$.    
\end{lemma}
 
\begin{lemma}\label{13}
 $f: [0,T] \to \mathbb{R}^d$ is a continuous function, and $f_{\epsilon}: [0,T] \to \mathbb{R}^d$ such that for each $\epsilon > 0$, $f_{\epsilon}(t) = f\left(t \wedge (T-\epsilon)\right)$, then as $\epsilon \to 0$, $f_{\epsilon} \to f$ uniformly over $[0,T]$.   
\end{lemma}

Next, we review and give two results related to the fusion algorithms.
\begin{lemma}\label{lemmagen}
    For any fixed $\boldsymbol{\lambda} \in \Delta_k$, $\tilde{Y}(T) \sim \mu_{\boldsymbol{\lambda}}$, the KL barycenter of $\{\mu_1, \ldots, \mu_k\}$.
\end{lemma}

\begin{proof}
In this proof, we use the following notations: suppose $x,y \in \mathbb{R}^d$ and $0 \leq s \leq t \leq T$, we denote $p^{i}(x,t |y,s)$ as the transition density of the $i$th auxiliary process from time $s$ to $t$. Similarly, $p^{\boldsymbol{\lambda}}(x,t |y,s)$ as the transition density of the barycenter process from time $s$ to $t$. 

Let $\boldsymbol{\lambda}$ be fixed, then at each time $t \in [0,T]$,
\begin{align*}
   \nabla \log \left(p_{\boldsymbol{\lambda}, t}(x) \right) = \nabla \sum _{i=1}^n \lambda_i \log \left(p_t^i(x)\right).
\end{align*}
Expanding LHS and RHS at the same time, we get
\begin{equation*}
\nabla \log \left( \int p^{\boldsymbol{\lambda}}(x,t|y,0)p_{\boldsymbol{\lambda}}(y)dy \right) = \nabla \sum _{i=1}^k \lambda_i \log \left( \int p^i(x,t|y,0)p_i(y)dy \right),    
\end{equation*}
Note that as $t \to 0$, $p^i(x,t|y,0) \to \delta(x-y)$ and $p^{\boldsymbol{\lambda}}(x,t|y,0) \to \delta(x-y)$, where the limit is the delta function. Therefore, from the compactness assumption and dominated convergence theorem,
\begin{align*}
\nabla \log p_{\boldsymbol{\lambda}}(x) &= \lim_{t \to 0}\nabla \log \left( \int p^{\boldsymbol{\lambda}}(x,t|y,0)p_{\boldsymbol{\lambda}}(y)dy \right)\\
    &=\displaystyle{\lim_{t \to 0}}\nabla \sum _{i=1}^k \lambda_i \log \left( \int p^i(x,t|y,0)p_i(y)dy \right)\\
    &= \nabla \sum _{i=1}^k \lambda_i \log p_i(x).
\end{align*}

Therefore, \begin{align*}
    \log p_{\boldsymbol{\lambda}}(x) &\propto \sum _{i=1}^k \lambda_i \log p_i(x) \\
    &= \log \left(  \prod _{i=1}^k p_i(x)^{\lambda_i} \right) \\ &=\log \left(  \prod _{i=1}^k p_i(x)^{\lambda_i} \right).
    \end{align*}

    Since $p_{\boldsymbol{\lambda}}(x)$ is a density function, then after normalization
    \begin{equation*}
        p_{\boldsymbol{\lambda}}(x) = \frac{\prod _{i=1}^k p_i(x)^{\lambda_i}}{\int \prod _{i=1}^k p_i(x)^{\lambda_i} dx},
    \end{equation*}which is the solution of KL barycenter problem with reference measures $p_1, \ldots, p_k$.
\end{proof}

Next we give the proof of Proposition \ref{lemma2}.
\begin{proof}
Recall that the objective function for $\boldsymbol{\lambda}$ is
\begin{equation}\label{temp}
    F(\boldsymbol{\lambda}) = \mathbb{E}_{\nu}\left[\log \nu(X) - \sum_{i=1}^{k}\lambda_i \log p_i(X)\right] + \log\left(  \int \prod _{i=1}^k p_i(y)^{\lambda_i}dy \right).
\end{equation}
We note that the first term is linear in $\boldsymbol{\lambda}$, so to show convexity, it is enough to show the second term is convex in $\boldsymbol{\lambda}$. If we denote $h_i(x) = \log\left(p_i(x)\right)$ for each $i = 1,2, \ldots, k$ and $X$ as the uniform distribution on $K$, then
\begin{align*}
\log\left(  \int \prod _{i=1}^k p_i(y)^{\lambda_i}dy \right) 
& = \log\left(  \int_{\mathbb{K}} \prod _{i=1}^k p_i(y)^{\lambda_i}dy \right)\\
& = \log\left(\frac{1}{|\mathbb{K}|}\int_{\mathbb{K}} \exp\left(\sum _{i=1}^k h_i(y)\lambda_i  \right)dy \right) + \log\left(|\mathbb{K}|\right)\\
&= \log\left(  \mathbb{E}\left[\exp\left(\boldsymbol{\lambda}^T Z\right) \right]  \right)+ \log\left(|\mathbb{K}|\right),
\end{align*}where $Z = \left(  h_1(X), \ldots, h_k(X)  \right)$ and $|\mathbb{K}|$ is the Lebesgue measure of $\mathbb{K}$. Since log of moment generating function is convex, then second term in Equation (\ref{temp}) is convex in $\boldsymbol{\lambda}$.    
\end{proof}

\begin{remark}\label{gradient}
In theory, the first order condition of the convex optimization problem (\ref{convex}) is
\begin{align*}
    \frac{\partial F}{\partial \lambda_i}(\boldsymbol{\lambda}) &=  - \int \nu(x) h_i(x)dx + \frac{\partial}{\partial \lambda_i} \log\left(  \int \prod _{l=1}^k p_l(y)^{\lambda_l}dy \right)\\
    &= -\mathbb{E}_{\nu}\left[h_i(X)\right] + \frac{\int \prod _{l=1}^k p_l(y)^{\lambda_l}\log p_i(y)dy}{\int \prod _{l=1}^k p_l(y)^{\lambda_l}dy}\\
    &= -\mathbb{E}_{\nu}\left[h_i(X)\right] + \frac{\int \exp\left( \sum _{l=1}^k\lambda_l h_l(y)\right)h_i(y)dy}{\int \exp\left( \sum _{l=1}^k\lambda_l h_l(y)\right)dy}.
\end{align*}    
In practice, each $h_i$ is replaced by the estimated auxiliary densities, and the second term is computed independent of the target data $\nu$. However, the implementation is extremely hard since the numerical integration of the second term may have large error and the error is hard to control.
\end{remark}

\subsection{Proof of Theorem \ref{err1}}\label{E1}
Before the proof of the sample complexity of the whole algorithm, we first prove a lemma about the auxiliary score estimation errors. The proof is adapted from \citet{discretizationerror}.
\begin{lemma}\label{backwardsampling}
Suppose that Assumption \ref{Lipschitzanderror} holds, $\boldsymbol{\lambda}$ is fixed, and the step size $h = T/N$ satisfies $h \lesssim 1/L$, where $L \geq 1$. Let $p_{\boldsymbol{\lambda}}$ and $\hat{p}_{\boldsymbol{\lambda}}$ denote the distribution of process (\ref{theobarycenterprocess}) and (\ref{pracbarycenter}) at time $T$, respectively. Then we have
\begin{equation*}
    \text{TV}\left(p_{\boldsymbol{\lambda}}, \hat{p}_{\boldsymbol{\lambda}} \right) \lesssim \exp(-T)\max_{i = 1,2,\ldots, k}\sqrt{D_{\text{KL}}\left(p^i_T \parallel \pi\right)} +  \sigma\sqrt{kT}\left(\epsilon_{\text{score}} + L\sqrt{dh} + Lh\sqrt{M}\right).
\end{equation*}
\end{lemma}

\begin{remark}
To interpret the result, suppose $\max_{i = 1,2,\ldots, k}\sqrt{D_{\text{KL}%
}\left(p^i_T \parallel \pi\right)} \lesssim \text{poly}(d)$ and $M \leq d$,
then for fixed $\epsilon$, if we choose $T \sim \log\left( \max_{i =
1,2,\ldots, k}\sqrt{D_{\text{KL}}\left(p^i_T \parallel \pi\right)}/\epsilon
\right)$ and $h \sim \frac{\epsilon^2}{L^2\sigma^2kd}$, and hiding the
logarithmic factors, then with $N \sim \frac{L^2\sigma^2kd}{\epsilon^2}$, $%
\text{SE} \lesssim \epsilon + \epsilon_{\text{score}}$. In particular, if we
want to choose the sampling error $\text{SE} \lesssim \epsilon$, it suffices
to have $\epsilon_{\text{score}} \lesssim \epsilon$.
\label{remark3}
\end{remark}

\begin{proof}
We denote the laws of process (\ref{theobarycenterprocess}) and (\ref{pracbarycenter}) as $\alpha$ and $\beta \in C([0, T ] : \mathbb{R}^d)$, respectively. For simplicity of the proof, we define a fictitious diffusion satisfying the SDE with $\hat{Y}(0) \sim \gamma^d_T$:
\begin{equation}\label{assist}
    d\hat{Y}(t) = \left( a\hat{Y}(t) + \sigma^2 \sum_{i=1}^k  \lambda_i s^i_{T-lh,\theta^*}\left(\hat{Y}(lh)\right) \right)dt + \sigma dW_i(t), t \in [lh, (l+1)h].
\end{equation}
since in practice, it is always convenient to use Gaussian $\pi$ as a prior. We denote law of process (\ref{assist}) as $\beta_T \in C([0, T ] : \mathbb{R}^d)$.

We also denote the score estimators of the process (\ref{theoforwardbarycenter}) as $s^{\boldsymbol{\lambda}}_{lh, \theta^*}$. Similarly as before, we consider only one stochastic process $Z(t)_{t \in [0,T]}$ now to use Girsanov's theorem.

For $t \in [lh, (l + 1)h]$, we have the discretization error $ \mathcal{L}$ with
\begin{align*}
    \mathcal{L}&= \sigma^2\mathbb{E}_{\alpha}\left[\left\lVert s^{\boldsymbol{\lambda}}_{T-lh, \theta^*}\left(Z(lh)\right) - \nabla \log p_{\boldsymbol{\lambda}, T-t}\left(Z(t)\right)\right\rVert_2^2\right] \\
 &= \sigma^2\mathbb{E}_{\alpha}\left[\left\lVert 
 \sum_{i=1}^k \lambda_i \left[ s^{i}_{T-lh, \theta^*}\left(Z(lh)\right) -\nabla \log p^{i}_{T-t}\left(Z(t)\right)\right]\right\rVert_2^2\right] \\
    & \lesssim \sigma^2\sum_{i=1}^k  \lambda_i^2\mathbb{E}_{\alpha}\left[\left\lVert  s^{i}_{T-lh, \theta^*}\left(Z(lh)\right) -\nabla \log p^{i}_{T-t}\left(Z(t)\right)\right\rVert_2^2\right] \\
    & \lesssim \sigma^2\sum_{i=1}^k \lambda_i^2\mathbb{E}_{\alpha}\left[\left\lVert  s^{i}_{T-lh, \theta^*}\left(Z(lh)\right) -\nabla \log p^{i}_{T-lh}\left(Z(lh)\right)\right\rVert_2^2\right] \\
    &+ \sigma^2\sum_{i=1}^k\lambda_i^2\mathbb{E}_{\alpha}\left[\left\lVert  \nabla \log p^{i}_{T-lh}\left(Z(lh)\right) -\nabla \log p^{i}_{T-t}\left(Z(lh)\right)\right\rVert_2^2\right] \\
    &+ \sigma^2\sum_{i=1}^k\lambda_i^2\mathbb{E}_{\alpha}\left[\left\lVert  \nabla \log p^{i}_{T-t}\left(Z(lh)\right) -\nabla \log p^{i}_{T-t}\left(Z(t)\right)\right\rVert_2^2\right]\\
    & \lesssim k\sigma^2 \left(\epsilon_{\text{score}}^2 + \mathbb{E}_{\alpha} \left[\left\lVert  \nabla \log \left( 
 \frac{p^{i}_{T-lh}}{p^{i}_{T-t}}\right)\left(Z(lh)\right)\right\rVert_2^2\right] + L^2 \mathbb{E}_{\alpha}\left[\left\lVert Z(lh) - Z(t)\right\rVert_2^2\right]\right).
\end{align*}
From Lemma 16 in \citet{discretizationerror}, we have the bound for the second term since $L \geq 1$,
\begin{align*}
    \mathbb{E}_{\alpha} \left[\left\lVert  \nabla \log \left( 
 \frac{p^{i}_{T-lh}}{p^{i}_{T-t}}\right)\left(Z(lh)\right)\right\rVert_2^2\right] &\lesssim L^2dh + L^2h^2\mathbb{E}_{\alpha}\left[\left\lVert Z(lh)\right\rVert_2^2\right] \\ &+ (1+L^2)h^2 \mathbb{E}_{\alpha}\left[\left\lVert \nabla \log p^{i}_{T-t} Z(lh)\right\rVert_2^2\right] \\
 &\lesssim L^2dh + L^2h^2\mathbb{E}_{\alpha}\left[\left\lVert Z(lh)\right\rVert_2^2\right] \\ &+ L^2h^2\mathbb{E}_{\alpha}\left[\left\lVert \nabla \log p^{i}_{T-t} Z(lh)\right\rVert_2^2\right].
\end{align*}
Moreover, from $L$-Lipschitz condition,
\begin{align*}
   \left\lVert \nabla \log p^{i}_{T-t} Z(lh)\right\rVert_2^2 &\lesssim \left\lVert \nabla \log p^{i}_{T-t} Z(t)\right\rVert_2^2 + \left\lVert \nabla \log p^{i}_{T-t} Z(lh) - \nabla \log p^{i}_{T-t} Z(t) \right\rVert_2^2 \\
   &\lesssim \left\lVert \nabla \log p^{i}_{T-t} Z(t)\right\rVert_2^2 + L^2\left\lVert Z(lh) -  Z(t)\right\rVert_2^2
\end{align*}
Hence, \begin{align*}
   \mathcal{L} &= \sigma^2\mathbb{E}_{\alpha}\left[\left\lVert s^{\boldsymbol{\lambda}}_{T-lh, \theta^*}\left(Z(lh)\right) - \nabla \log p_{\boldsymbol{\lambda},T-t}\left(Z(t)\right)\right\rVert_2^2\right]\\
   & \lesssim k\sigma^2 \epsilon_{\text{score}}^2 + k\sigma^2L^2dh + k\sigma^2L^2h^2\mathbb{E}_{\alpha}\left[\left\lVert Z(lh)\right\rVert_2^2\right]\\ &+ k\sigma^2L^2h^2\mathbb{E}_{\alpha}\left[\left\lVert \nabla \log p^{i}_{T-t} Z(t)\right\rVert_2^2\right] + k\sigma^2L^2 \mathbb{E}_{\alpha}\left[\left\lVert Z(lh) - Z(t)\right\rVert_2^2\right].
\end{align*}
From Lemma \ref{10} and Lemma \ref{11}, we have 
\begin{align*}
   \mathcal{L} &= \sigma^2\mathbb{E}_{\alpha}\left[\left\lVert s^{\boldsymbol{\lambda}}_{T-lh, \theta^*}\left(Z(lh)\right) - \nabla \log p_{\boldsymbol{\lambda},T-t}\left(Z(t)\right)\right\rVert_2^2\right]\\
   & \lesssim k\sigma^2\left( \epsilon_{\text{score}}^2 + L^2dh + L^2h^2\left(d + M\right) + L^3dh^2 + L^2 \left(dh + Mh^2\right)\right)\\
   & \lesssim k\sigma^2 \left(\epsilon_{\text{score}}^2 + L^2dh + L^2h^2M\right).
\end{align*}
Therefore,
\begin{align*}
    \mathcal{L} &= \sigma^2\sum_{l=0}^{N-1}\mathbb{E}_{\alpha}\left[\int_{lh}^{(l+1)h}\left\lVert s^{\boldsymbol{\lambda}}_{T-lh, \theta^*}\left(Z(lh)\right) - \nabla \log p_{\boldsymbol{\lambda},T-t}\left(Z(t)\right)\right\rVert_2^2dt\right]\\ &\lesssim \sigma^2 k T\left(\epsilon_{\text{score}}^2 + L^2dh + L^2h^2M\right).
\end{align*}
Next, we claim that 
\begin{equation}\label{claim2}
    D_{\text{KL}}\left(\alpha \parallel \beta_T\right) \lesssim k\sigma^2 T\left(\epsilon_{\text{score}}^2 + L^2dh + L^2h^2M\right).
\end{equation}
Then from triangle inequality, Pinsker's inequality, and data processing inequality, \begin{align*}
    \text{TV}\left(p_{\boldsymbol{\lambda}}, \hat{p}_{\boldsymbol{\lambda}} \right) 
    &\leq \text{TV}\left(\alpha ,\beta\right) \\
    &\leq \text{TV}\left(\beta,\beta_T\right) + \text{TV}\left(\alpha,\beta_T\right) \\
    &\leq \text{TV}\left(\pi, \gamma^d_T \right) + \text{TV}\left(\alpha,\beta_T\right)\\
    & \lesssim \exp(-T)\max_{i = 1,2,\ldots, k}\sqrt{D_{\text{KL}}\left(p^i_T \parallel \pi\right)} +  \sigma\sqrt{kT}\left(\epsilon_{\text{score}} + L\sqrt{dh} + Lh\sqrt{M}\right).
\end{align*}

Hence it suffices to prove Equation (\ref{claim2}). We will use a localization argument and apply Girsanov's theorem. The notations are the same as in Theorem \ref{Girsanov}.

Let $t \in [0,T]$, $\mathcal{L}(t) = \int_0^t b(s)dB(s)$, where $B$ is an $\alpha$-Brownian motion and for $t \in [lh, (l+1)h]$, 
\begin{equation*}
    b(t) = \sigma\left( s^{\boldsymbol{\lambda}}_{T-lh, \theta^*}\left(Z(lh)\right) - \nabla \log p_{\boldsymbol{\lambda},T-t}\left(Z(t)\right) \right).
\end{equation*}
Recall that 
\begin{equation*}
    \mathbb{E}_{\alpha}\left[ \int_0^T \left \Vert b(s) \right \Vert_2^2ds\right] \lesssim kT\sigma^2 \left(\epsilon_{\text{score}}^2 + L^2dh + L^2h^2M\right).
\end{equation*}
Since $\{\mathcal{E}\left(\mathcal{L}\right)(t)\}_{t \in [0,T]}$ is a local martingale, then there exists a non-decreasing sequence of stopping times $T_n \to T$ such that $\{\mathcal{E}\left(\mathcal{L}\right)(t \wedge T_n )\}_{t \in [0,T]}$ is a true martingale. Note that $\mathcal{E}\left(\mathcal{L}\right)(t \wedge T_n ) = \mathcal{E}\left(\mathcal{L}^n \right)(t)$, where $\mathcal{L}^n(t) = \mathcal{L}(t \wedge T_n)$, therefore
\begin{equation*}
    \mathbb{E}_{\alpha}\left[\mathcal{E}\left(\mathcal{L}^n \right)(T)\right] = \mathbb{E}_{\alpha}\left[\mathcal{E}\left(\mathcal{L}^n \right)(0)\right] = 1.
\end{equation*}
Applying Theorem \ref{Girsanov} to $\mathcal{L}^n(t) = \int_0^t b(s)\mathbf{1}_{[0, T_n]}(s)dB(s)$, we have that under the measure $P^n = \mathcal{E}\left(\mathcal{L}^n \right)(T)\alpha$, there exists a Brownian motion $\beta^n$ such that for all $t \in [0,T]$,
\begin{equation*}
    dB(t) =  \sigma\left( s^{\boldsymbol{\lambda}}_{T-lh, \theta^*}\left(Z(lh)\right) - \nabla \log p_{\boldsymbol{\lambda},T-t}\left(Z(t)\right) \right)\mathbf{1}_{[0, T_n]}(t)dt + d\beta^n(t). 
\end{equation*}
Since under $\alpha$ we have almost surely
\begin{equation*}
    dZ(t) = \left( aZ(t) + \sigma^2  \nabla \log p_{\boldsymbol{\lambda},T-t}\left(Z(t)\right)\right)dt + \sigma dB(t), Z(0) \sim \gamma^d,
\end{equation*}
which also holds $P^n$-almost surely since $P^n \ll \alpha$. Therefore, $P^n$-almost surely, $Z(0) \sim \gamma^d$ and
\begin{align*}
    dZ(t)&= \left[aZ(t) + \sigma^2 s^{\boldsymbol{\lambda}}_{T-lh, \theta^*}\left(Z(lh)\right)\right]\mathbf{1}_{[0, T_n]}dt \\ &+\left[aZ(t) \nabla \log p_{\boldsymbol{\lambda},T-t}\left(Z(t)\right) \right]\mathbf{1}_{[T_n, T]}dt +\sigma d\beta(t).
\end{align*}
In other words, $P^n$ is the law of the solution of the above SDE. Plugging in the Radon-Nikodym derivatives, we get 
\begin{align*}
    D_{\text{KL}}\left(\alpha \parallel P^n\right) &= \mathbb{E}_{\alpha}\left[ \log\left( \frac{d\alpha}{dP^n} \right) \right]\\
    &= \mathbb{E}_{\alpha}\left[ \log\left(\frac{1}{\mathcal{E}\left( \mathcal{L} \right)(T_n)}\right) \right]\\
    &= \mathbb{E}_{\alpha}\left[ -\mathcal{L}(T_n) + \frac{1}{2}\int_0^{T_n}\left \Vert b(s)\right \Vert_2^2ds \right]\\
    &= \mathbb{E}_{\alpha}\left[\frac{1}{2}\int_0^{T_n}\left \Vert b(s)\right \Vert_2^2ds \right]\\
    &\leq \mathbb{E}_{\alpha}\left[\frac{1}{2}\int_0^{T}\left \Vert b(s)\right \Vert_2^2ds \right]\\
    & \lesssim kT\sigma^2 \left(\epsilon_{\text{score}}^2 + L^2dh + L^2h^2M\right)
\end{align*}since $\mathcal{L}(T_n)$ is a martingale and $T_n$ is a bounded stopping time (apply optional sampling theorem).

Now consider a coupling of $\left(P^n\right)_{n \in \mathbb{N}}$, $\beta_T$: a sequence of stochastic processes $\left(Z^n\right)_{n \in \mathbb{N}}$ over the same probability space, a stochastic process $Z$ and a single Brownian motion $W$ over that space such that $Z(0) = Z^n(0)$ almost surely, $Z(0) \sim \gamma^d$,
\begin{align*}
  dZ^n(t) &= \left[aZ^n(t) + \sigma^2 s^{\boldsymbol{\lambda}}_{T-lh, \theta^*}\left(Z^n(lh)\right)\right]\mathbf{1}_{[0, T_n]}dt \\&+ \left[aZ^n(t) + \nabla \log p_{\boldsymbol{\lambda},T-t}\left(Z^n(t)\right) \right]\mathbf{1}_{[T_n, T]}dt +\sigma dW(t),  
\end{align*}
and 
\begin{equation*}
   dZ(t) = \left[aZ(t) + \sigma^2 s^{\boldsymbol{\lambda}}_{T-lh, \theta^*}\left(Z^n(lh)\right)\right]dt +\sigma dW(t).
\end{equation*}
Hence law of $Z^n$ is $P^n$ and law of $Z$ is $\beta_T$. The existence of such coupling is shown in \citet{discretizationerror}.

Fix $\epsilon > 0$, define the map $\pi_{\epsilon}: C([0, T ] : \mathbb{R}^d) \to C([0, T ] : \mathbb{R}^d)$ such that
\begin{equation*}
    \pi_{\epsilon}(\omega)(t) = \omega\left(t \wedge T - \epsilon \right).
\end{equation*}
Since for each $t \in [0, T_n]$, $Z^n(t) = Z(t)$, then from Lemma \ref{12}, we have $\pi_{\epsilon}\left(Z^n\right) \to \pi_{\epsilon}\left(Z\right)$ almost surely uniformly over $[0,T]$, which implies that $\pi_{\epsilon\text{} \#} P^n \to \pi_{\epsilon\text{} \#} \beta_T$ weakly.

Since KL divergence is lower semicontinuous, then from data processing inequality, we have
\begin{align*}
    D_{\text{KL}}\left(\pi_{\epsilon\text{} \#} \alpha \parallel \pi_{\epsilon\text{} \#} \beta_T \right) &\leq \liminf_{n \to \infty} D_{\text{KL}}\left(\pi_{\epsilon\text{} \#} \alpha \parallel \pi_{\epsilon\text{} \#} P^n \right)\\
    &\leq D_{\text{KL}}\left(\alpha \parallel P^n \right)\\
    & \lesssim kT \sigma^2\left(\epsilon_{\text{score}}^2 + L^2dh + L^2h^2M\right).
\end{align*}
From Lemma \ref{13}, as $\epsilon \to 0$, $\pi_{\epsilon}(\omega) \to \omega$ uniformly over $[0,T]$. Hence, from Corollary 9.4.6 in \citet{AGS05}, as $\epsilon \to 0$, $D_{\text{KL}}\left(\pi_{\epsilon\text{} \#} \alpha \parallel \pi_{\epsilon\text{} \#} \beta_T \right) \to D_{\text{KL}}\left(\alpha \parallel \beta_T\right)$. Therefore, from Pinsker's inequality,
\begin{equation*}
    D_{\text{KL}}\left(\alpha \parallel \beta_T\right) \lesssim kT \sigma^2\left(\epsilon_{\text{score}}^2 + L^2dh + L^2h^2M\right).
\end{equation*}
\end{proof}
Before the proof, we introduce some notations that will only be used for the proof of Theorem \ref{err1}. Recall that the vanilla fusion method requires two layers of approximation before running the Frank-Wolfe method: we use target samples to estimate an expectation and we also estimate the densities of auxiliaries. As a notation, we denote $\hat{\bar{p}}_{\hat{\boldsymbol{\lambda}}}$ as the distribution of the generated sample by vanilla fusion, which is $\hat{\nu}_D$ in Section \ref{Sec4}. $\boldsymbol{\hat{\lambda}}$ is the weight computed with $n$ target samples, $p_{\hat{\boldsymbol{\lambda}}}$ denotes the barycenter of $\{\mu_1, \ldots, \mu_k\}$ with the weight $\hat{\boldsymbol{\lambda}}$, and $\bar{p}_{\boldsymbol{\hat{\boldsymbol{\lambda}}}}$ denotes the barycenter of $\{\bar{p}_1, \ldots, \bar{p}_k\}$ with the weight $\hat{\boldsymbol{\lambda}}$, where $\{\bar{p}_1, \ldots, \bar{p}_k\}$ is the collection of estimates of auxiliary densities. Note that $\bar{p}_{\boldsymbol{\hat{\boldsymbol{\lambda}}}} \sim \hat{\mu}_{\boldsymbol{\lambda}}$ in Section \ref{Sec4}. 
\begin{proof}
From triangle inequality, we have \begin{align*}
    \text{TV}\left(\nu, \hat{\bar{p}}_{\hat{\boldsymbol{\lambda}}} \right) 
    &\leq \text{TV}\left(\nu, p_{\hat{\boldsymbol{\lambda}}} \right) + \text{TV}\left( p_{\boldsymbol{\hat{\boldsymbol{\lambda}}}}, \bar{p}_{\boldsymbol{\hat{\boldsymbol{\lambda}}}}\right) + \text{TV}\left(\bar{p}_{\boldsymbol{\hat{\boldsymbol{\lambda}}}}, \hat{\bar{p}}_{\hat{\boldsymbol{\lambda}}}\right)\\
    &:= I_1 + I_2 + I_3,
\end{align*}where $I_1$ represents the error when computing using the Frank-Wolfe method, $I_2 \leq \epsilon_2$ by assumption, and $I_3$ is the error from auxiliary score estimations, which is bounded by Lemma \ref{backwardsampling}.

Therefore it only remains to bound $I_1$. From Pinsker's inequality, 
\begin{equation*}
    I_1 = \text{TV}\left(\nu, p_{\hat{\boldsymbol{\lambda}}} \right) \lesssim \sqrt{D_{\text{KL}}\left(\nu \parallel p_{\hat{\boldsymbol{\lambda}}} \right)},
\end{equation*}hence it is enough to bound $D_{\text{KL}}\left(\nu \parallel p_{\hat{\boldsymbol{\lambda}}} \right)$. From the compactedness assumption, we note that the objective function $F$ of problem (\ref{convex}) is $\tilde{L}$-smooth for some constant $\tilde{L}$. Since the simplex in real space is convex, we denote the diameter of constrain set as $D$.

Recall that in Remark \ref{gradient} in Section \ref{lemmaSec}, the target gradient $\nabla F$ is given by
$$F_i(\boldsymbol{\lambda}) = -\mathbb{E}_{\nu}\left[h_i(X)\right] + \frac{\int \exp\left( \sum _{l=1}^k\lambda_l h_l(y)\right)h_i(y)dy}{\int \exp\left( \sum _{l=1}^k\lambda_l h_l(y)\right)dy}.$$
Thus we define $\hat{F}$ as the estimator of $F$ with the gradient estimated by
$$\nabla \hat{F}_i(\boldsymbol{\lambda}) = \frac{-1}{n}\sum _{j=1}^n h_i(x_j) + \frac{\int \exp\left( \sum _{l=1}^k\lambda_l h_l(y)\right)h_i(y)dy}{\int \exp\left( \sum _{l=1}^k\lambda_l h_l(y)\right)dy}, $$
where $x_j$ are i.i.d. samples from the target distribution $\nu$. Obviously $\hat{F}$ is also convex and $\tilde{L}$-smooth on a compact set with diameter $D$.

We denote $\hat{\boldsymbol{\lambda}}(\tau)$ as the weight computed after $\tau$ iterations with $n$ target samples, then from a standard Frank-Wolfe error analysis (e.g. Theorem 2.2 in \citet{Frank-Wolfe}), we have
$$\hat{F}(\hat{\boldsymbol{\lambda}}(\tau)) - \hat{F}(\boldsymbol{\lambda}^*) \leq \frac{2\tilde{L}D^2}{\tau + 3}.$$
From Fundamental Theorem of Calculus, there exists a curve $\gamma: [0,1] \to \Delta_k$ such that 
$$\gamma(0) = \hat{\boldsymbol{\lambda}}(\tau), \gamma(1) = \boldsymbol{\lambda}^*,$$
$$F(\hat{\boldsymbol{\lambda}}(\tau)) - F(\boldsymbol{\lambda}^*) = \int_{\gamma}\nabla F(z)dz,$$
and
$$\hat{F}(\hat{\boldsymbol{\lambda}}(\tau)) - \hat{F}(\boldsymbol{\lambda}^*) = \int_{\gamma}\nabla \hat{F}(z)dz,$$
where the right hand sides are line integrals.
From Hoeffding's inequality, for a fixed $z \in \Delta_k$, with probability at least $1-\delta$, 
$$\nabla F(z) \lesssim \nabla \hat{F}(z)+\mathcal{O}\left(\left(\log\left(\frac{1}{\delta}\right)\right)^{1/2}n^{-1/2}\right).$$
Therefore, 
\begin{align*}
  F(\hat{\boldsymbol{\lambda}}(\tau)) - F(\boldsymbol{\lambda}^*) &= \int_{\gamma}\nabla F(z)dz \\
  &\lesssim \int_{\gamma}\nabla \hat{F}(z)dz + \mathcal{O}\left(\left(\log\left(\frac{1}{\delta}\right)\right)^{1/2}n^{-1/2}\right)\\
  &= \hat{F}(\hat{\boldsymbol{\lambda}}(\tau)) - \hat{F}(\boldsymbol{\lambda}^*) + \mathcal{O}\left(\left(\log\left(\frac{1}{\delta}\right)\right)^{1/2}n^{-1/2}\right)\\
  &\leq \frac{2\tilde{L}D^2}{\tau + 3}  + \mathcal{O}\left(\left(\log\left(\frac{1}{\delta}\right)\right)^{1/2}n^{-1/2}\right).
\end{align*}
If we let $\tau \to \infty$, then
\begin{align*}
    D_{\text{KL}}\left(\nu \parallel p_{\hat{\boldsymbol{\lambda}}(\tau)} \right) &\leq  D_{\text{KL}}\left(\nu \parallel p_{\boldsymbol{\lambda}^*} \right) +  \mathcal{O}\left(\left(\log\left(\frac{1}{\delta}\right)\right)^{1/2}n^{-1/2}\right)\\
    & \lesssim \epsilon_0^2 + \mathcal{O}\left(\left(\log\left(\frac{1}{\delta}\right)\right)^{1/2}n^{-1/2}\right).
\end{align*}

Therefore, from Pinsker's inequality, with probability at least $1-\delta$,
\begin{align*}
    \text{TV}\left(\nu, \hat{\bar{p}}_{\hat{\boldsymbol{\lambda}}} \right) &\lesssim \epsilon_0 + \epsilon_2 + \exp(-T)\max_{i = 1,2,\ldots, k}\sqrt{D_{\text{KL}}\left(p^i_T \parallel \pi\right)} +  \sigma\sqrt{kT}\left(\epsilon_{\text{score}} + L\sqrt{dh} + Lh\sqrt{M}\right)\\ &+ \mathcal{O}\left(\left(\log\left(\frac{1}{\delta}\right)\right)^{1/4}n^{-1/4}\right). 
\end{align*}
\end{proof}

\subsection{Proof of Theorem \ref{err2}}\label{E2}
Before the proof, we define some notations that will be used in this proof. $\hat{p}_{\hat{\boldsymbol{\Lambda}}}$ denotes the output distribution of Algorithm \ref{Algo:2}, which is  $\hat{\nu}_P$ in Section \ref{Sec4}. For a fixed small $\tilde{T} \ll 1$, in the calibration phase of ScoreFusion, we denote the forward process as $Z$ (which is $\tilde{X}^{\nu}$ in Section \ref{Sec4}): for $t \in [0, \tilde{T}]$,
\begin{equation}\label{algo2forward}
    dZ(t) = -a Z(t)dt + \sigma dW(t), Z(0) \sim \nu. 
\end{equation}We learn an optimal weight by solving problem (\ref{newtrainloss}). We still denote the marginal distribution of process (\ref{algo2forward}) at time $t$ for fixed $\boldsymbol{\Lambda}$ as $p^{\nu}_t$. Even though in practice we do not use the backward process of process (\ref{algo2forward}), the following two versions of backward processes will help in the proof of Theorem \ref{err2}:  
for $t \in [0,\tilde{T}]$ with $\tilde{Z}(0) \sim \gamma^d_{\tilde{T}} \sim \hat{Z}(0)$, and fixed $\boldsymbol{\Lambda}$, 
\begin{equation}\label{fake1}
    d\tilde{Z}(t) = \left( a\tilde{Z}(t) + \sigma^2   \nabla \log p^{\nu}_{T-t}\left(\tilde{Z}(t)\right)\right)dt + \sigma dW(t), \tilde{Z}(\tilde{T}) \sim \nu,
\end{equation}
and for $l = 0, 1, \ldots, N_{\tilde{T}}$, 
\begin{equation}\label{fake2}
    d\hat{Z}(t) = \left( a\hat{Z}(t) + \sigma^2 \sum_{i=1}^k  \Lambda_i s^i_{T-lh,\theta^*}\left(\hat{Z}(lh)\right) \right)dt + \sigma dW(t), t \in [lh, (l+1)h],
\end{equation}where $hN_{\tilde{T}} = \tilde{T}$.
Process (\ref{fake2}) is the time-discretization version of process (\ref{fake1}) without the initialization error (since $\tilde{Z}(0) \sim \hat{Z}(0)$). We denote the law of process (\ref{fake1}) and (\ref{fake2}) as $\alpha_{\tilde{T}}$ and $\beta_{\tilde{T}} \in \mathcal{P}(C([0, T ] : \mathbb{R}^d))$, respectively. For fixed $\boldsymbol{\Lambda}$, we call $\tilde{Z}(\tilde{T}) \sim p^{\tilde{T}}_{\boldsymbol{\Lambda}}$ (which is in fact $\nu$) and $\hat{Z}(\tilde{T}) \sim \hat{p}^{\tilde{T}}_{\boldsymbol{\Lambda}}$.

\begin{proof}
From triangle inequality, we have
\begin{align*}
    \text{TV}\left(\nu, \hat{p}_{\hat{\boldsymbol{\Lambda}}} \right) &\leq
    \text{TV}\left(\nu, \hat{p}^{\tilde{T}}_{\hat{\boldsymbol{\Lambda}}} \right) +     \text{TV}\left(\hat{p}^{\tilde{T}}_{\hat{\boldsymbol{\Lambda}}}, p^{\tilde{T}}_{\hat{\boldsymbol{\Lambda}}} \right) + \text{TV}\left( p^{\tilde{T}}_{\hat{\boldsymbol{\Lambda}}} , p^{\tilde{T}}_{\boldsymbol{\Lambda}^*}\right) + \text{TV}\left( p^{\tilde{T}}_{\boldsymbol{\Lambda}^*}, p_{\boldsymbol{\Lambda}^*}\right) + \text{TV}\left( p_{\boldsymbol{\Lambda}^*}, p_{\hat{\boldsymbol{\Lambda}}}\right)+ \text{TV}\left( p_{\hat{\boldsymbol{\Lambda}}} , \hat{p}_{\hat{\boldsymbol{\Lambda}}}\right)\\
    & \lesssim \text{TV}\left(\nu, \hat{p}^{\tilde{T}}_{\hat{\boldsymbol{\Lambda}}} \right) + \text{TV}\left( p_{\hat{\boldsymbol{\Lambda}}} , \hat{p}_{\hat{\boldsymbol{\Lambda}}}\right) + \text{TV}\left(\nu, p_{\boldsymbol{\Lambda}^*}\right) + \text{TV}\left(p_{\hat{\boldsymbol{\Lambda}}}, p_{\boldsymbol{\Lambda}^*}\right)\\
    &\lesssim \text{TV}\left(\nu, \hat{p}^{\tilde{T}}_{\hat{\boldsymbol{\Lambda}}} \right) + \text{TV}\left( p_{\hat{\boldsymbol{\Lambda}}} , \hat{p}_{\hat{\boldsymbol{\Lambda}}}\right) + \epsilon_1 + \text{TV}\left(p_{\hat{\boldsymbol{\Lambda}}}, p_{\boldsymbol{\Lambda}^*}\right).
\end{align*}
From Lemma \ref{backwardsampling}, we bound the second term
\begin{equation*}
    \text{TV}\left( p_{\hat{\boldsymbol{\Lambda}}} , \hat{p}_{\hat{\boldsymbol{\Lambda}}}\right) \lesssim \exp(-T)\max_{i = 1,2,\ldots, k}\sqrt{D_{\text{KL}}\left(p^i_T \parallel \pi\right)} +  \sqrt{kT}\sigma\left(\epsilon_{\text{score}} + L\sqrt{dh} + Lh\sqrt{M}\right).
\end{equation*}
To bound the first term, we use a Girsanov's theorem and approximation argument similar as in Section \ref{E1} to get
\begin{align*}
D_{\text{KL}}\left(\nu \parallel \hat{p}^{\tilde{T}}_{\hat{\boldsymbol{\Lambda}}}\right) &\lesssim 
    D_{\text{KL}}\left(\alpha_{\tilde{T}} \parallel \beta_{\tilde{T}}\right)\\ &\lesssim  \frac{1}{\tilde{T}}\sum_{l=0}^{N_{\tilde{T}}-1}\mathbb{E}_{\alpha_{\tilde{T}}}\left[\int_{lh}^{(l+1)h}\sigma^2 \left\lVert s^{\boldsymbol{\Lambda}}_{T-lh, \theta^*}\left(Z(lh)\right) - \nabla \log p^{\nu}_{T- t}\left(Z(t)\right)\right\rVert_2^2dt\right]\\
 &\lesssim  \frac{1}{\tilde{T}} \int_0^{\tilde{T}}\left[\sigma^2\mathbb{E}_{Z(t) \sim p^{\nu}_t}\left[\left\lVert \sum_{i=1}^{k} \left( \Lambda_i s^i_{t, \theta^*}\left(Z(t)\right)\right) - \nabla \log p^{\nu}_{t}(Z(t))  \right\rVert_2^2 \right] \right]dt\\
  &\lesssim \tilde{\mathcal{L}}\left(\hat{\boldsymbol{\Lambda}};\theta^*,\sigma^2\right) = \tilde{\mathcal{L}}\left(\boldsymbol{\Lambda}^*;\theta^*, \sigma^2\right)  + \left[\tilde{\mathcal{L}}\left(\hat{\boldsymbol{\Lambda}};\theta^*, \sigma^2\right)- \tilde{\mathcal{L}}\left(\boldsymbol{\Lambda}^*;\theta^*, \sigma^2\right)\right]\\
  &:= I_1 + I_2,
\end{align*}where $I_1$ represents the approximation error and $I_2$ represents the excess risk. 

We note that from the bi-Lipschitz assumption and the compact support assumption,
\begin{align*}
D_{KL}\left(p_{\hat{\boldsymbol{\Lambda}}} \parallel p_{\boldsymbol{\Lambda}^*}\right) \sim \tilde{\mathcal{L}}\left(\hat{\boldsymbol{\Lambda}};\theta^*,\sigma^2\right),
\end{align*}hence we only need to bound $I_1$ and $I_2$ then.

From McDiarmid's inequality, for $\delta > 0$, with probability at least $1-\delta$, 
\begin{align*}
    I_2 \lesssim \mathcal{O}\left(\sigma^2\left(\log\left(\frac{1}{\delta}\right)\right)^{1/2}\left(N_{\tilde{T}}n\right)^{-1/2}\right) \lesssim \mathcal{O}\left(\sigma^2\left(\log\left(\frac{1}{\delta}\right)\right)^{1/2}n^{-1/2}\right) 
\end{align*} since $\tilde{T} \lesssim T$ and $N_{\tilde{T}}$ is small. 

Finally, we need to give a bound on $I_1$. The intuition is that from continuity of a diffusion process, when $h$ is small, then $p_h$ and $p_0$ are similar. Since the backward fused process is constructed as a process whose drift term is a linear combination of the auxiliary drifts, then the approximation error of the linear regression should be small, given Assumption \ref{optimalaux}. 

Fix $t \in [0, \tilde{T}]$, then from the Lipschitz and the compactedness assumption, the loss $\mathcal{L}$  is 
\begin{align*}
\mathcal{L} &= \tilde{\mathcal{L}}\left(\boldsymbol{\Lambda}^*;\theta^*, \sigma^2\right) \\
    &= \frac{\sigma^2}{\tilde{T}}\int_0^{\tilde{T}}\mathbb{E}_{Z(t) \sim p^{\nu}_t}\left[\left\lVert \sum_{i=1}^{k}\Lambda^*_i   s^i_{t, \theta^*}\left(Z(t)\right) - \nabla \log p^{\nu}_{t}(Z(t))  \right\rVert_2^2\right]dt\\ 
    &\lesssim \frac{\sigma^2}{\tilde{T}}\int_0^{\tilde{T}}\mathbb{E}_{Z(t) \sim p^{\nu}_t}\left[\left\lVert \sum_{i=1}^{k}\Lambda^*_i   s^i_{t, \theta^*}\left(Z(t)\right) - \sum_{i=1}^{k}\Lambda^*_i  \nabla \log p^i_t(Z(t))  \right\rVert_2^2\right]dt\\ &+\frac{\sigma^2}{\tilde{T}}\int_0^{\tilde{T}} \mathbb{E}_{Z(t) \sim p^{\nu}_t}\left[\left\lVert \sum_{i=1}^{k}\Lambda^*_i  \nabla \log p^i_t(Z(t))  - \nabla \log p^{\nu}_{t}(Z(t)) \right\rVert_2^2\right]dt\\
    &\lesssim \sigma^2k\epsilon_{\text{score}}^2 + \sigma^2 \mathbb{E}_{Z(0) \sim \nu}\left[\left\lVert \sum_{i=1}^{k}\Lambda^*_i  \nabla \log p^i_0(Z(0))  - \nabla \log p^{\nu}_{0}(Z(0)) \right\rVert_2^2\right]dt \\ &+ \frac{\sigma^2}{\tilde{T}}\int_0^{\tilde{T}}\mathbb{E}_{Z(t) \sim p^{\nu}_t}\left[\left\lVert \sum_{i=1}^{k}\Lambda^*_i  \nabla \log p^i_t(Z(t))  - \sum_{i=1}^{k}\Lambda^*_i  \nabla \log p^i_0(Z(t)) \right\rVert_2^2\right]dt\\
    &+\frac{\sigma^2}{\tilde{T}}\int_0^{\tilde{T}}\mathbb{E}_{Z(t) \sim p^{\nu}_t}\left[\left\lVert   \nabla \log p^{\nu}_t(Z(t))  -\nabla \log p^{\nu}_0(Z(t)) \right\rVert_2^2\right]dt\\
    & \lesssim \sigma^2k\epsilon_{\text{score}}^2 + \sigma^2\mathbb{E}_{Z(0) \sim \nu}\left[\left\lVert p^{\nu}_0(Z(0))  - p_{\boldsymbol{\Lambda}^*}(Z(0))  \right\rVert_2^2\right] \\ &+  \sigma^2\mathbb{E}_{Z(\tilde{T}) \sim \gamma^d_{\tilde{T}}}\left[\left\lVert p^{\nu}_{\tilde{T}}(Z(\tilde{T}))  - p^{\nu}_0(Z(\tilde{T}))  \right\rVert_2^2\right] + \max_{j = 1,2,\ldots,k}\sigma^2\mathbb{E}_{Z(\tilde{T}) \sim \gamma^d_{\tilde{T}}}\left[\left\lVert p^{j}_{\tilde{T}}(Z(\tilde{T}))  - p^{j}_0(Z(\tilde{T}))  \right\rVert_2^2\right] \\
    & \lesssim \sigma^2k\epsilon_{\text{score}}^2 + \sigma^2D_{\text{KL}}\left(\nu \parallel p_{\boldsymbol{\Lambda}}^*\right) + \max_{j = 1,2,\ldots,k}\sigma^2D_{\text{KL}}\left( p^{j}_{\tilde{T}} \parallel  p^{j}_0\right) \\
    &\lesssim \sigma^2k\epsilon_{\text{score}}^2 + \sigma^2 \epsilon_1^2 + \sigma^2k\mathcal{O}\left(\left(\tilde{T}\right)^{1/2}\right).
\end{align*}
Therefore, from Pinsker's inequality, with probability at least $1-\delta$,
\begin{align*}
\text{TV}\left(\nu, \hat{p}_{\hat{\boldsymbol{\Lambda}}} \right) 
& \lesssim \text{TV}\left(\nu, \hat{p}^{\tilde{T}}_{\hat{\boldsymbol{\Lambda}}} \right) + \text{TV}\left( p_{\hat{\boldsymbol{\Lambda}}} , \hat{p}_{\hat{\boldsymbol{\Lambda}}}\right) + \epsilon_1 + \text{TV}\left(p_{\hat{\boldsymbol{\Lambda}}}, p_{\boldsymbol{\Lambda}^*}\right)\\
& \lesssim \sqrt{D_{\text{KL}}\left(\nu \parallel \hat{p}^{\tilde{T}}_{\hat{\boldsymbol{\Lambda}}}\right)} + \text{TV}\left( p_{\hat{\boldsymbol{\Lambda}}} , \hat{p}_{\hat{\boldsymbol{\Lambda}}}\right) + \epsilon_1 + \sqrt{D_{\text{KL}}\left(p_{\hat{\boldsymbol{\Lambda}}} \parallel p_{\boldsymbol{\Lambda}^*}\right)}\\
& \lesssim (\sigma+1) \epsilon_1 + \sigma\sqrt{k}\mathcal{O}\left(\tilde{T}^{1/4}\right) + \mathcal{O}\left(\sigma\left(\log\left(\frac{1}{\delta}\right)\right)^{1/4}n^{-1/4}\right)\\ &+ \exp(-T)\max_{i = 1,2,\ldots, k}\sqrt{D_{\text{KL}}\left(p^i_T \parallel \pi\right)} +  \sigma\sqrt{kT}\left(\epsilon_{\text{score}} + L\sqrt{dh} + Lh\sqrt{M}\right),
\end{align*}which finishes the proof.
\end{proof}

\section{RELATED WORK IN FINANCE}\label{finance}
After the first draft of our paper, a concurrent work has been published by \citet{KLfinance}, which discusses the KL barycenter in the process level.  \citet{KLfinance} considers a constrained optimization problem in the finance setting of merging experts’ ideas, while our KL barycenter problem is unconstrained to fuse several auxiliary processes. The KL barycenter problem in \citet{KLfinance} imposes additional constrained conditions and coefficients of the SDEs. However, at this stage, our assumption is enough for the purpose of combining auxiliary models in the setting of generative models and is easy to integrate with neural network architecture. Our solution of the barycenter problem (Theorem \ref{sol2}) is equivalent to Proposition 2.4 in \citet{KLfinance} since the Radon-Nikodym derivative is 1 if our setting is plugged in. Moreover, the two papers discuss two different problems, use different approaches, and derive different results: in \citet{KLfinance}, the constrained optimization is solved via the dynamic programming approach and it is not related to statistics or machine learning, while in our case, we utilize the optimality of in the sense of KL barycenter to design a new machine learning method (ScoreFusion) and derive the sample complexity bound.

\section{TERMINOLOGY CLARIFICATION}\label{sec:clarify}
\begin{itemize}
    \item When we use the phrase ``diffusion model'' in this paper, or in discussing how to ``fuse'' a number of them, the underlying object is the U-Net \citep{unet} that parametrizes the score function of the time-reversed Ornstein–Uhlenbeck process, i.e., the $s_{t, \theta}(X)$ that parametrizes the $\nabla \log p_{t}(X)$ term in Equation \ref{continuous_rep}. Although our notation is based on the stochastic differential equation (SDE) view of diffusion modeling, we note that denoising diffusion probabilistic modeling (DDPM; \citet{ddpm}) can be formulated as a time-discretized version of SDE diffusion, as shown in \citet{scoreSDE}; the U-Net (i.e. $s_{t, \theta}(X)$) in SDE plays a similar role as the $\epsilon_{\theta}(X, t)$ function in DDPM.
    \item In this paper, we often used the term ``MNIST'' to refer to the digits dataset; but to be exact, what we meant is the extended version of the original MNIST dataset called EMNIST \citep{emnist}. It is curated by the same institution (NIST), but containing a significantly larger set of digit samples. Since MNIST is the better-known name for the digits dataset, we referred to it as such in the text. In our code base, however, you would see us referencing the dataset as \texttt{EMNIST}, its real name.
\end{itemize}

\section{MNIST DIGITS EXPERIMENT}\label{sec:mnist-appendix}

\subsection{Model Architecture}
Our parametrization of $s_{t, \theta}(X)$ employs a U-Net architecture composed of $8$ convolutional layers ($4$ for encoding and $4$ for decoding) with group normalization applied after each convolution, totaling $1M$ trainable parameters. Its code is modified from the \texttt{ScoreNet} class in the GitHub repository of \citet{scoreSDE}. The model also has skip connections between corresponding layers in the encoder and decoder. Temporal information is encoded via Gaussian random Fourier projections, followed by a dense layer to produce time embeddings, which are injected into each layer of the network through fully connected layers. This allows the model to condition its outputs on time steps effectively while maintaining multi-resolution feature extraction.

\subsection{Implementation Details}
Calibration of ScoreFusion (Part I of Algorithm \ref{Algo:2} in the main text) can be understood as optimizing the normalized weights of an additional linear layer that sum over the $s_{t, \theta_i}(X)$ score tensors. The PyTorch workflow is implemented by our \texttt{FusionNet} class in \texttt{src/model\_EMNIST.py}, provided in our Github repository. To ensure a fair comparison, the baselines and the auxiliary score models share the same U-Net architecture dimensions. The only difference between a Baseline 1 (B1) instance and an auxiliary instance is whether they are pre-trained or not. Each Baseline 2 (B2) instance is instantiated from a pre-trained auxiliary instance, and its parameters are fine-tuned on the target data until its test loss starts to increase.

We follow the standard machine learning convention of splitting each dataset into train, validation, and test sets with stratified sampling to ensure class balance. The ratio of training data to validation data is $4:1$. We use the ground truth digit labels only for data-splitting, hiding them from the model during training. Model training taking more than an hour was run on two NVIDIA A40 GPUs in a computing cluster, while lightweight tasks were run on Google Colab using one A100 GPU.

Auxiliary and baseline U-Nets use the same trainer function \texttt{generic\_train(*args)}, implemented in \texttt{src/model\_EMNIST.py}. Auxiliary models are trained from scratch on 25000 MNIST images. Baseline 1 instances are also trained from scratch, with training data ranging from $2^6$ up to $2^{14}$. At training time, we supply an additional $25\%$ data as the validation set. In short, auxiliary and baseline training share the training workflow, U-Net architecture, and score matching loss \texttt{loss\_fn(*args)} as defined in \texttt{src/model\_EMNIST.py}. Test evaluations are conducted on 5000 held-out images with a batch size of 200. The trainer contains the usual PyTorch pipeline of mini-batch + ADAM optimizer, plus three adaptations:

\begin{itemize}
    \item ADAM learning rate (\texttt{lr}) is 1e-4 for all auxiliary score training. Baseline 1 training follows the lr schedule in format (lr, [sample size using this lr]): {(1e-3, [64, 128]), (2e-4, [256, 512, 1024]), (1e-4, [4096, 16384])}. We set a higher lr for smaller datasets to accelerate loss decrease; empirically, it goes a lot slower on smaller datasets. The batch size is $min(128, TrainSize)$ for both training and validation loss.
    \item \texttt{ExponentialMovingAverage} (see \texttt{src/training.py}) is used for all checkpoint updating, the decay rate being a default $0.999$.
    \item Early-stopping of both baseline and auxiliary training is determined by the same algorithm to ensure consistency. A detailed description is given in below.
\end{itemize}

\paragraph{Early Stopping} To balance overfitting reduction and adequate learning, our experiments early-stop the model learning when validation loss exceeds the lowest realized validation loss by $\geq 50\%$ for more than $50$ consecutive epochs, implemented by the \texttt{EarlyStopper} class in \texttt{src/training.py}. We examine the fine-tuning / training loss curves of Baseline $2$ and $1$ U-Nets on $\leq 1000$ MNIST images and provide one such plot in Figure \ref{fig:training-loss}. These inspections suggest that there is no under-training of neither Baseline 1 nor Baseline 2.

\paragraph{Baseline 2 Setup} The auxiliary model we fine-tuned on generates an empirical digits distribution of \{‘\texttt{7}’: $72\%$, ‘\texttt{9}’: $24\%$, others: $4\%$\} over 10,000 samples. Out of the four pre-trained auxiliaries, this one is chosen because it was trained from the $(70\%, 30\%)$ frequency and therefore is the closest in frequency to the target distribution of \{‘\texttt{7}’: $60\%$, ‘\texttt{9}’: $40\%$\}; intuitively, it should be the easiest to finetune among the four. Using full parameter fine-tuning, we initialize the score weights with the chosen auxiliary, and update all weights every step. We use the ADAM optimizer with \texttt{lr} = 2e-5. \texttt{Max\_epoch} is set at 200, and rest of the training/test hyperparameters remain the same as stated as above (e.g. early-stopping criterion, exponential moving average).

\subsection{Additional Statistics \& Samples}\label{subsub:mnist-stats}

Tables \ref{tab:aux-hd-breakdown} and \ref{tab:aux-hd-breakdown-a2} provide additional information about the outcome of the MNIST experiments. Figure \ref{fig:training-loss} shows the typical finetuning / training loss dynamics for Baselines 1 and 2. Lastly, Figures \ref{fig:mnist-samples} and \ref{fig:mnist-naive} show i.i.d. digits samples generated by the ScoreFusion-trained model versus those of the two baseline methods.

\begin{table}[htbp]
  \centering
  \caption{Full digits proportions of 1024 images sampled from each of the four unadapted auxiliary models, classified by a SpinalNet \citep{spinalnet}. For simplicity, in the main text we have combined the proportions of digits other than \texttt{7} and \texttt{9} into one meta class named \texttt{Others}.}
    \begin{tabular}{crrrrrrrrrr}
    \Xhline{1.2pt} 
    Auxiliary & \multicolumn{1}{c}{\texttt{0}} & \multicolumn{1}{c}{\texttt{1}} & \multicolumn{1}{c}{\texttt{2}} & \multicolumn{1}{c}{\texttt{3}} & \multicolumn{1}{c}{\texttt{4}} & \multicolumn{1}{c}{\texttt{5}} & \multicolumn{1}{c}{\texttt{6}} & \multicolumn{1}{c}{\texttt{7}} & \multicolumn{1}{c}{8} & \multicolumn{1}{c}{9} \\
    \midrule
    1     & \multicolumn{1}{c}{0.1\%} & \multicolumn{1}{c}{0.1\%} & \multicolumn{1}{c}{0.6\%} & \multicolumn{1}{c}{0.6\%} & \multicolumn{1}{c}{1.1\%} & \multicolumn{1}{c}{0.3\%} & \multicolumn{1}{c}{0.0\%} & \multicolumn{1}{c}{18.7\%} & \multicolumn{1}{c}{0.2\%} & \multicolumn{1}{c}{78.2\%} \\
    2     & \multicolumn{1}{c}{0.1\%} & \multicolumn{1}{c}{0.1\%} & \multicolumn{1}{c}{0.3\%} & \multicolumn{1}{c}{0.8\%} & \multicolumn{1}{c}{1.1\%} & \multicolumn{1}{c}{0.5\%} & \multicolumn{1}{c}{0.0\%} & \multicolumn{1}{c}{41.1\%} & \multicolumn{1}{c}{0.2\%} & \multicolumn{1}{c}{55.8\%} \\
    3     & \multicolumn{1}{c}{0.0\%} & \multicolumn{1}{c}{0.2\%} & \multicolumn{1}{c}{0.7\%} & \multicolumn{1}{c}{0.7\%} & \multicolumn{1}{c}{1.2\%} & \multicolumn{1}{c}{0.8\%} & \multicolumn{1}{c}{0.0\%} & \multicolumn{1}{c}{72.1\%} & \multicolumn{1}{c}{0.6\%} & \multicolumn{1}{c}{23.7\%} \\
    4     & \multicolumn{1}{c}{0.1\%} & \multicolumn{1}{c}{0.5\%} & \multicolumn{1}{c}{0.7\%} & \multicolumn{1}{c}{0.5\%} & \multicolumn{1}{c}{0.9\%} & \multicolumn{1}{c}{0.4\%} & \multicolumn{1}{c}{0.1\%} & \multicolumn{1}{c}{87.9\%} & \multicolumn{1}{c}{0.3\%} & \multicolumn{1}{c}{8.6\%} \\
    \midrule
    Target Distribution &       &       &       &       &       &       &       & 60\%  &       & 40\% \\
    \Xhline{1.2pt} 
    \end{tabular}%
  \label{tab:aux-hd-breakdown}%
\end{table}%

\begin{table}[htbp]
  \centering
  \caption{Optimal weights $\boldsymbol{\lambda}^*$ corresponding to the ScoreFusion models whose NLL test losses we reported in Table \ref{nll}. Each column is a weight vector that parameterizes the ScoreFusion model trained with $2^j$ data.}
    \begin{tabular}{crrrrrrr}
    \Xhline{1.2pt} 
    $\boldsymbol{\lambda}_i$ & \multicolumn{1}{c}{$2^{6}$} & \multicolumn{1}{c}{$2^{7}$} & \multicolumn{1}{c}{$2^{8}$} & \multicolumn{1}{c}{$2^{9}$} & \multicolumn{1}{c}{$2^{10}$} & \multicolumn{1}{c}{$2^{12}$} & \multicolumn{1}{c}{$2^{14}$} \\
    \midrule
    $i=1$     & \multicolumn{1}{c}{0.199} & \multicolumn{1}{c}{0.187} & \multicolumn{1}{c}{0.182} & \multicolumn{1}{c}{0.181} & \multicolumn{1}{c}{0.167} & \multicolumn{1}{c}{0.183} & \multicolumn{1}{c}{0.176}  \\
    $i=2$   & \multicolumn{1}{c}{0.305}  & \multicolumn{1}{c}{0.326} & \multicolumn{1}{c}{0.328} & \multicolumn{1}{c}{0.319} & \multicolumn{1}{c}{0.345} & \multicolumn{1}{c}{0.311} & \multicolumn{1}{c}{0.310} \\
    $i=3$     & \multicolumn{1}{c}{0.279} & \multicolumn{1}{c}{0.267} & \multicolumn{1}{c}{0.284} & \multicolumn{1}{c}{0.285} & \multicolumn{1}{c}{0.319} & \multicolumn{1}{c}{0.294} & \multicolumn{1}{c}{0.295} \\
    $i=4$     & \multicolumn{1}{c}{0.217} & \multicolumn{1}{c}{0.220} & \multicolumn{1}{c}{0.206} & \multicolumn{1}{c}{0.216} & \multicolumn{1}{c}{0.170} & \multicolumn{1}{c}{0.213} & \multicolumn{1}{c}{0.220} \\
    \Xhline{1.2pt} 
    \end{tabular}%
  \label{tab:aux-hd-breakdown-a2}%
\end{table}%

\begin{figure*}[htbp]
    \centering
    \begin{subfigure}
        \centering
\includegraphics[width=0.47\textwidth]{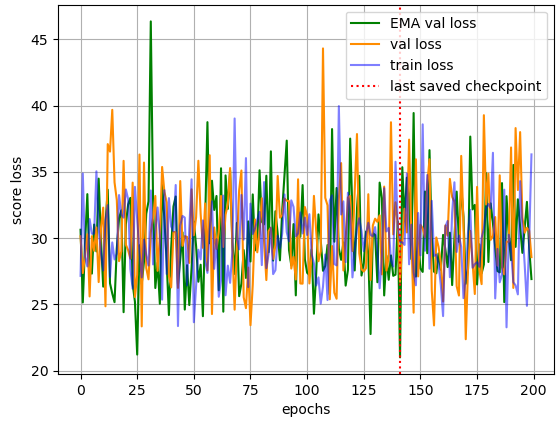}
    \end{subfigure}
    \hfill
    \begin{subfigure}
        \centering     \includegraphics[width=0.485\textwidth]{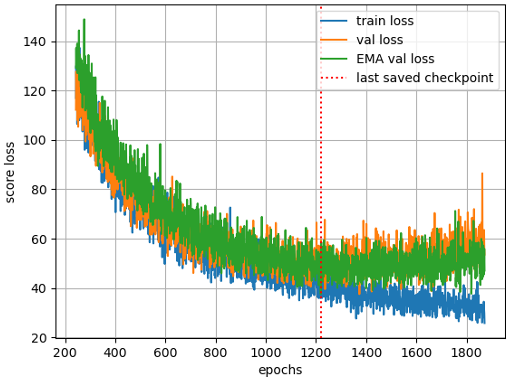}
    \end{subfigure}
    \caption{\textbf{Left}: Typical fine-tuning loss curves for Baseline 2. Train data size + Val data size $= 64+16 = 80$. \textbf{Right}: Baseline 1 training loss curves. Train data size + Val data size = 512+128 =640. EMA = Exponential Moving Average. Dotted line is the epoch where EMA checkpoint had the lowest loss on the validation data.}
    \label{fig:training-loss}
\end{figure*}

\begin{figure*}[htbp]
    \centering
    \begin{subfigure}
        \centering
\includegraphics[width=0.48\textwidth]{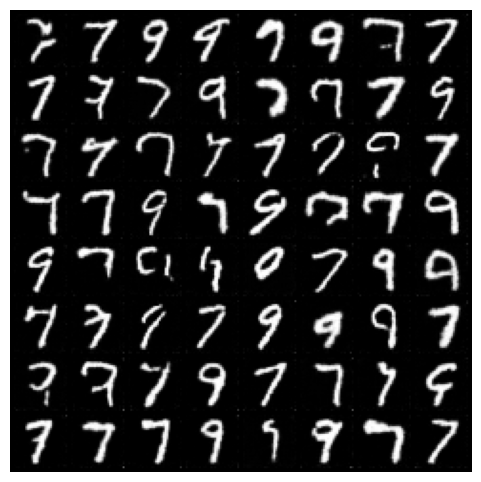}
    \end{subfigure}
    \hfill
    \begin{subfigure}
        \centering     \includegraphics[width=0.48\textwidth]{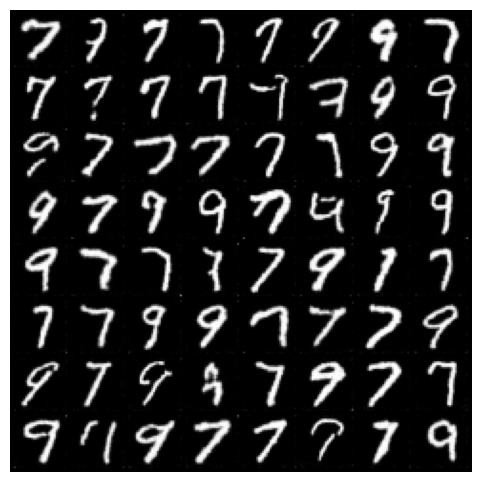}
    \end{subfigure}
    \caption{\textbf{Left}: uncurated samples generated by Baseline 2, obtained from directly fine-tuning a 70-30 auxiliary model. \textbf{Right}: uncurated samples generated by calibrating a ScoreFusion model. Both \textbf{Left} and \textbf{Right} models were fine-tuned/calibrated on the same $64$ images from the target population.}
    \label{fig:mnist-samples}
\end{figure*}

\begin{figure}[ht]
    \centering
    \includegraphics[width=0.95\textwidth]{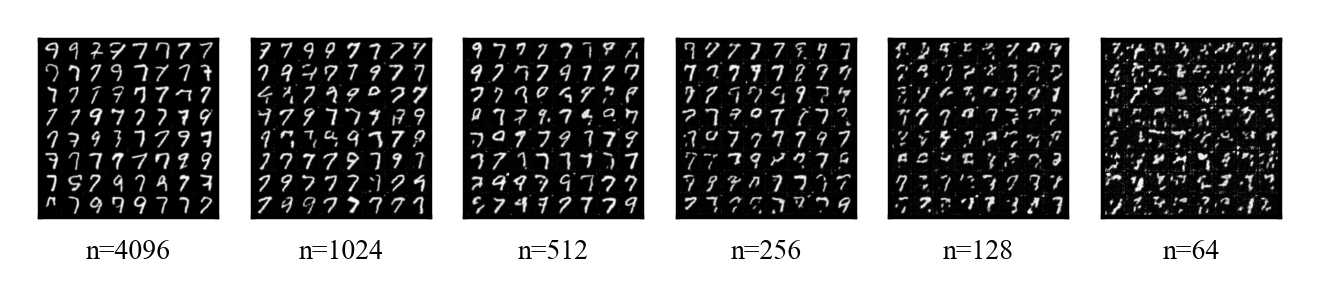}
    \caption{Samples generated by Baseline 1 under different quantities of training data.}
    \label{fig:mnist-naive}
\end{figure}

\section{SDXL PROFESSIONAL PORTRAITS EXPERIMENT}\label{face}

\subsection{Model Architecture}

All generative models in our stable diffusion experiment are derived from the open-source SDXL 1.0 base \citep{sdxl1_release}. As a type of latent diffusion model \citep{sd}, SDXL is composed of a variational autoencoder (VAE), a text encoder and a 2D conditioned U-Net. The U-Net consists of a combination of convolutional and cross-attention blocks. It operates across three resolution scales, with cross-attention integrated into both the downsampling and upsampling paths. The model conditions on text and time-step information via cross-attention, using positional embeddings for time encoding. According to Hugging Face's documentation \citep{lora_tutorial}, fine-tuned models in the stable diffusion community are almost always produced with DreamBooth \citep{dreambooth}, which outputs a LoRA adapter (Low-Rank Adaptation; \citet{lora}) of only the U-Net model, leaving the text encoder and the VAE unchanged. Intuitively, the LoRA adapter adds low-rank perturbations to the dense layers in a base U-Net so as to minimize the empirical score-matching loss on the fine-tuning dataset. 

The two auxiliary diffusion models used in our experiment are also products of LoRA fine-tuning; their U-Nets were each fine-tuned from a gender- and race-homogenoeous dataset by \citet{af_lora, wm_lora}, such that the unconditional generation (unspecified conditions being a person's gender and race) of professional portraits creates image samples that predominantly mirror the monolithic phenotype of the corresponding finetuning population. This design choice is intended to stylize the issue of distributional social bias in popular text-to-image (TTI) models, a problem widely recognized in the AI alignment literature \citep{dalleval, luccioni2024}. The sub-problem addressed by our SDXL experiment is: given TTI model checkpoints that are individually fine-tuned on a distinct subpopulation, sampling from the KL barycenter provides a theoretically grounded approach to organically blend heterogeneous features in the subpopulations into one generative model.

\subsection{Implementation Details}

Since Appendix \ref{sec:mnist-appendix} and Section \ref{sub:MNIST} have thoroughly tested the sample efficiency of the training phase of ScoreFusion and reflected our theoretical development, we focus on exploring the inference-time / sampling benefit of our proposed KL barycenter approach in the SDXL experiment. Hence we assumed that the barycenter weights $\boldsymbol{\lambda}$ are fixed at $(0.5, 0.5)$ in this experiment. Our goal is to probe the ability of KL barycenter sampling to sample from low-probability regions (with respect to the auxiliary models' sampling distributions) in the pixel space; we provide more theoretical motivation in section \ref{sec:theory-inspired}.

To implement this sampling, we revised the \href{https://github.com/huggingface/diffusers/blob/main/src/diffusers/pipelines/stable_diffusion_xl/pipeline_stable_diffusion_xl.py}{source code} of the \textit{Diffusers} library that implements the denoising loop of SDXL 1.0, allowing us to ensemble score evaluations of the two auxiliary models at inference time. 

\subsubsection{Sampling from KL-Divergence Barycenter}\label{sec:sdxl-kl}

Because the \texttt{StableDiffusionXLPipeline} class in Hugging Face's \texttt{Diffusers} library does not natively support the ensembling of two U-Nets denoisers during inference, we adapted their source code so that it can ensemble the log-probability gradient tensors for a specified set of barycenter weights. The new pipeline class can be inspected in \texttt{src/SDXL\_inference.py} in our Github repository. As mentioned in the previous subsection, the two auxiliary models share the exact same VAE and text encoder. Therefore, the only component of the generation pipeline that we combine at inference time is the U-Net's outputs.

Our hyperparameter setting follows the default values of a general \texttt{Diffusers} inference pipeline; checkpoints are loaded at \texttt{fp16} precision, CFG guidance scale is $5.0$, guidance rescale is zeroed, latent noise has shape 4x128x128, and the final tensor output has shape 4x1024x1024. The number of denoising steps is set to $100$, and the timesteps are given by the default discrete Euler scheduler. We used the same text prompts across all generations: \{`Positive Prompt': \textit{``a photo of a mathematics scientist, looking at the camera, ultra quality, sharp focus''}, `Negative Prompt': \textit{``cartoon, anime, 3d, painting, b\&w, low quality.''}\}

\subsubsection{Checkpoint Merging Model}

For comparison, we also generated samples from a diffusion model that results from equally merging LoRA checkpoints of the two auxiliary models. Note that because the base SDXL model is the same for both models, merging the LoRA checkpoints and appending the merged checkpoint to the base model is the same as merging the entirety of the two auxiliary model checkpoints. As stated in the main text, checkpoint merging is a common model adaptation method in the online stable diffusion community. Figure \ref{fig:automatic1111} shows the checkpoint merging UI of the popular image generation tool \texttt{stable-diffusion-webui} \citep{AUTOMATIC1111}. The inference-time hyperparameter setting is the same as the one stated in Section \ref{sec:sdxl-kl}.

\begin{figure}[ht]
    \centering
    \includegraphics[width=0.7\textwidth]{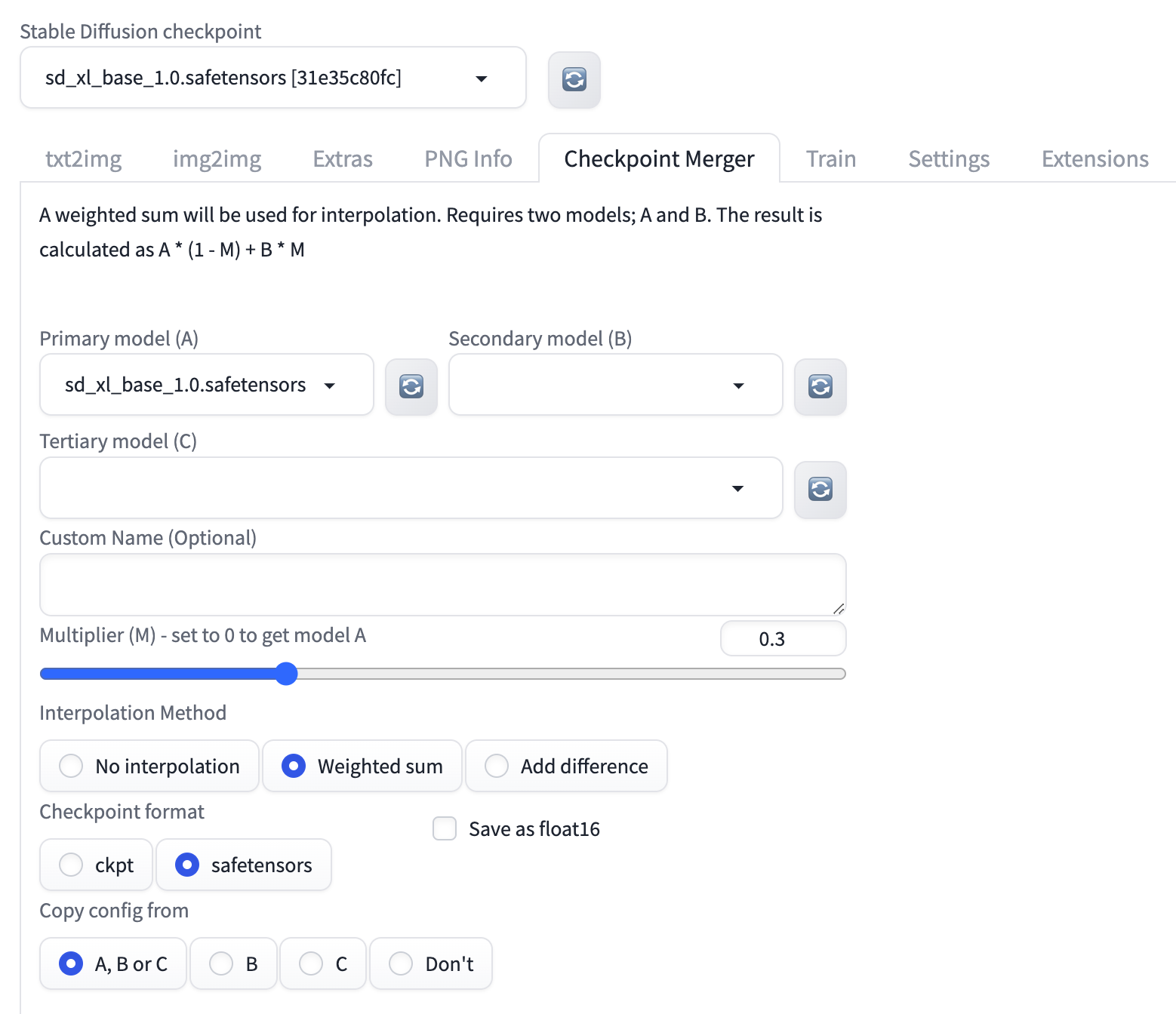}
    \caption{UI for checkpoint merging in \texttt{stable-diffusion-webui} \citep{AUTOMATIC1111}. Screenshot taken by us. The exposition text for computing the weighted sum was from the UI developer. The public repository has over $140$k GitHub stars and $26.7$k forks as of October 2024. To put its popularity in perspective, OpenAI's API demo repository \texttt{openai-cookbook} has $59.2$k stars and $9.4$k forks.}
    \label{fig:automatic1111}
\end{figure}

\subsubsection{Individual Auxiliary Models}

Sampling from an individual auxiliary model is the same as sampling from a regular SDXL model, also equivalent to setting $\boldsymbol{\lambda} = (1,0)$ or $(0,1)$ for the KL barycenter approach or checkpoint merging. The inference-time hyperparameter setting is the same as the one stated in Section \ref{sec:sdxl-kl}.

\subsubsection{CLIP Distance Computations}\label{subsub:CLIP}
To quantify the semantic shifts and nuances we observe visually in the generated samples, we calculate CLIP distances \citep{clip} between each sampled image and four semantic concepts: [Female, Male, East Asian, White] using OpenAI’s \href{https://huggingface.co/openai/clip-vit-base-patch32}{\texttt{clip-vit-base-patch32}} image encoder model. Per convention, the CLIP distance is calculated as $100 \times {cosine\_similarity}(E_I, E_T)$, where $E_I$ and $E_T$ are respectively the image and text embedding vectors encoded by the CLIP model. Intuitively, it measures the semantic similarity between a sampled image and a gender / ethnic concept. This setup was used to produce Figures \ref{fig:gender-scatter} and \ref{fig:main-kdes} shown earlier in the text. The kernel density estimation (KDE) contour plot is simply a smoothed version of the scatter plot, made using the \texttt{Seaborn} package for better readability. A couple additional comments are in place regarding Figure \ref{fig:gender-scatter}:

\begin{itemize}
    \item We use a 2D joint plot instead of picking one gender, because gender expression can be non-binary and an individual representation may be rationalizable to multiple classes simultaneously; a “woman” can also exhibit visual traits that are considered “masculine” by social norms; the capacity for storing and representing ambiguity, we believe, should be a trait of a robust, information-theoretically efficient world simulator.
    \item One can see the two unimodal clusters (red \& orange) formed by the empirical distribution of each biased, stereotypical model, confirming the qualitative judgement.
    \item The KL barycenter’s empirical distribution (blue) \textbf{interpolates} the valley region between the two modes, which was underexplored by the two biased models. By contrast, the baseline’s (checkpoint-merged model) empirical distribution appears to be more of a bimodal mixture of the two biased models’ empirical distributions, i.e., it is less effective than the barycenter at interpolating the underrepresented regions.
\end{itemize}

Figure \ref{fig:sdxl-ethnic-clips} repeats the statistical analysis of Figure \ref{fig:gender-scatter} for two ethnic concepts ``East Asian'' and ``White''. Notice that the empirical distributions of the two biased models still form two distinct unimodal clusters, but the distance between their centroids is shorter than that of the gender clusters’. The relative ambiguity is expected, as there are many more identifiable ethnic categories a person can fall under besides East Asian or White. A sharper clustering would require the inclusion of more ethnic covariates to ``lift'' the scatter points to a higher-dimensional space. One can still observe the semantic dispersion of the KL barycenter in the valley region between the two biased modes, though the distinction with the checkpoint-merged model is less pronounced than that in Figure \ref{fig:gender-scatter}. Both of the adapted models exhibit a slight bias towards sampling the ``White'' hyperplane, which, as we hypothesized earlier, is likely an artifact of marginalization over other latent ethnic concepts.

\begin{figure*}[htbp]
    \centering
    \begin{subfigure}
        \centering
        \includegraphics[width=0.48\textwidth]{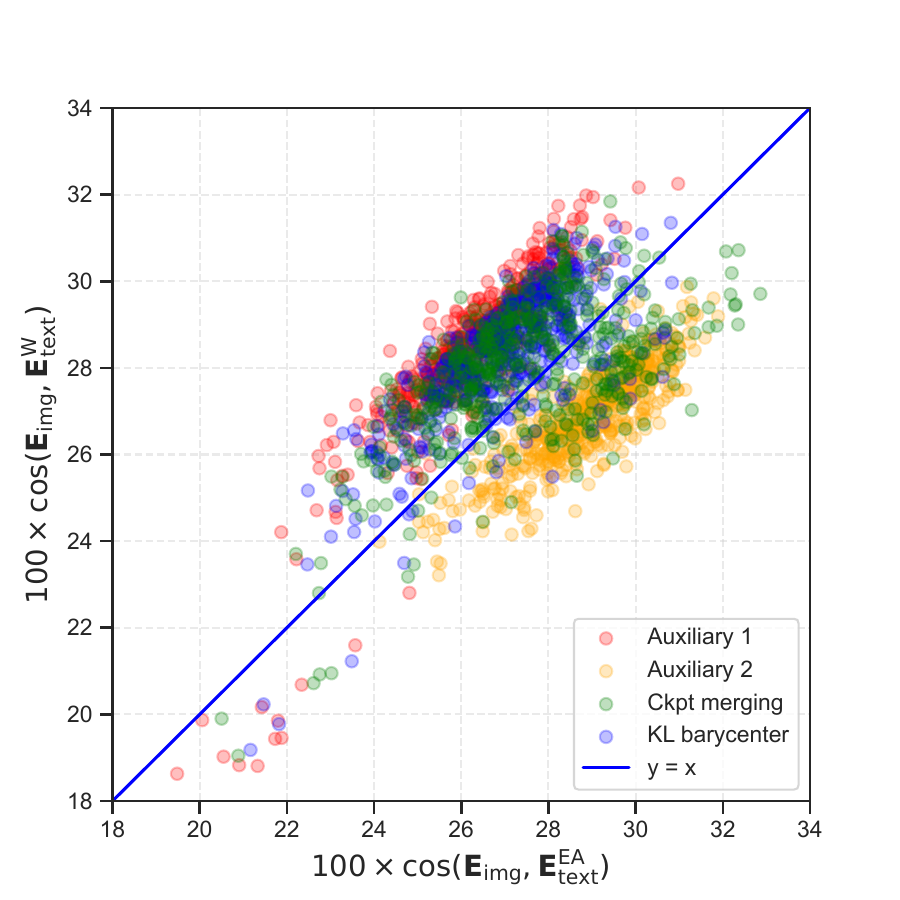}
    \end{subfigure}
    \hfill
    \begin{subfigure}
        \centering    
        \includegraphics[width=0.48\textwidth]{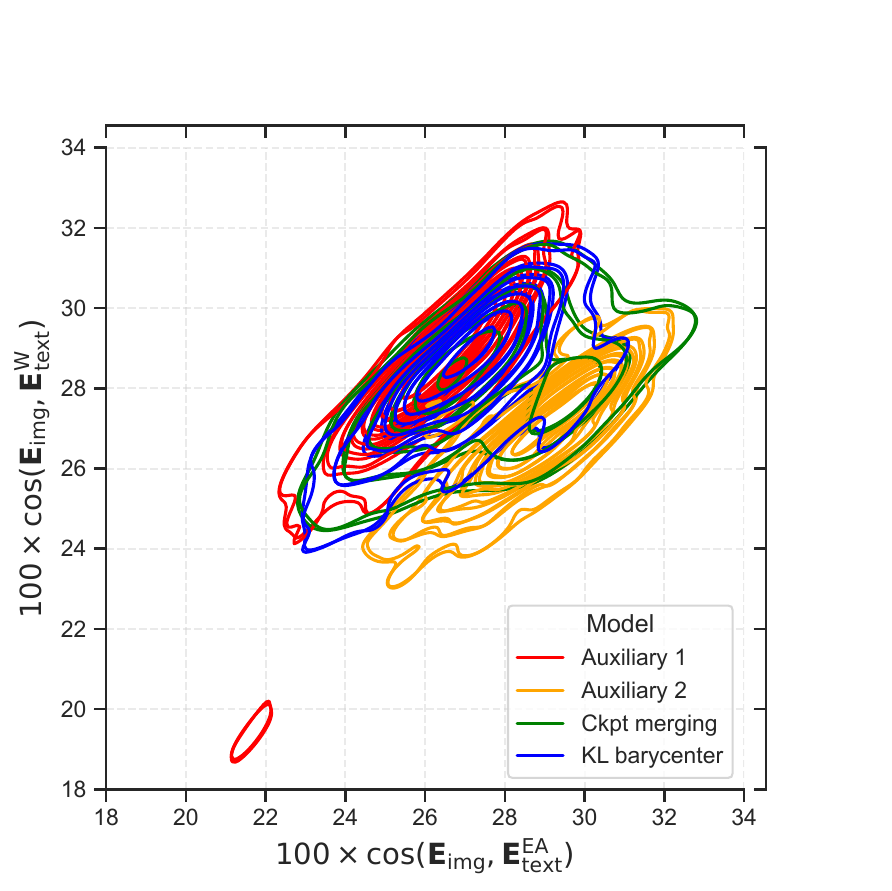}
    \end{subfigure}
    \caption{\textbf{Left}: Scatter plot of the empirical distribution (512 images) of each generative model, projected onto a 2D ethnicity semantic space. \textbf{Right}: KDE contour plot of the kernel density estimation (KDE) of the left scatter plot, with a bandwidth of 0.8. \newline $\mathbf{E}_{img}$ are CLIP embeddings of images samples. $\mathbf{E}_{text}^{EA}, \mathbf{E}_{text}^{W}$ are (fixed) text embeddings of \textit{``a photo of an East Asian scientist''} and \textit{``a photo of a White scientist''}. The line $y=x$ is drawn to indicate the idea of ``ethnic ambiguity''. The same PyTorch seed was used for all four pipelines.}
    \label{fig:sdxl-ethnic-clips}
\end{figure*}

\subsection{Ablation Study}\label{sub:ablation}

\begin{figure}[htbp]
    \centering
    \begin{subfigure}
        \centering
        \includegraphics[width=0.48\textwidth]{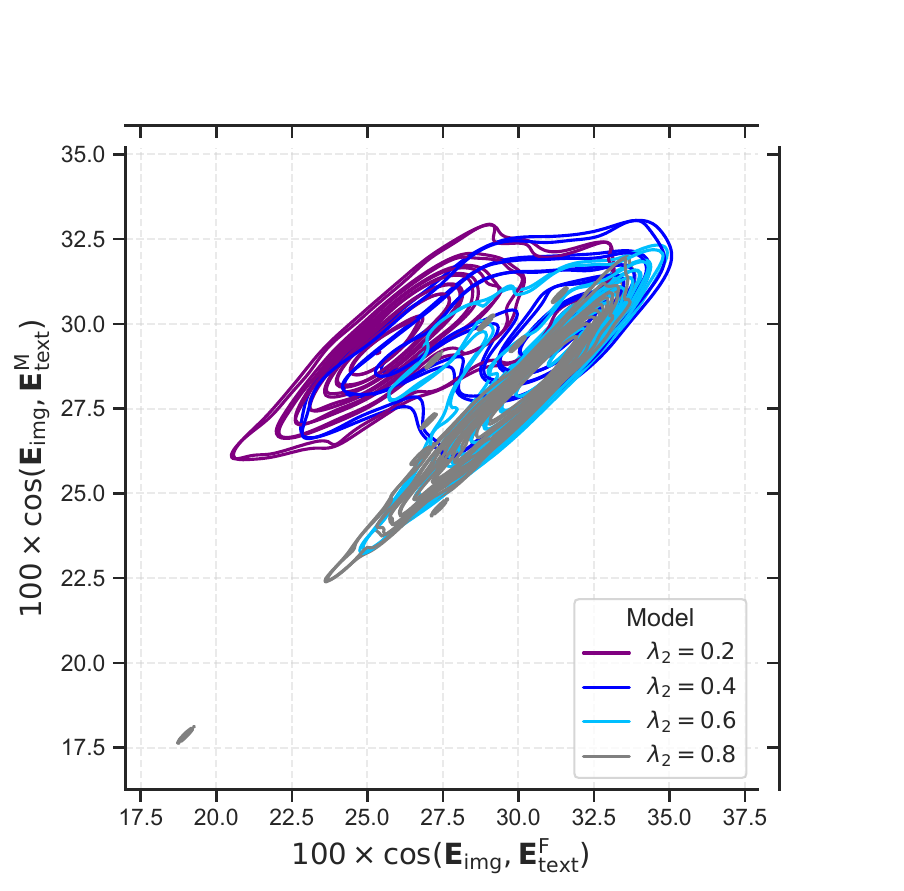}
    \end{subfigure}
    \hfill
    \begin{subfigure}
        \centering    
        \includegraphics[width=0.48\textwidth]{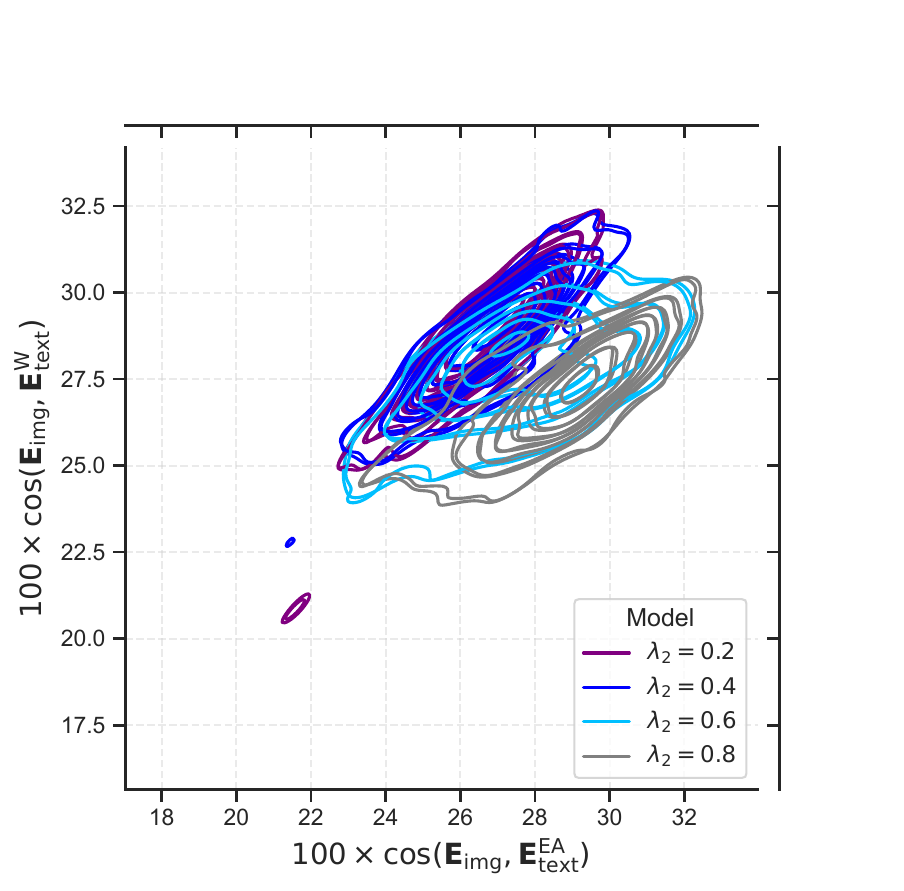}
    \end{subfigure}
    \caption{\textbf{Left}: KDE contour plot of the KL barycenter distribution under different weights $\lambda_2$, projected onto a 2D gender semantic space. \textbf{Right}: Same analysis, but onto a 2D ethnicity semantic space. \newline $\lambda_2=0$ and $\lambda_2=1$ (not shown) respectively correspond to ``\textit{Auxiliary 1}'' and ``\textit{Auxiliary 2}'' in Figure \ref{fig:SDXL-ablation-02461}.}
    \label{fig:SDXL-ablation-2468}
\end{figure}

\begin{figure}[htbp]
    \centering
    \begin{subfigure}
        \centering
        \includegraphics[width=0.48\textwidth]{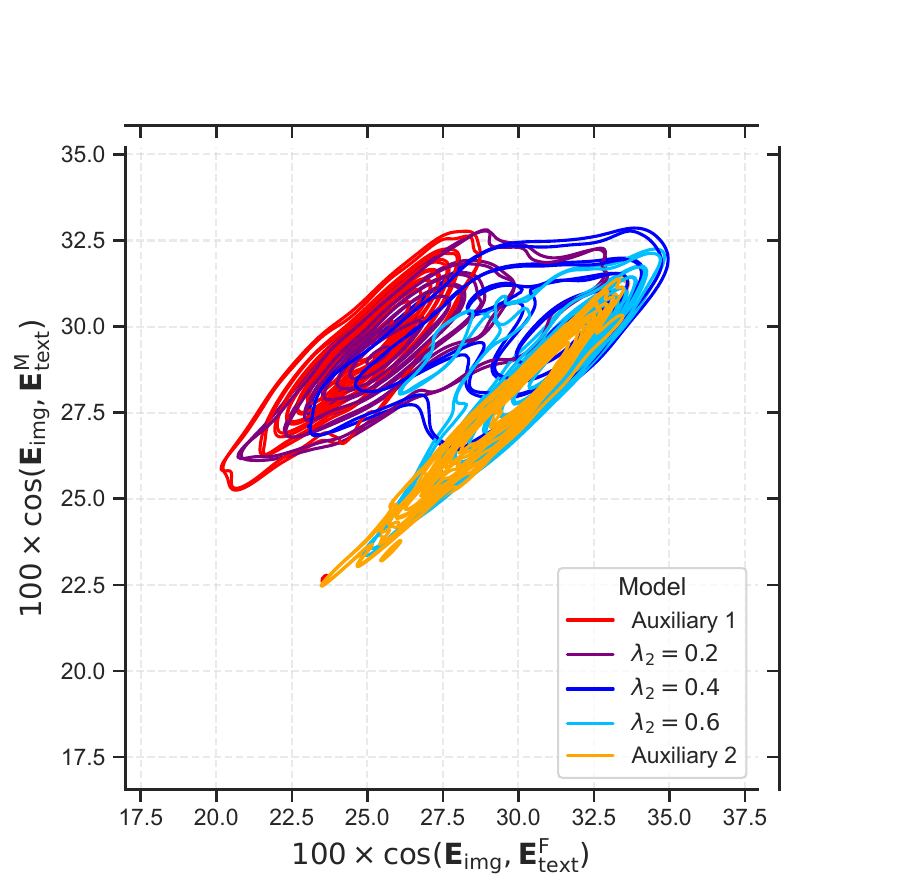}
    \end{subfigure}
    \hfill
    \begin{subfigure}
        \centering    
        \includegraphics[width=0.48\textwidth]{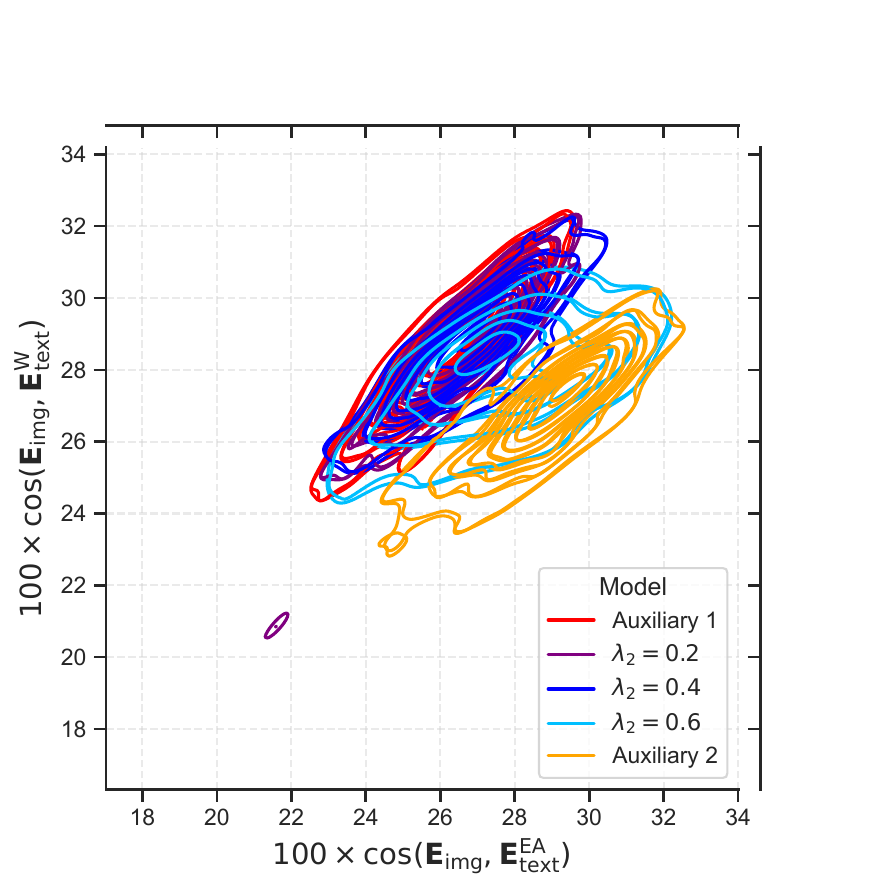}
    \end{subfigure}
    \caption{\textbf{Left}: KDE contour plot of the KL barycenter distribution under different weights $\lambda_2$, projected onto a 2D gender semantic space. \textbf{Right}: Same analysis, but for ethnicity CLIP distances. \newline In this figure, we focus on comparing three non-trivial interpolation values with the original distributions.}
    \label{fig:SDXL-ablation-02461}
\end{figure}

We provide additional quantitative evaluations on the visual effect of varying the KL barycenter weight vector $\boldsymbol{\lambda} \in \Delta^1$. Figure \ref{fig:SDXL-ablation-2468} showcases shifts in image semantics (left subfigure is in the gender space, right subfigure is in the ethnicity space). Both subfigures show a \textbf{directional} drift of the semantic distribution as $\lambda_2$ increases, indicating the smoothness of the KL-divergence barycenter trajectory.

Another noteworthy observation is the asynchronous-ness of the two semantic transitions (gender \& ethnicity): one might think that the transition would proceed at roughly the same rate, but instead, Figure \ref{fig:SDXL-ablation-2468} shows that between $0.2$ (purple contours) to $0.4$ (blue contours), the gender distribution drifts dramatically while the ethnicity distribution stays mostly unchanged. Figure \ref{fig:SDXL-ablation-02461} swaps out samples from $\lambda_2=0.8$ for the two original auxiliary models to allow for a complementary comparison.

In addition to Figure \ref{fig:compare-interpolation}, we provide another visual comparison of KL barycenter and checkpoint merging on different $\boldsymbol{\lambda}$ values in Figure \ref{fig:compare-interpolation-appendix-1}.

\begin{figure*}[htbp]
  \centering
  \includegraphics[width=0.95\textwidth]{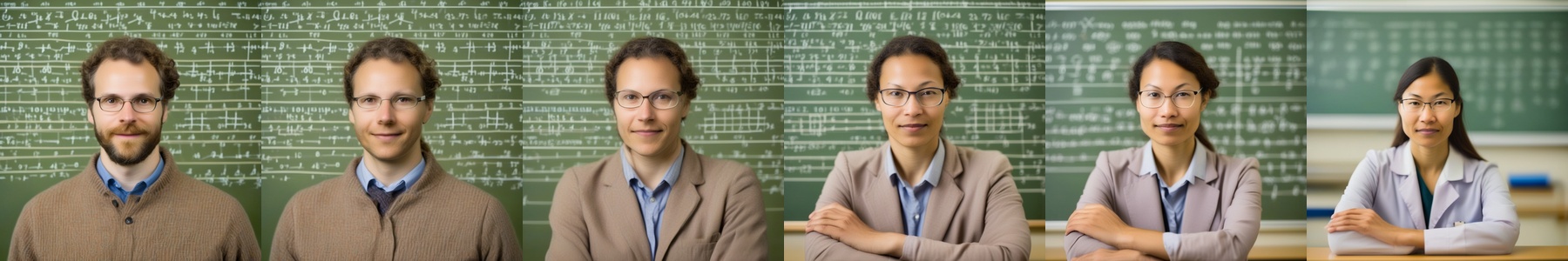}
  \includegraphics[width=0.95\textwidth]{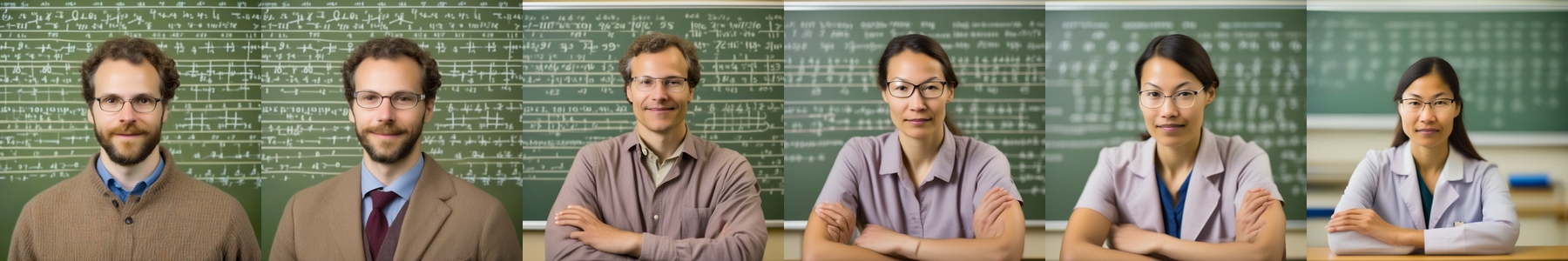}

  \caption{
    \textbf{Top row}: KL barycenter. 
    \textbf{Bottom row}: Checkpoint merging.\newline The same Gaussian noise (but different from that of Figure \ref{fig:compare-interpolation}) was used to seed all twelve images. From left to right, $\lambda_2 \in \{0, 0.2, 0.4, 0.6, 0.8, 1.0\}$ and $\lambda_1 = 1 - \lambda_2$. $\lambda_2=0$ and $\lambda_2=1$ each reduce to an original (biased) auxiliary SDXL model. Observe that the bottom row samples show an abrupt identity shift from $\lambda_2 = 0.2$ to $0.4$ and $0.4$ to $0.6$, whereas the top row shows a smoother transition from one identity concept to another.
    }
  \label{fig:compare-interpolation-appendix-1}
\end{figure*}

\subsection{Additional Portrait Samples}\label{SDXL:more-samples}

For consistency, we use the same text prompt as the one stated in section \ref{sec:sdxl-kl} for all image generations. Importantly, the text prompt does not specify the gender or race of the person. However, as Figure \ref{fig:SDXL-biased} shows, two grids of samples generated by the two auxiliary models display a homogeneous gender \& ethnic expression. This setup is to contrast how our KL barycenter approach can effectively spotlight an underrepresented subpopulation in either model. Figures \ref{fig:SDXL-comp1}, and \ref{fig:SDXL-comp2} provide additional samples generated by the model combination methods (checkpoint merging \& KL barycenter) beyond those presented in the main text.

\begin{figure*}[htbp]
    \centering
    \begin{subfigure}
        \centering
\includegraphics[width=0.42\textwidth]{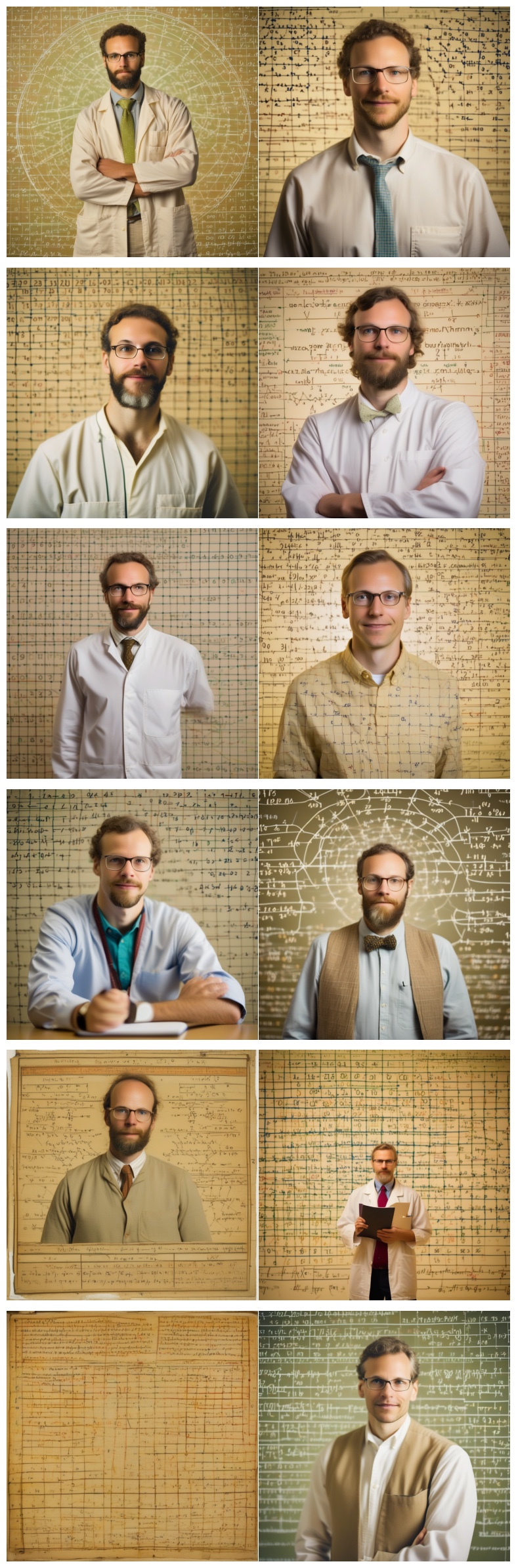}
    \end{subfigure}
    \hfill
    \begin{subfigure}
        \centering     \includegraphics[width=0.42\textwidth]{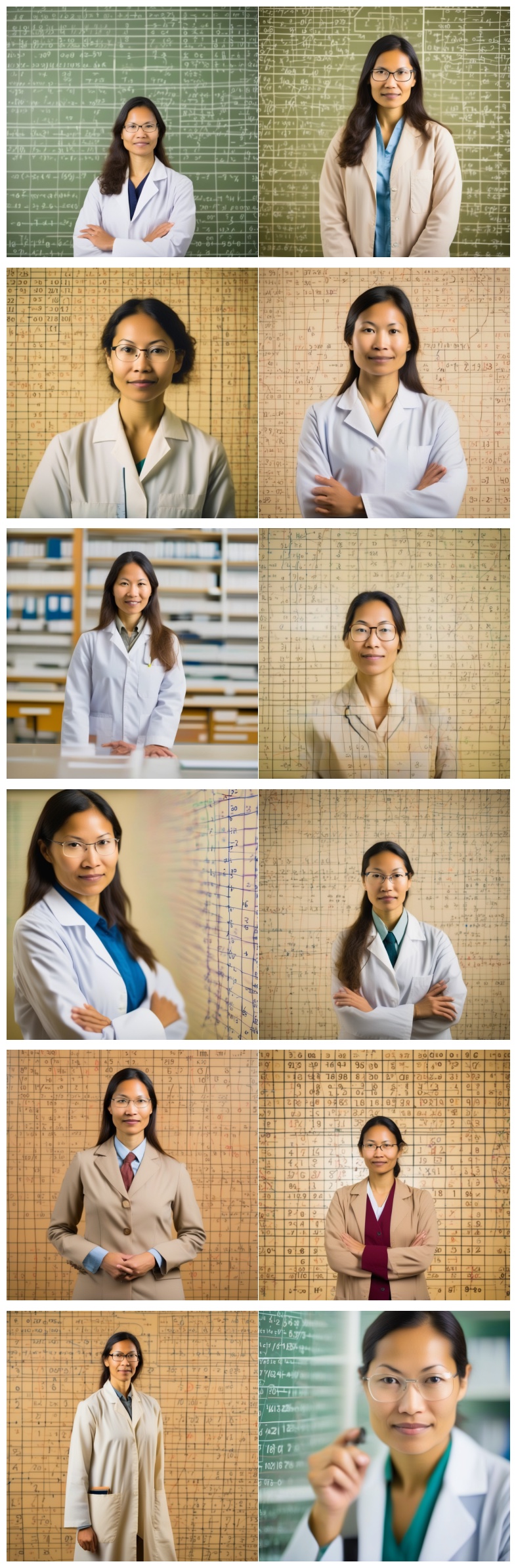}
    \end{subfigure}
    \caption{\textbf{Left}: uncurated samples from the first (biased) auxiliary model. \textbf{Right}: uncurated samples from the second (biased) auxiliary model. Each 6x2 grid of portraits (here and in Figure \ref{fig:SDXL-comp1}) is generated from the same set of initial latent noise tensors, ensuring comparability across models.}
    \label{fig:SDXL-biased}
\end{figure*}

\begin{figure*}[htbp]
    \centering
    \begin{subfigure}
        \centering
\includegraphics[width=0.42\textwidth]{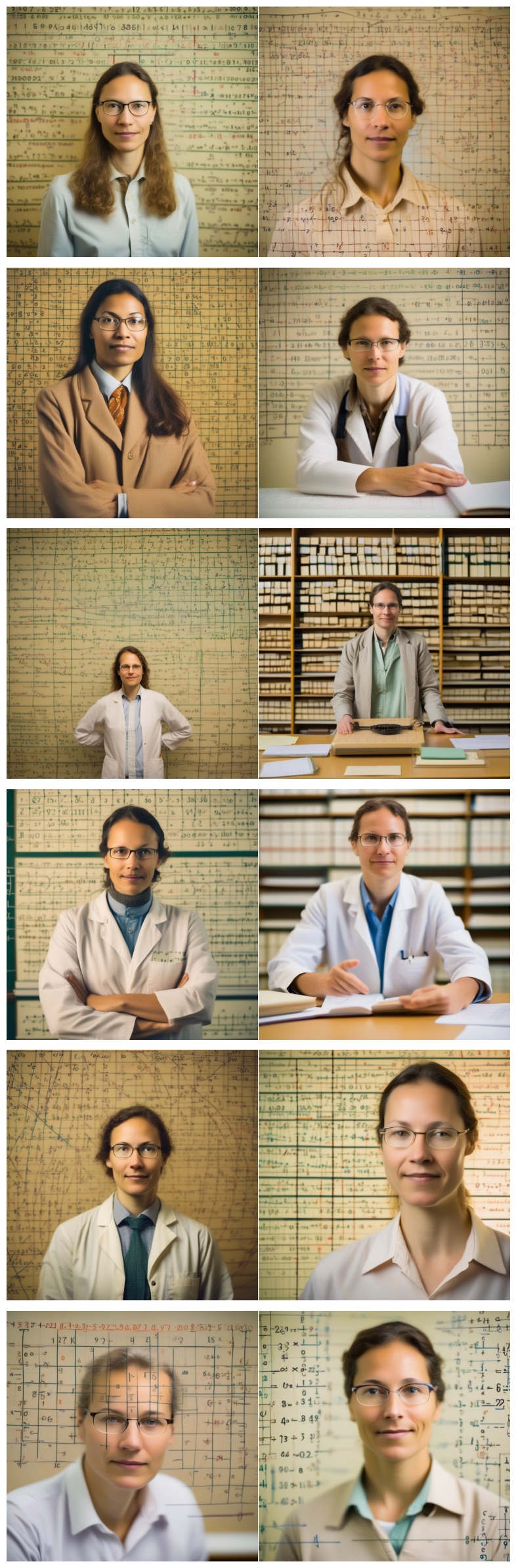}
    \end{subfigure}
    \hfill
    \begin{subfigure}
        \centering     \includegraphics[width=0.42\textwidth]{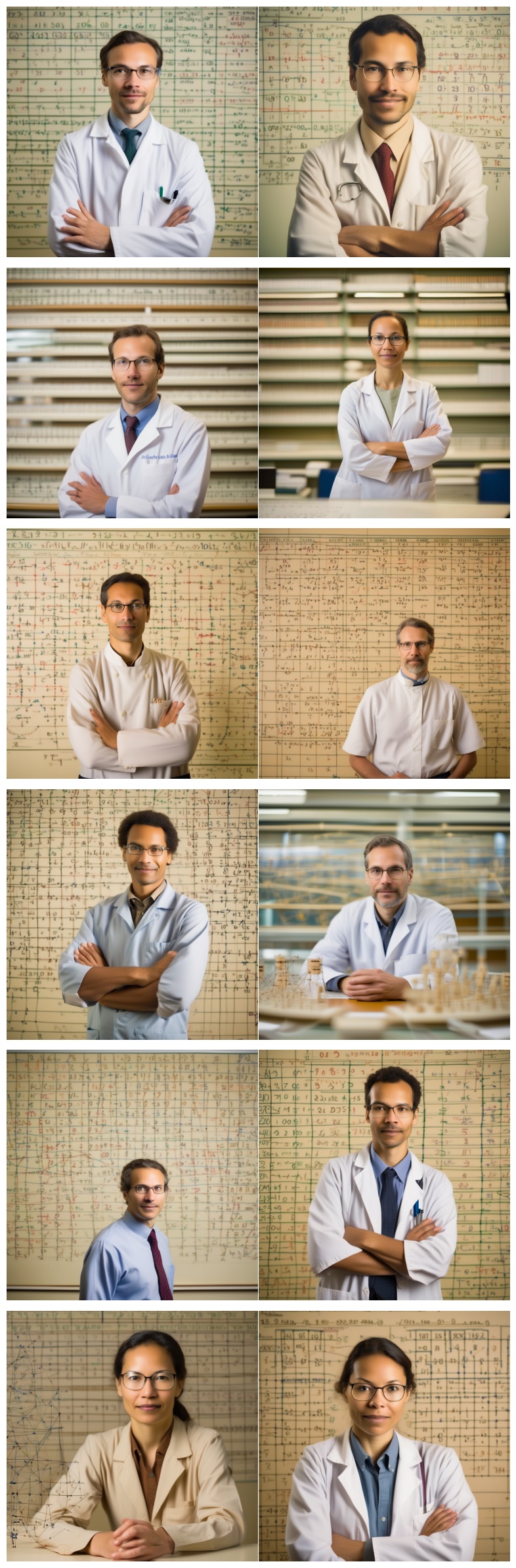}
    \end{subfigure}
    \caption{\textbf{Left}: uncurated samples from the KL barycenter distribution with $\boldsymbol{\lambda} = (0.5, 0.5)$. \textbf{Right}: uncurated samples from a 50-50 checkpoint merged model. The left grid embodies a more neutral, ambiguous gender expressions, such as the absence of mustaches or beards present in the right and Figure \ref{fig:SDXL-biased}.}
    \label{fig:SDXL-comp1}
\end{figure*}

\begin{figure*}[htbp]
    \centering
    \begin{subfigure}
        \centering
\includegraphics[width=0.42\textwidth]{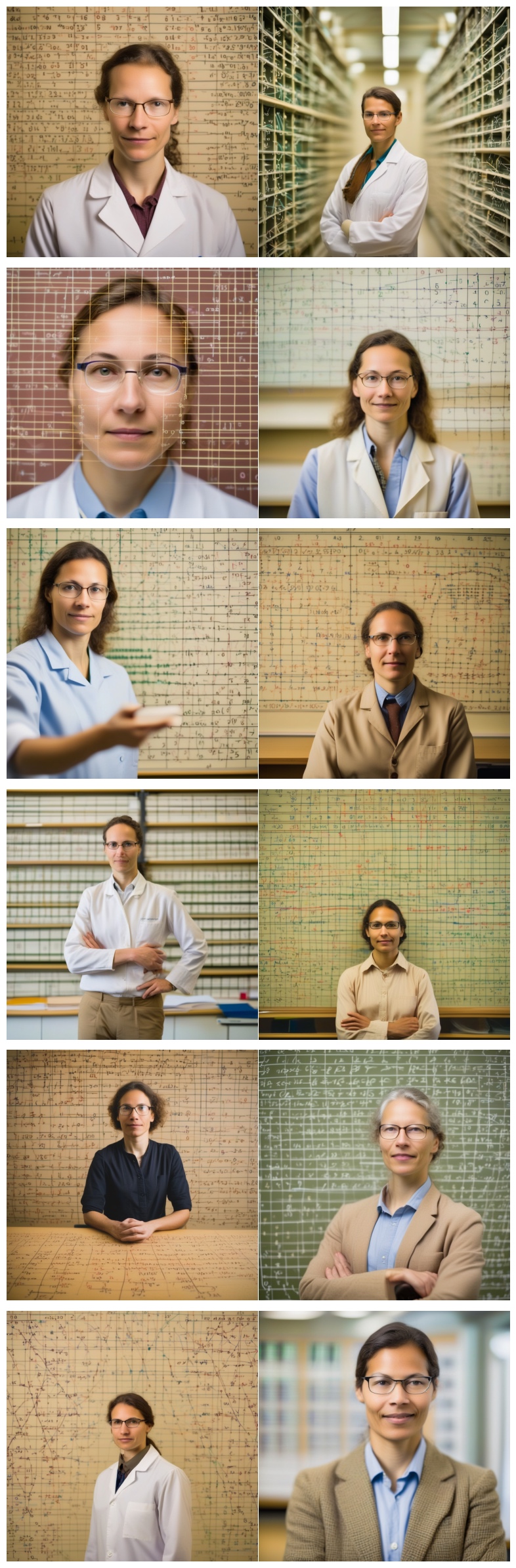}
    \end{subfigure}
    \hfill
    \begin{subfigure}
        \centering     \includegraphics[width=0.42\textwidth]{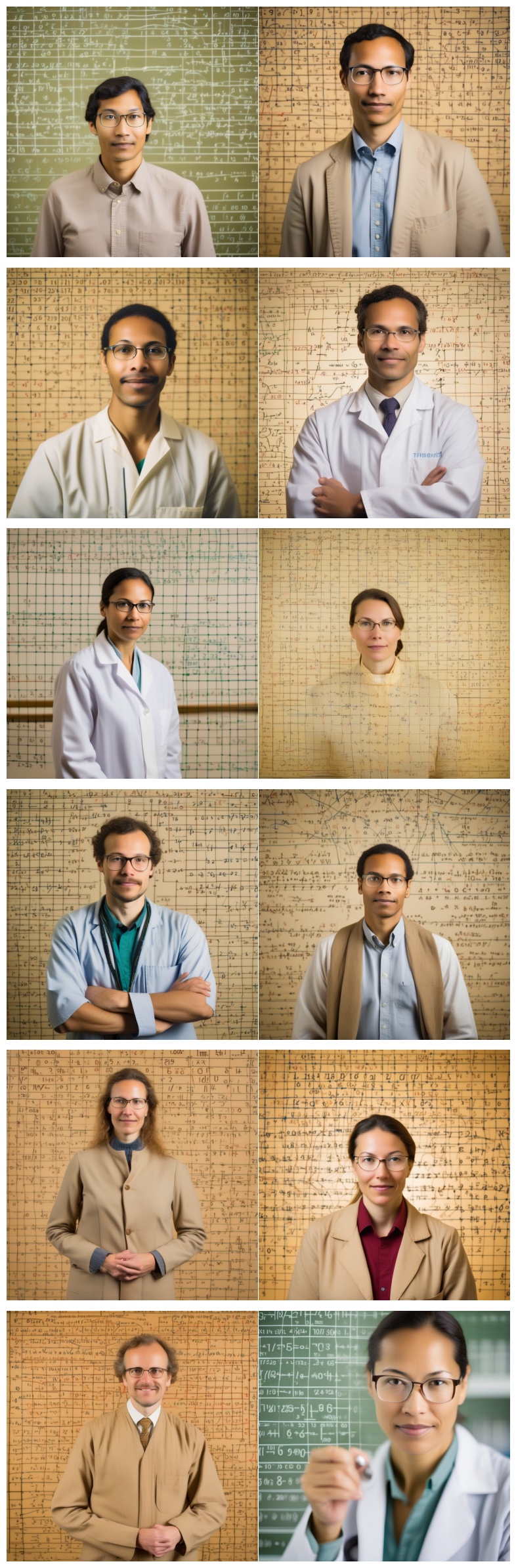}
    \end{subfigure}
    \caption{More image samples, continuing Figure \ref{fig:SDXL-comp1}. \textbf{Left}: uncurated samples from the KL barycenter distribution with $\boldsymbol{\lambda} = (0.5, 0.5)$. \textbf{Right}: uncurated samples from a 50-50 checkpoint merged model.}
    \label{fig:SDXL-comp2}
\end{figure*}

\subsection{Theory-Inspired Heuristic for the Phenomenon}\label{sec:theory-inspired}

As we saw in Figures \ref{fig:SDXL-comp1} and \ref{fig:SDXL-comp2}, the KL barycenter of the two auxiliary models generate professional portraits that seem unlikely under the marginal distribution of either model, depicting an underrepresented and much more gender-neutral gender expression. We provide a heuristic, theory-based explanation for this phenonmenon, using an example of combining two one-dimensional Gaussian models. Recall from Theorem \ref{sol1} the density function of the distribution-level KL barycenter with barycenter weights vector $\boldsymbol{\lambda}$:

\begin{equation*}
    p_{\boldsymbol{\lambda}}(x) = \frac{\prod _{i=1}^k p_i(x)^{\lambda_i}}{\int_{\mathbb{R}^d} \prod _{i=1}^k p_i(x)^{\lambda_i} dx},
\end{equation*}

where $p_1,\cdots,p_k$ are the probability density functions of the marginal distributions of the $k$ auxiliary models, and $d$ is the dimension of the Euclidean space that $x$ lives in. 

Set $k=2$, $d=1$. Let $p_1$ and $p_2$ be probability density functions of Gaussian distributions with means $\mu_1$, $\mu_2$ and variance $\sigma_1^2$, $\sigma_2^2$. Then the numerator becomes

\begin{equation*}
    \prod_{i=1}^2 p_i(x)^{\lambda_i} = \left( \prod_{i=1}^2 \left( \frac{1}{\sqrt{2 \pi \sigma_i^2}} \right)^{\lambda_i} \right) \exp\left( -\sum_{i=1}^2 \lambda_i \frac{(x - \mu_i)^2}{2 \sigma_i^2} \right).
\end{equation*}

Putting aside constant terms, we see that $p_{\boldsymbol{\lambda}}$ is also Gaussian. With some routine algebraic manipulations, it can be shown that the distribution has mean $\mu_{\boldsymbol{\lambda}} = \frac{B}{A}$ and variance $\sigma_{\lambda}^2=\frac{1}{2A}$, where 

\begin{equation*}
    A = \lambda_1 \frac{1}{2\sigma_1^2} + \lambda_2 \frac{1}{2\sigma_2^2},\; B = \lambda_1 \frac{\mu_1}{2\sigma_1^2} + \lambda_2 \frac{\mu_2}{2\sigma_2^2}
\end{equation*}

and we impose the constraint that $\lambda_i \geq 0$ and $\lambda_1 + \lambda_2=1$. When $\lambda_i = 1$, $\mu_{\boldsymbol{\lambda}}=\mu_i$ and $\sigma_{\boldsymbol{\lambda}}=\sigma^2_i$. Figure \ref{fig:1d-gaussian} shows the barycenter distribution under different $\lambda_1$ values.

\begin{figure}[htbp]
    \centering
    \includegraphics[width=0.45\textwidth]{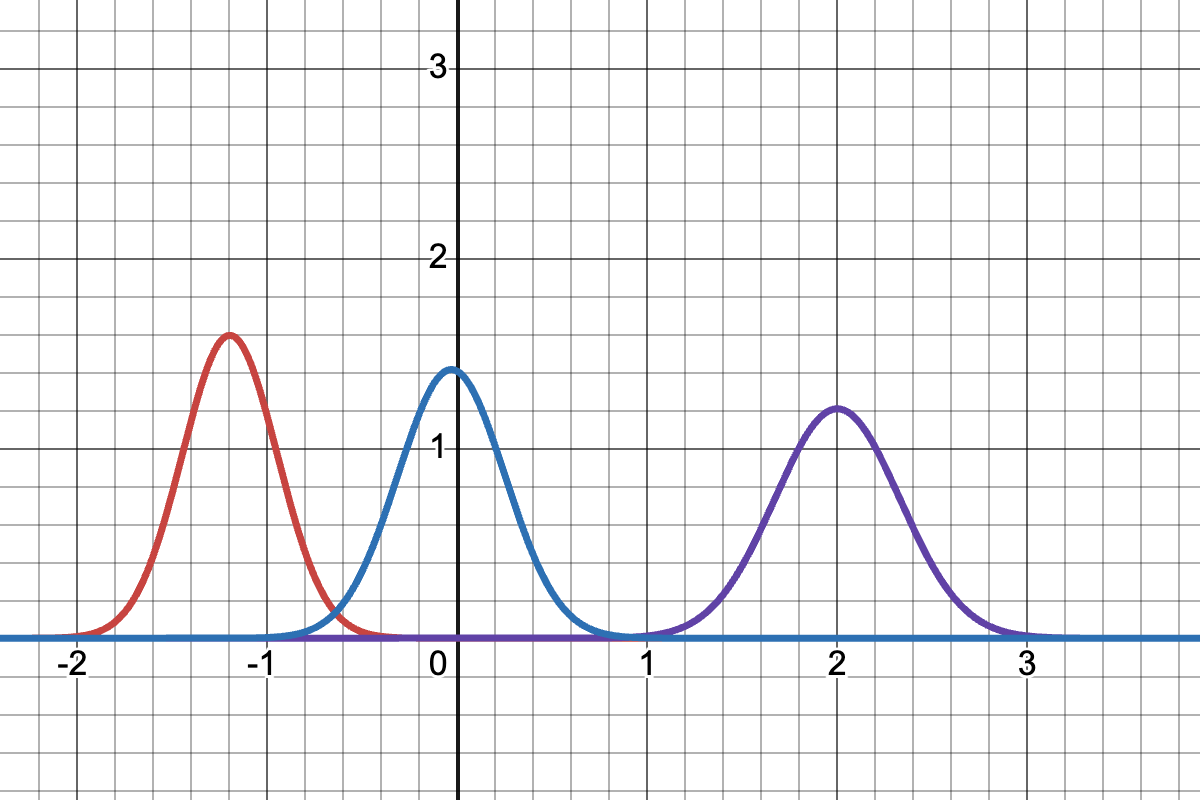}
    \caption{\textbf{Blue}: KL-divergence barycenter PDF with $\lambda_1 = 0.5$. \textbf{Red} and \textbf{Purple}: two auxiliary (unimodal Gaussian) PDFs. Red corresponds to $\lambda_1 = 1$, purple corresponds to $\lambda_1 = 0$.}
    \label{fig:1d-gaussian}
\end{figure}

One way to interpret this stylized setup is that the real line represents a spectrum of gender expression (or another semantic concept), with the origin ($x=0$) denoting the peak of gender ambiguity. This corresponds to the two auxiliary models in Figure \ref{fig:SDXL-biased}, each generating distinct, yet narrowly focused gender representations. As Figure \ref{fig:1d-gaussian} illustrates, at $\lambda_1 = 0.5$, the barycenter distribution ultimately emphasizes sampling from regions where neither generative model assigns high probability. Like celestial bodies pulling one another into balance, the diffusion process finds a shared orbit between the biases of both models. The barycenter does not merely split the difference—it crafts a new gravitational midpoint, where the intersection of opposing forces yields something not yet observed, harmonizing their influence on the vast canvas of the latent space.

\section{ADDITIONAL EXPERIMENT WITH BIMODAL GAUSSIAN MIXTURES}\label{sub:1D}

We also test ScoreFusion's ability to approximate a one-dimensional \textit{bimodal} Gaussian mixture distribution using two auxillary models that were trained on different Gaussian mixtures. Since the data is synthetic, we can access the true density function of the target distribution and auxiliary distributions, shown in the right of Figure \ref{fig:1d-64}; the ground truth distribution is in grey. Table \ref{tab:w1-a1} gives the $1$-Wasserstein distance $\mathcal{W}_1$ between the distribution learned by ScoreFusion and the ground truth distribution, calculated using \texttt{SciPy}.

\begin{table}[ht]
\centering
\caption{1-Wasserstein distance from the ground truth distribution. Standard error is calculated from the $\mathcal{W}_1$ distances of $10$ random draws of $8096$ samples from each generator.}
\label{tab:w1-a1}
\begin{tabular}{@{}cccc@{}}
\Xhline{1.2pt} 
Model & $2^5$ & $2^6$ & $2^7$ \\ 
\midrule
Baseline & $106.93 \pm 1.43$ & $13.46 \pm 0.28$ & $16.74 \pm 0.27$ \\
ScoreFusion & $\mathbf{0.39 \pm 0.02}$ & $\mathbf{0.51 \pm 0.03}$ & $\mathbf{0.36 \pm 0.02}$ \\
\midrule
$\boldsymbol{\lambda}^*$ of ScoreFusion & $[0.62, 0.38]$ & $[0.65, 0.35]$ & $[0.46, 0.54]$ \\
\Xhline{1.2pt} 
\end{tabular}
\end{table}

\begin{table}[ht]
\centering
\label{tab:w1-a2}
\begin{tabular}{@{}ccccccc@{}}
\Xhline{1.2pt} 
Model & $2^8$ & $2^9$ & $2^{10}$ \\ 
\midrule
Baseline & $2.13 \pm 0.12$ & $0.55 \pm 0.04$ & $\mathbf{0.15 \pm 0.02}$ \\
ScoreFusion & $\mathbf{0.58 \pm 0.03}$ & $\mathbf{0.38 \pm 0.02}$ & $0.30 \pm 0.02$ \\
\midrule
$\boldsymbol{\lambda}^*$ of ScoreFusion & $[0.68, 0.32]$ & $[0.61, 0.39]$ & $[0.58, 0.42]$ \\
\Xhline{1.2pt} 
\end{tabular}
\end{table}

\begin{figure*}[ht]
    \centering
    \begin{subfigure}
        \centering
\includegraphics[width=0.48\textwidth]{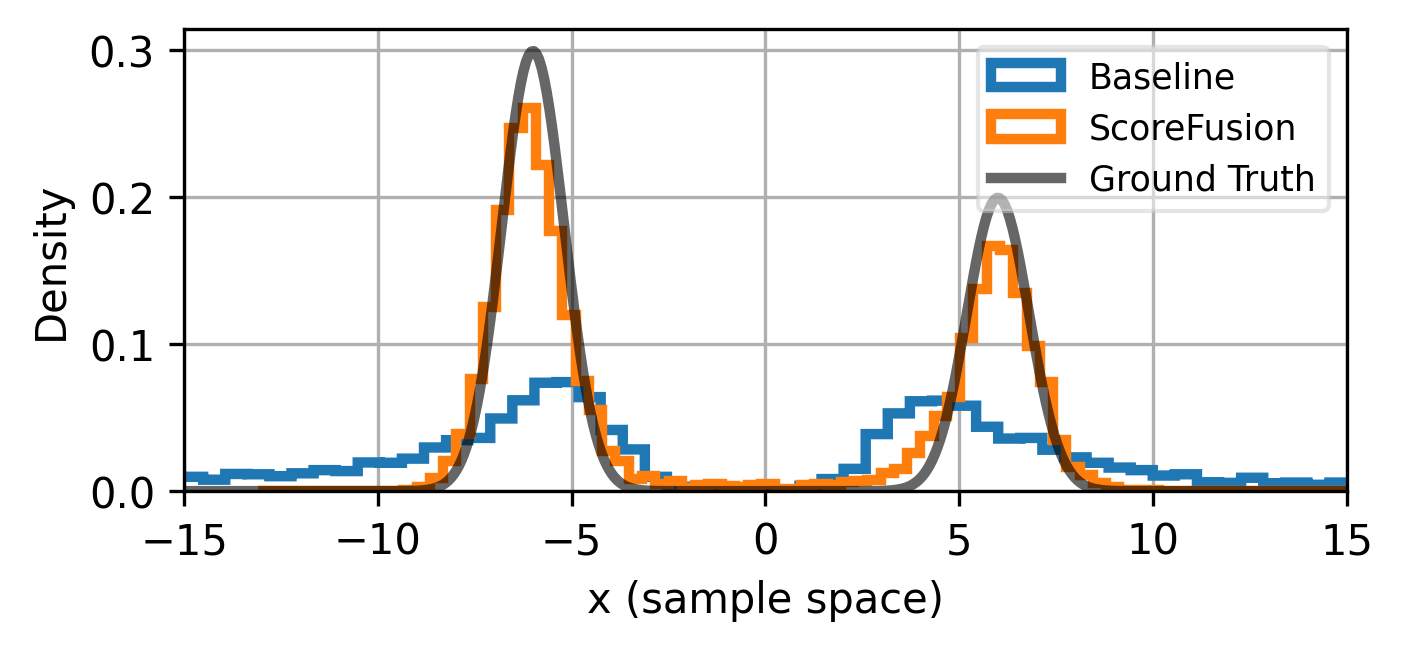}
    \end{subfigure}
    \hfill
    \begin{subfigure}
        \centering     \includegraphics[width=0.48\textwidth]{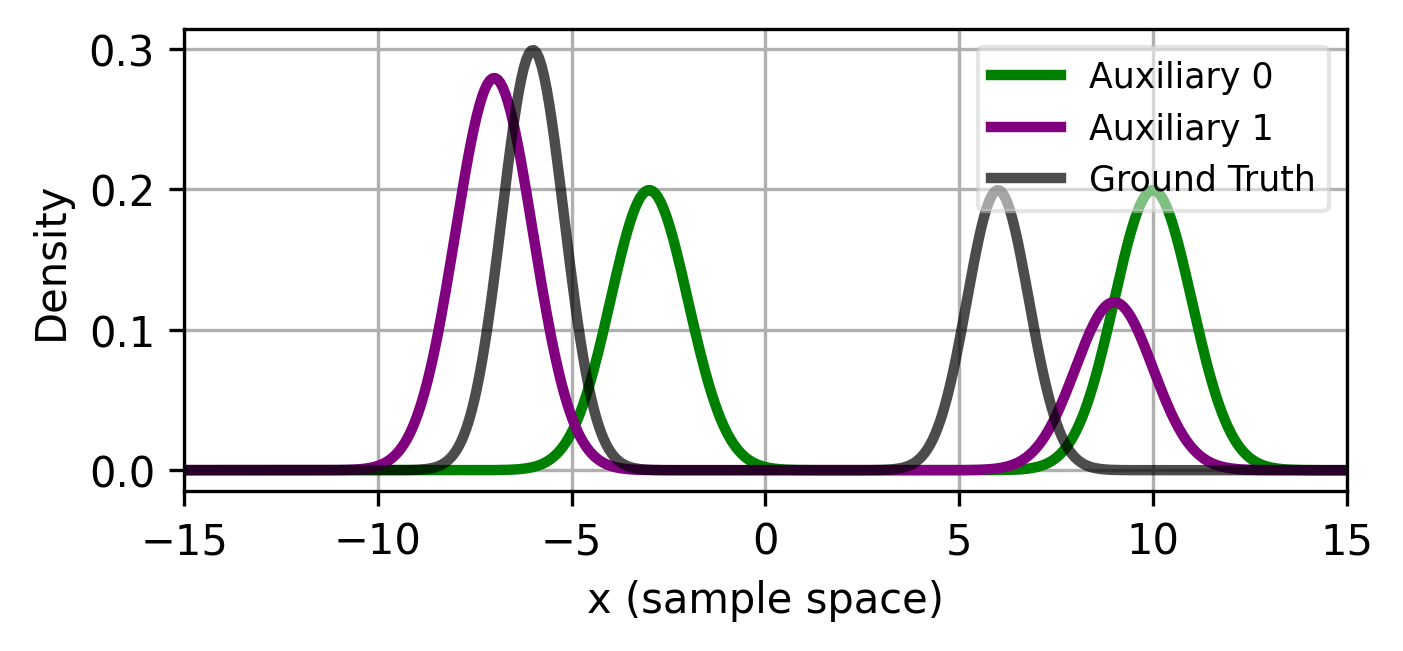}
    \end{subfigure}
    \caption{\textbf{Left}: Histograms of $8096$ ScoreFusion samples and $8096$ baseline samples; both models are calibrated / trained on $64$ samples. \textbf{Right}: Density functions of ground truth vs. the auxiliary distributions.}
    \label{fig:1d-64}
\end{figure*}

Using only $64$ training data, ScoreFusion can already learn a good representation of the ground truth distribution. In contrast, the standard diffusion model is overly widespread and fails to capture the modes of the Gaussian mixture. Moreover,   ScoreFusion  consistently outperforms the baseline in $\mathcal{W}_1$ distance when the number of training data is fewer than $2^{10}$. 

Additional histograms of the distributions learned by ScoreFusion versus the baseline are attached:

\begin{figure}[H]
    \centering
    \begin{subfigure}
        \centering
\includegraphics[width=0.48\textwidth]{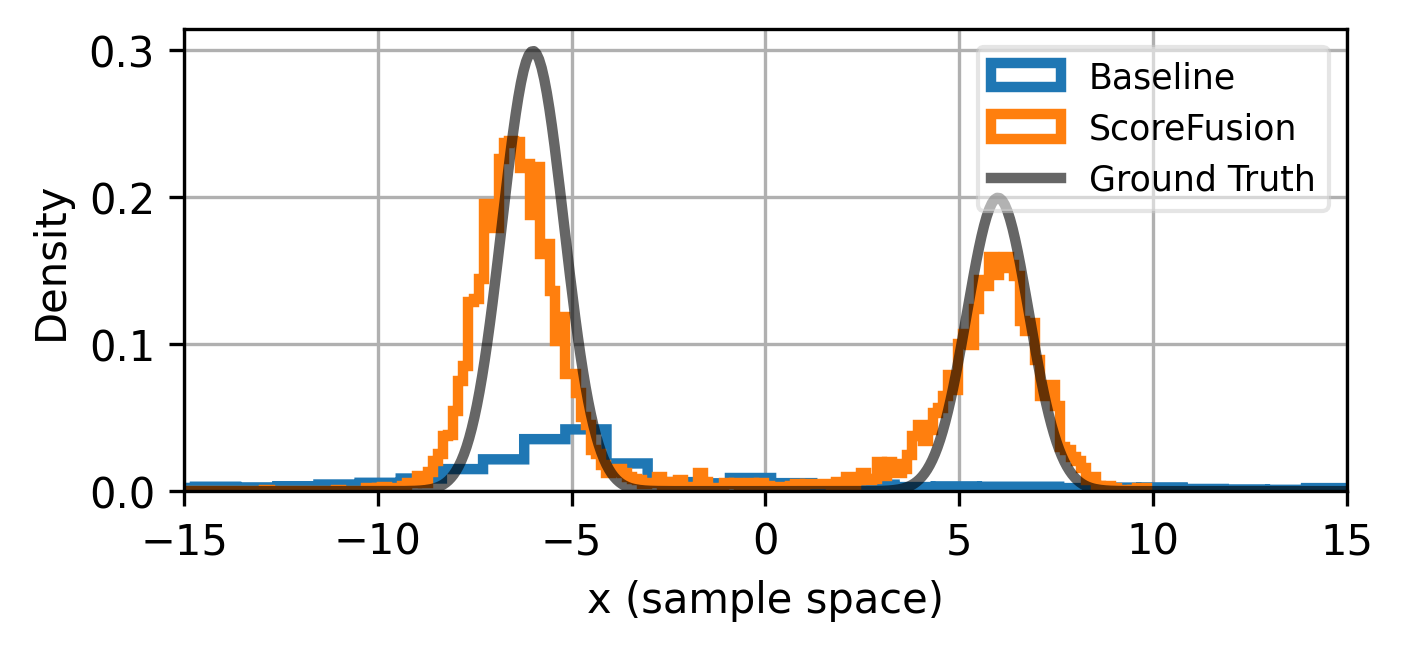}
    \end{subfigure}
    \hfill
    \begin{subfigure}
        \centering     \includegraphics[width=0.48\textwidth]{images/1D/1d_64_comp1.png}
    \end{subfigure}
    \caption{\textbf{Left}: Models trained on $32$ samples. \textbf{Right}: Models trained on $64$ samples.}
    \label{fig:1d-a1}
\end{figure}

\begin{figure}[H]
    \centering
    \begin{subfigure}
        \centering
\includegraphics[width=0.48\textwidth]{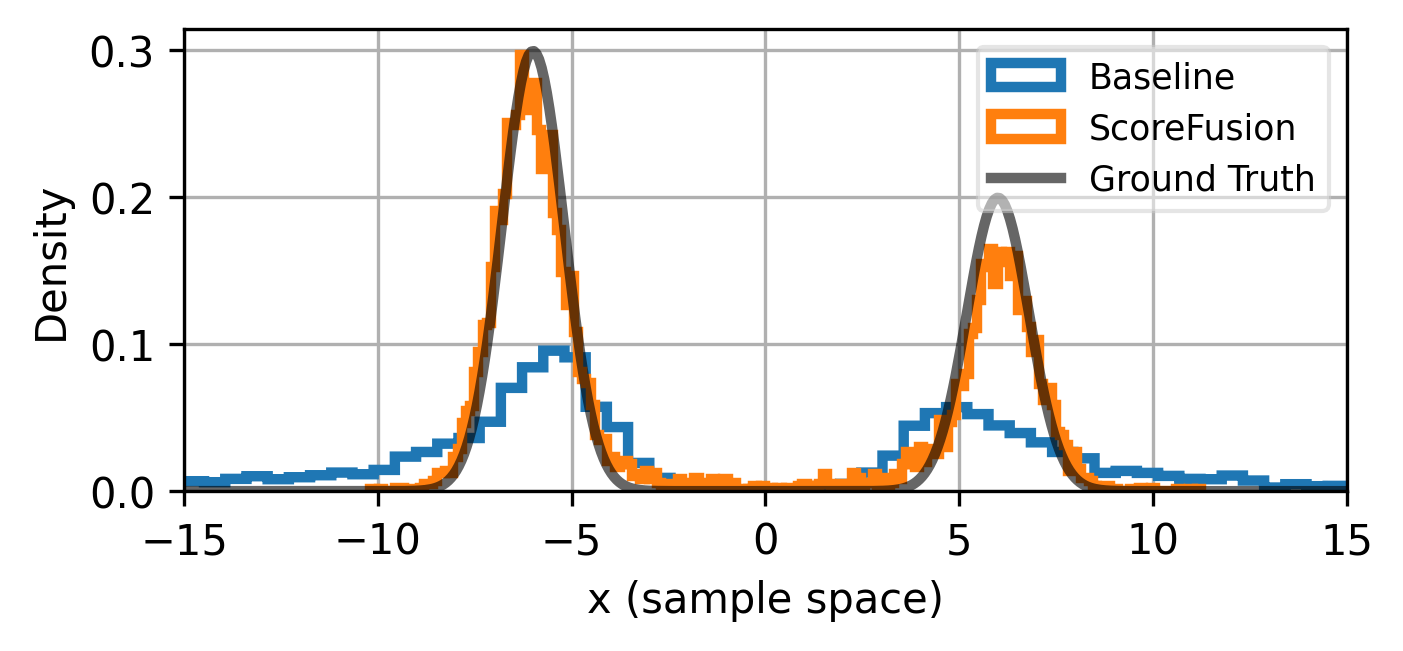}
    \end{subfigure}
    \hfill
    \begin{subfigure}
        \centering     \includegraphics[width=0.48\textwidth]{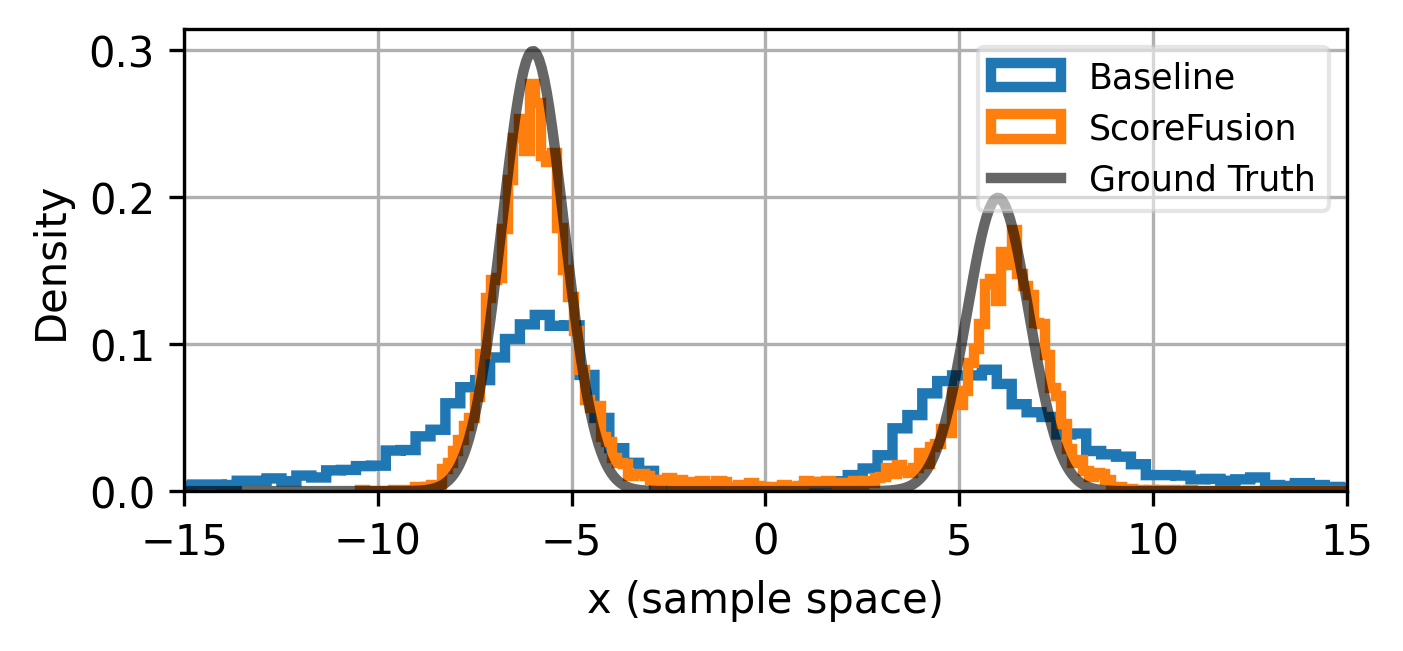}
    \end{subfigure}
    \caption{\textbf{Left}: Models trained on $128$ samples. \textbf{Right}: Models trained on $256$ samples.}
    \label{fig:1d-a2}
\end{figure}

\begin{figure}[H]
    \centering
    \begin{subfigure}
        \centering
\includegraphics[width=0.48\textwidth]{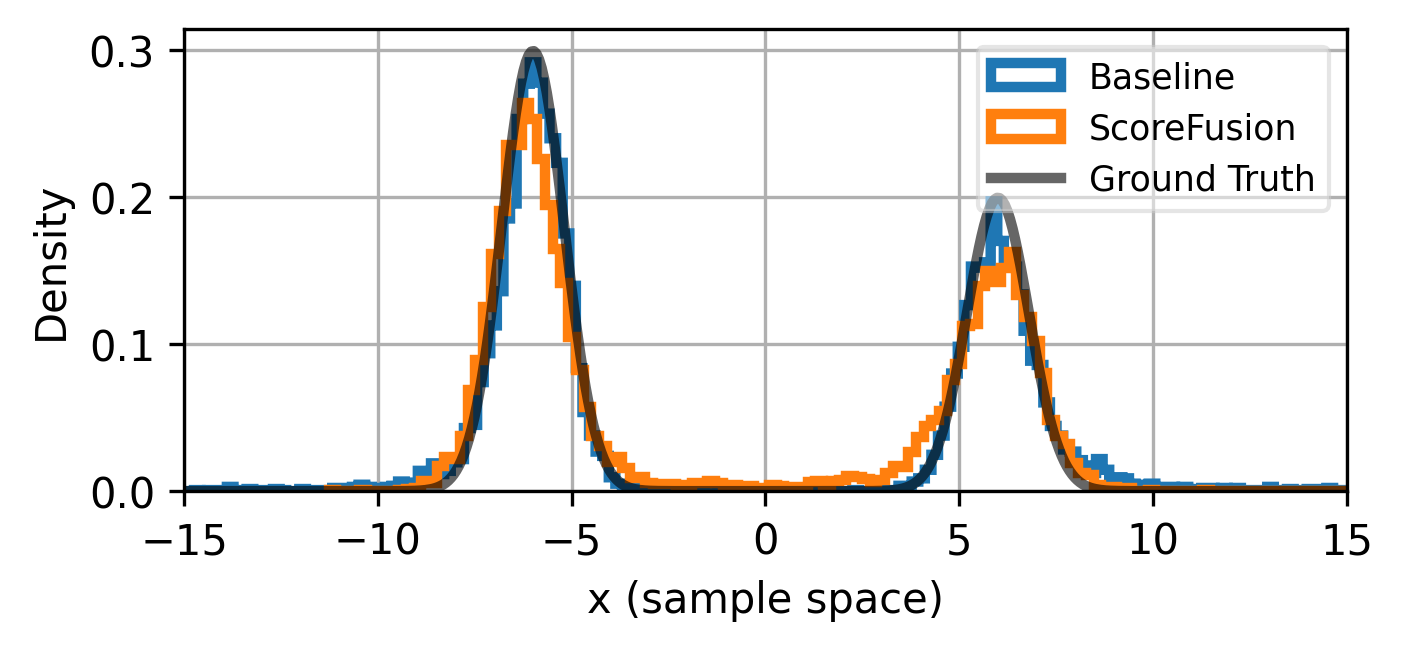}
    \end{subfigure}
    \hfill
    \begin{subfigure}
        \centering     \includegraphics[width=0.48\textwidth]{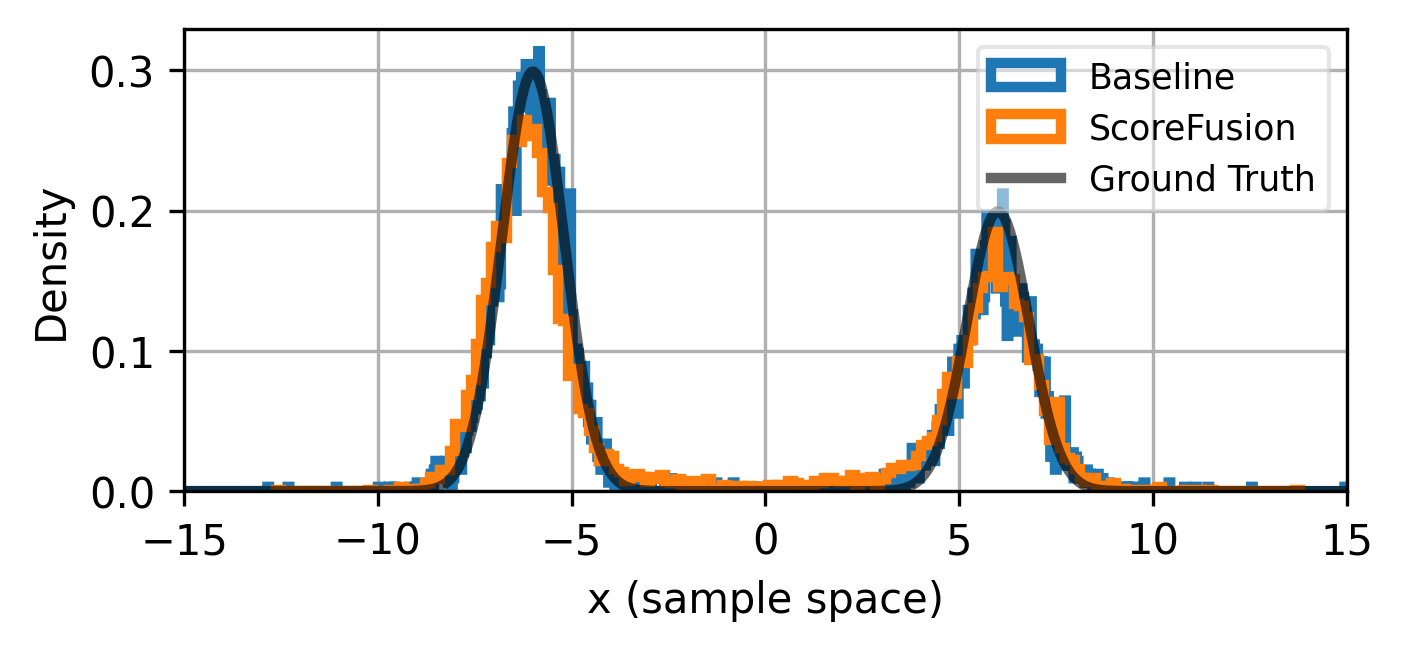}
    \end{subfigure}
    \caption{\textbf{Left}: Models trained on $512$ samples. \textbf{Right}: Models trained on $1024$ samples.}
    \label{fig:1d-a3}
\end{figure}

\end{document}